\documentclass{article} % ICML 推荐使用默认 article 类，由 sty 文件控制版面

\usepackage{microtype}
\usepackage{graphicx}
\usepackage{booktabs} 
\usepackage{hyperref}

\usepackage[preprint]{icml2026}

\usepackage{times}
\usepackage{latexsym}
\usepackage{subcaption} % 推荐加入，方便排版子图

\usepackage[T1]{fontenc}
\usepackage[utf8]{inputenc}
\usepackage{inconsolata}

\usepackage{graphicx}
\usepackage{subcaption}
\usepackage{multirow}

\usepackage{amsmath}

\newtheorem{lemma}{Lemma}[section]
\newtheorem{proof}{Proof}[section]
\usepackage{paralist}

\newtheorem{proposition}{Proposition}
\newtheorem{assumption}{Assumption}
\newtheorem{corollary}{Corollary}

\usepackage[capitalize]{cleveref}
\crefname{section}{Sec.}{Secs.}
\Crefname{section}{Section}{Sections}
\Crefname{table}{Table}{Tables}
\crefname{table}{Tab.}{Tabs.}

\usepackage{microtype}
\usepackage{booktabs} % 

\definecolor{darkgreen}{rgb}{0.0, 0.5, 0.0} 

\newcommand{\comt}[1]{#1}

\renewcommand{\comt}[1]{}

\usepackage[dvipsnames,table]{xcolor}

\usepackage{tabularx}
\usepackage{multicol}
\definecolor{myblue}{RGB}{235,235,250}

\usepackage{xspace}
\usepackage{pifont}
\usepackage{latexsym}
\usepackage{inconsolata}
\usepackage{arydshln}
\usepackage{enumitem}
\usepackage{wrapfig}
\usepackage{array}
\usepackage{epigraph}

\usepackage{amsmath,amssymb}

\usepackage{times}
\usepackage{latexsym}
\usepackage{amsmath,amssymb}
\usepackage{graphicx}
\usepackage{booktabs}
\usepackage{microtype}

\definecolor{lightpink}{RGB}{204, 231, 207} %lightgreen tabcolor3
\definecolor{lightblue}{RGB}{210, 220, 250} %lightblue tabcolor5
\definecolor{lightgray}{RGB}{237, 237, 237} %lightblue tabcolor5

\definecolor{superlightred}{rgb}{0.99, 0.92, 0.92}
\definecolor{darkgreen}{RGB}{50,100,0}
\definecolor{darkred}{RGB}{200, 0, 0}

\usepackage{makecell}
\usepackage{colortbl}

\renewcommand{\multirowsetup}{\centering}
\definecolor{mygray}{gray}{.92}
\definecolor{mygreen1}{RGB}{253, 244, 244}%255, 234, 234
\definecolor{mygreen2}{RGB}{238, 243, 243}%223, 240, 240
\definecolor{ForestGreen}{RGB}{34,139,34}
\newcommand{\fg}[1]{\mathbf{\mathcolor{ForestGreen}{#1}}}
\definecolor{Forestred}{RGB}{220,50,50}

\begin{document}

% =========================================================
% ICML 作者与标题块 (必须放在 \twocolumn[...] 内部)
% =========================================================
\twocolumn[

% \icmltitle{CLASP: Prompt-Routed Integration and Sparsification of Multi-layer visual tokens}
\icmltitle{CLASP: Class-Adaptive Layer Fusion and Dual-Stage Pruning for Multimodal Large Language Models}

% List of affiliations: The first argument should be a (short)
% identifier you will use later to specify author affiliations
% Academic affiliations should list Department, University, City, Region, Country
% Industry affiliations should list Company, City, Region, Country

\begin{icmlauthorlist}
\icmlauthor{Yunkai Dang\textsuperscript{*}}{sch} 
\icmlauthor{Yizhu Jiang\textsuperscript{*}}{sch} 
\icmlauthor{Yifan Jiang}{sch}
\icmlauthor{Qi Fan}{sch}
\icmlauthor{Yinghuan Shi}{sch}
\icmlauthor{Wenbin Li\textsuperscript{$\dagger$}}{sch} % <--- 加上了 $ 符号解决报错
\icmlauthor{Yang Gao}{sch}
\end{icmlauthorlist}

\icmlaffiliation{sch}{School of Artificial Intelligence Science and Technology, Nanjing University}

\icmlcorrespondingauthor{Wenbin Li}{liwenbin.nju@gmail.com, yunkaidang1@gmail.com}

\icmlkeywords{Machine Learning, ICML, Vision Language Models} 

\vskip 0.3in
] 

% 使用 \quad 增加一个空格，然后在后面补上对应的 dagger 符号解释
\printAffiliationsAndNotice{\icmlEqualContribution \quad $\dagger$ Corresponding author.}
% 正文内容
% 

\begin{abstract} 
Multimodal Large Language Models (MLLMs) suffer from substantial computational overheads, driven by the massive redundancy of visual token sequences. 
To mitigate such redundancy, existing works typically rely on single-layer ViT features and static pruning strategies.
However, these fixed configurations often render them brittle across diverse instructions. 
To address these limitations, we present \textit{class-adaptive layer fusion and dual-stage pruning (CLASP)}, a plug-and-play token reduction framework. 
Specifically, we construct a category-specific visual representation via multi-layer vision feature fusion. 
Then we perform dual-stage pruning that allocates the token budget between attention-salient pivots (relevance) and redundancy-aware completion tokens (coverage). 
By class-adaptively pruning, our method shows that prompt-conditioned feature fusion and budget allocation enable aggressive yet robust visual token pruning. 
Experiments show that our method achieves superior performance across various benchmarks, pruning ratios, and MLLM architectures compared to existing methods.
Code will be available at \url{https://github.com/Yunkaidang/CLASP}.
\end{abstract}

\section{Introduction}
\label{sec:intro}

\begin{figure}[t]
    \centering
    \includegraphics[width=1\linewidth]{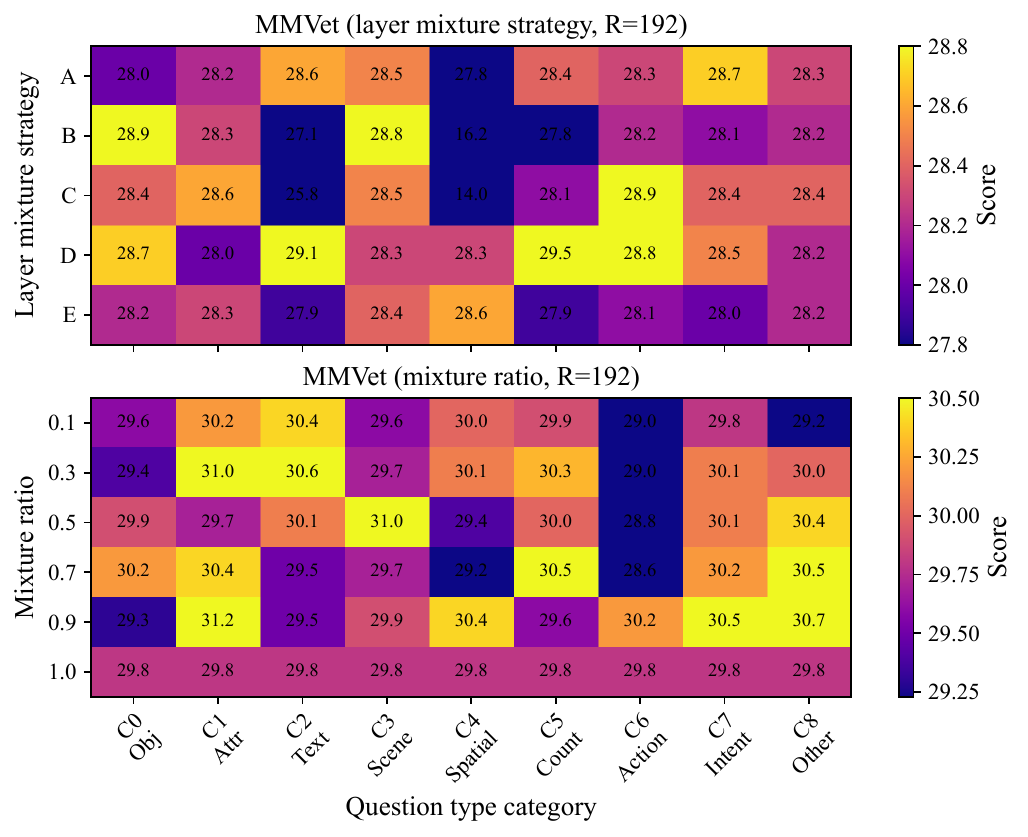}
    \caption{
    Impact of hyperparameter settings on MMVet dataset performance (LLaVA-v1.5-7B, 192 retained tokens).
    The heatmaps illustrate the score distribution across different question categories under varying conditions.
    \textbf{Top:} Evaluation of five representative layer-fusion strategies (A–E), ordered by an increasing proportion of weights assigned to deeper layers (\textit{i.e.}, shifting from shallow in A to deep in E).
    \textbf{Bottom:} Sweeping the attention--similarity mixing coefficient (higher: more attention; lower: more similarity).}
    \label{fig:motivation-1}
\end{figure}

Multimodal Large Language Models (MLLMs) extend the instruction-following and reasoning capabilities of LLMs to the visual domain~\citep{yin2024survey，dang2024explainable}.
These models typically align pre-trained vision encoders with autoregressive language decoders through an alignment module~\citep{wu2023multimodal}. 
Despite their impressive performance across various multimodal benchmarks~\citep{liu2024mmbench,fu2025mme}, MLLMs face significant challenges in practical utility. 
The primary bottleneck stems from the substantial computational overhead associated with processing visual tokens.
In these architectures, images are typically serialized into long, high-density sequences of patch-level tokens~\citep{jin2025efficient}. 
This representation causes the computational cost of self-attention to scale quadratically with the input resolution. 
For instance, increasing the input resolution from $336 \times 336$ in LLaVA-1.5~\citep{liu2023llava} to $672 \times 672$ in LLaVA-NeXT~\citep{liu2024llava} raises the token count from $576$ to a staggering $2,880$. 
Such a proliferation of tokens imposes excessive demands on both memory footprint and inference latency. 
Ultimately, these overheads create a major barrier to efficient inference and deployment~\citep{kong2025token}.

To reduce the inference cost of MLLMs, recent studies~\citep{xing2024pyramiddrop,chen2024image,zhang2024sparsevlm,bolya2022token,wen2025token,zou2025don} propose token reduction approaches to prune visual tokens.
Existing approaches can be broadly grouped into two primary directions.
The first line leverages cross-modal attention scores~\citep{chen2024image,zhang2024sparsevlm,zhan2024exploring,arif2025hired,zhang2025vispruner} to identify and retain instruction-relevant visual regions, primarily focusing on token relevance.
The second line employs similarity-based pruning~\citep{bolya2022token,xing2024pyramiddrop,yang2025visionzip,alvar2025divprune,zhang2025beyond} to remove redundant patches by merging or dropping low-saliency tokens, emphasizing the coverage of visual information.
While existing approaches have achieved promising results from various perspectives, most methods still rely on category-insensitive and fixed strategies for both visual feature extraction and token pruning.
Specifically, these methods rely on a single or a fixed set of ViT layers to derive visual representations.
These static representations serve as both pruning signals and decoder inputs, yet they fail to capture task-specific nuances. 
Building on this static foundation, existing methods~\citep{chen2024image,zhang2024sparsevlm,wen2025token,zou2025don} then apply a rigid pruning strategy with predefined parameters, regardless of the instruction category.
However, such methods overlook a critical nuance: different intents emphasize distinct levels of \textit{visual abstraction} and impose varying requirements on spatial \textit{relevance and coverage}.

Based on these insights, we contend that both visual feature extraction and token pruning should be dynamically adjusted according to input question categories. 
%We argue against applying a uniform, global rule across all types of instructions. 
To verify this, we evaluate the sensitivity of specific instruction categories to different configurations under a fixed token budget (MM-Vet, LLaVA-v1.5-7B, $R=192$). 
We first conduct a preliminary analysis on layer-mixture strategies to optimize visual feature extraction (Fig.~\ref{fig:motivation-1}, top). 
Specifically, we employ multiple schemes (A$\rightarrow$E) that shift the fusion emphasis from shallow to deep ViT layers. 
The results indicate that model accuracy varies significantly depending on the layer depth of the integrated features. 
For instance, the performance of the Count category degrades as the fusion emphasis shifts from hybrid to deep layers ($29.5$ \textit{vs.}$27.9$). 
% 为了xxx，我们
To further examine the impact of pruning criteria, we investigate the limitations of relying on a fixed balance between attention and similarity (Fig.~\ref{fig:motivation-1}, bottom). 
We implement a hybrid strategy that transitions the pruning logic between attention-led and similarity-led modes. 
The findings show that a single, static mixing ratio fails to achieve optimal performance for every category. 
Specifically, text-based categories reach peak accuracy at lower ratios ($\approx 0.3$), whereas counting categories require higher ratios ($0.7$--$0.9$) in MMVet. 
These results confirm that effective token reduction requires class-aware adaptation of both layer-wise feature integration and the balance between relevance and coverage.

% As shown in Fig.~\ref{fig:motivation-1}, we vary the layer used to score visual tokens and the relevance--coverage split under a fixed token budget.
% We find that visual token importance is prompt-dependent, and static pruning policies are brittle across intent categories.
% Fine-grained prompts (\textit{e.g.}, OCR and counting) rely on local strokes, textures, and layouts, which are better preserved in shallow-to-mid representations.
% In contrast, holistic prompts (\textit{e.g.}, captioning and scene reasoning) can rely more on semantic invariances in deeper layers and often prefer stronger coverage to preserve global context.
% The optimal relevance–coverage balance varies across instruction categories, and mismatched choices can cause sharp category-wise degradation.

Motivated by this evidence, we propose \textit{class-adaptive layer fusion and dual-stage pruning (CLASP)}, a framework that dynamically fuses visual features and prunes tokens. 
In particular, we first introduce \textit{class-adaptive layer fusion} to optimize the visual representation. 
Instead of relying on a single ViT layer, this module adaptively integrates visual features from multiple layers, conditioned on the input instructions. 
This process preserves category-relevant local details while effectively removing redundant information. 
Building on these optimized representations, we implement a \textit{class-adaptive dual-stage pruning strategy}. 
The pruning process consists of a relevance-preserving stage to protect key evidence. 
It is followed by a coverage-oriented stage to maintain the necessary context while further reducing redundancy. 
We evaluate our method across a comprehensive suite of eight image benchmarks and three video benchmarks. 
The results demonstrate that our approach consistently surpasses existing state-of-the-art token pruning methods. 
Notably, our method preserves approximately $94.7\%$ of the original performance even with a token reduction of up to $88.9\%$. 
Besides, our method is plug-and-play and easily integrable into various MLLM architectures, which makes our framework well-suited for efficient practical deployment.

In summary, our main contributions are as follows:
\begin{itemize}
    \item We reveal the limitations of visual feature extraction from fixed ViT layers and the reliance on fixed attention or fixed similarity mechanisms.
    \item We propose class-adaptive layer fusion to extract more category-relevant visual representations.
    \item We introduce a class-adaptive dual-stage pruning strategy to balance the relevance and coverage requirements of visual tokens.
    \item We evaluate our method across various benchmarks, demonstrating consistent improvements in the accuracy--efficiency trade-off over existing approaches.
\end{itemize}

\section{Related Work}

\paragraph{Vision-language Models (VLMs).}
Multimodal large language models (MLLMs) enable visual reasoning by coupling vision encoders with LLMs through alignment interfaces~\citep{yin2024survey，dang2025exploring}, such as projectors~\cite{liu2023llava, liu2023improvedllava} or query-based connectors~\citep{li2023blip, instructblip}.
To enhance these capabilities, subsequent systems have focused on scaling vision backbones, training data, and instruction-tuning recipes~\citep{zhu2023minigpt, Qwen2-VL, chen2024internvl, zhu2025internvl3,yang2026annotation}.
However, as these models scale, the serialization of redundant visual inputs into long token sequences imposes a significant computational burden.
This burden becomes particularly acute when handling high-resolution images, where the resulting explosion in sequence length leads to prohibitive inference latency~\citep{liu2023llava, liu2024llava, LLaVA-OneVision-1.5}.
Beyond efficiency, current MLLMs still struggle with fine-grained visual perception, often leading to perception gaps and hallucinations~\cite{Li-hallucination-2023,yu2024rlaif,dang2025fuse}.
These perceptual shortcomings indicate a pressing need for more faithful visual representations and scalable context handling~\cite{zou2025don, wen2025stop}.
A primary structural bottleneck is that most projector-based models extract tokens from only a single, static and often late, vision-encoder layer.
By relying on such a fixed extraction point, these models fail to dynamically adapt the level of visual abstraction to the specific requirements of diverse instructions.

\paragraph{Visual Token Compression and Pruning.}
Visual token compression and pruning have emerged as pivotal techniques to alleviate the quadratic attention cost of long sequences, particularly as MLLMs increasingly handle high-resolution inputs. 
To achieve such efficient reduction, various methodologies have been developed, differing in their importance signals and optimization goals, such as instruction relevance versus redundancy. 
Attention-based approaches, like FastV~\cite{chen2024image} and VTW~\cite{lin2025boosting}, typically prune tokens using cross-modal attention patterns, while PyramidDrop~\cite{xing2024pyramiddrop} operates on intra-modal attention inside the vision backbone.
Instruction-guided methods like SparseVLM~\cite{zhang2024sparsevlm} further strengthen the role of textual instructions to condition these attention signals.
However, attention-based criteria are often noisy and tend to retain highly duplicated tokens.
To address this, similarity-driven methods (\textit{e.g.}, ToMe~\cite{bolya2022token}, VisionZip~\cite{yang2025visionzip}, HiRED~\cite{arif2025hired}, VisPruner~\cite{zhang2025vispruner}) leverage intrinsic visual cues to manage redundancy and maintain representation diversity.
Nonetheless, these approaches are fundamentally limited by a fixed-layer visual readout, which fails to provide the adaptive level of visual abstraction required by diverse tasks. 
Furthermore, they employ a rigid pruning policy that cannot balance instruction relevance and visual coverage according to the specific intent category.
In contrast, CLASP leverages the inferred instruction category to dynamically guide both multi-layer visual feature fusion and hybrid attention-similarity pruning.

\paragraph{ }

\section{Preliminary}
\label{sec:method}

\begin{figure*}[t]
    \centering
    \includegraphics[width=0.98\linewidth]{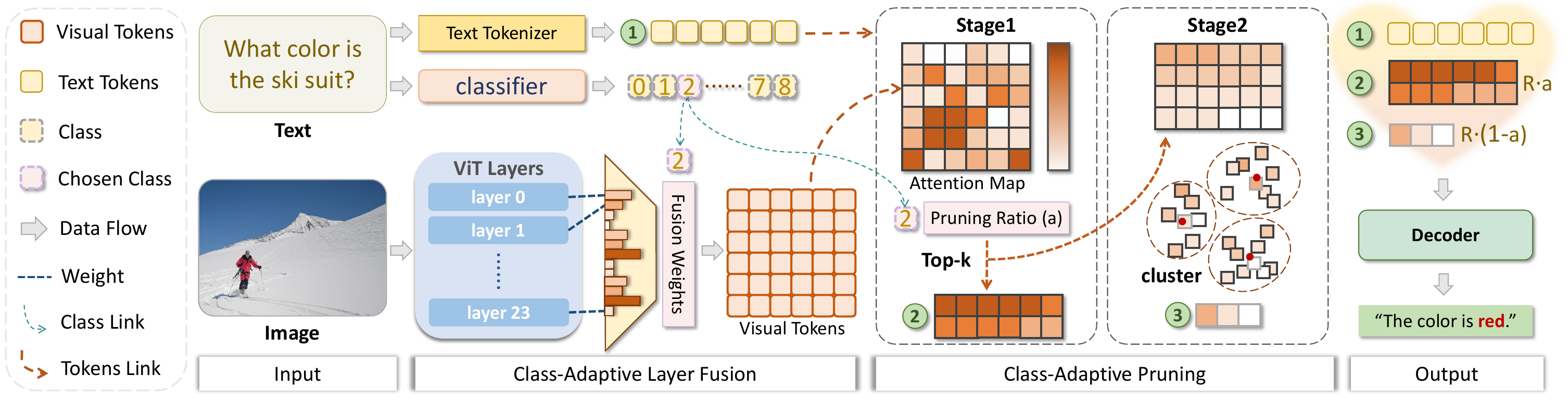}
    \caption{\textbf{Overview of our method.} The framework utilizes a prompt-to-class router to condition visual processing on textual intent. \textbf{(i) Class-Adaptive Layer Fusion:} ViT features are aggregated from multiple layers using class-specific mixture weights to capture an appropriate level of visual abstraction. \textbf{(ii) Class-Adaptive Pruning:} The projected visual token budget $R$ is dynamically split (ratio $a$) between attention-based selection (Stage 1) and similarity-based clustering (Stage 2) to balance instruction relevance with visual coverage.}
    \label{fig:motivation}
\end{figure*}

% \subsection{Preliminary}
% \label{sec:prelim}

\paragraph{Architecture of MLLMs.}
We consider a multimodal LLM consisting of (i) a vision encoder with $L$ layers,
(ii) a lightweight vision--text interface (\textit{e.g.}, a projector), and (iii) a causal text decoder.
Specifically, the vision encoder extracts visual features from input images, which are then aligned with the text embedding space to guide the generation process.
For sample $n$, let $\mathcal{V}_n$ denote the set of visual token indices. The vision encoder produces layer-wise \emph{patch token representations}
$\{\mathbf{Z}_n^{(l)}\}_{l=1}^L$, where $\mathbf{z}^{(l)}_{n,t}\in\mathbb{R}^{d_v}$ lies in the
vision-encoder embedding space:
\begin{equation}
\mathbf{Z}^{(l)}_n \triangleq \{\mathbf{z}^{(l)}_{n,t}\}_{t\in\mathcal{V}_n},\qquad l=1,\ldots,L.
\label{eq:prelim-vision}
\end{equation}
The interface module (\textit{e.g.}, a projector) maps the selected/fused visual representations into
the decoder embedding space, yielding \emph{decoder-space visual tokens} $\tilde{\mathbf{Z}}_n$,
which are concatenated with the text embeddings of the prompt $x_n$.
The decoder generates $\mathbf{y}_n$ autoregressively:
\begin{equation}
p(\mathbf{y}_n\mid x_n,\tilde{\mathbf{Z}}_n)=\prod_i p\!\left(y_{n,i}\mid y_{n,<i},x_n,\tilde{\mathbf{Z}}_n\right),
\end{equation}

\paragraph{Attention Mechanism for Token Relevance.}
Long dense visual token sequences are a primary efficiency bottleneck in MLLMs.
A common training-free strategy is to use attention weights as a proxy for token relevance and perform token reduction accordingly
(\textit{e.g.}, by retaining the most attended tokens)~\citep{zhang2024sparsevlm,wen2025stop,zhang2025vispruner}.
Given \emph{visual} token representations $\tilde{\mathbf{Z}}_n$ (as part of the input sequence) participating in an attention block
(\textit{e.g.}, a vision-encoder or multimodal-decoder layer), self-attention is:
\begin{equation}
\mathrm{Attn}(\mathbf{Q},\mathbf{K},\mathbf{V})
=\mathrm{softmax}\left(\frac{\mathbf{Q}\mathbf{K}^{\top}}{\sqrt{d_k}}\right)\mathbf{V},
\label{eq:prelim-attn}
\end{equation}
with $\mathbf{Q},\mathbf{K},\mathbf{V}$ denoting the query/key/value projections and $d_k$ the key dimension (per head).
Let $\mathbf{A}_n=\mathrm{softmax}(\mathbf{Q}\mathbf{K}^{\top}/\sqrt{d_k})$ be the corresponding attention matrix
(mean over heads in practice), where $\mathbf{A}_{n,i,t}$ denotes the attention mass assigned by query token $i$ to key token $t$.
To score a visual token $t\in\mathcal{V}_n$, we measure the attention it receives from a compact reference set $\mathcal{S}$
(\textit{e.g.}, instruction tokens or designated summary/query tokens):
\begin{equation}
\phi_{n,t}=\frac{1}{|\mathcal{S}|}\sum_{i\in\mathcal{S}}\mathbf{A}_{n,i,t},
\qquad t\in\mathcal{V}_n.
\label{eq:prelim-saliency}
\end{equation}
Given a retention budget $K$, attention-based selection keeps the top-$K$ highest-scoring tokens:
\begin{equation}
\mathcal{P}_n=\mathrm{Top}_{K}\bigl( \{ \phi_{n,t} \}_{t\in\mathcal{V}_n} \bigr),
\label{eq:prelim-attn-keep}
\end{equation}
where $\mathrm{Top}_{K}(\cdot)$ returns indices of the $K$ largest values.

\paragraph{Similarity Mechanism for Coverage.}
Attention is a strong relevance cue, yet attention-only pruning may retain many near-duplicate tokens and can be unstable under attention
shift/dispersion~\citep{zhang2025vispruner,wen2025stop}.
Therefore, several approaches incorporate similarity-based signals to ensure broader coverage~\citep{zhang2025vispruner}.

Given token vectors $\{\tilde{\mathbf{z}}_{n,t}\}_{t\in\mathcal{V}_n}$, define normalized features
$\mathbf{u}_{n,t}=\tilde{\mathbf{z}}_{n,t}/\|\tilde{\mathbf{z}}_{n,t}\|_2$ so that
$\mathrm{sim}(t,t')=\mathbf{u}_{n,t}^\top\mathbf{u}_{n,t'}$ is cosine similarity.
Let $\mathcal{P}_n$ be the attention-selected pivot set in Eq.~\eqref{eq:prelim-attn-keep}.
For any candidate token $t \in \mathcal{V}_n \setminus \mathcal{P}_n$,
we define its \emph{worst-case} redundancy with respect to the pivots as:
\begin{equation}
\rho_{n,t}=\max_{j\in\mathcal{P}_n}\mathbf{u}_{n,t}^\top \mathbf{u}_{n,j}.
\label{eq:prelim-redundancy}
\end{equation}
A smaller $\rho_{n,t}$ indicates that token $t$ is weakly covered by the pivot set (low cosine overlap)
and thus provides more complementary information, which is preferred for enhancing contextual coverage.

\section{Method}

\paragraph{Overview.}
Given a prompt $x_n$ and an input image $\mathbf{I}_n$, our goal is to reduce the \emph{effective} number of visual tokens processed by the multimodal decoder to a target budget $R$, while preserving instruction-relevant content and maintaining token coverage.
Following SparseVLM~\citep{zhang2024sparsevlm}, we adopt a three-stage progressive pruning schedule and sparsify the visual token set at intermediate decoder layers $\{2,6,15\}$.
The stage-wise retention budgets are configured to match the target effective budget $R$.
As illustrated in Figure~\ref{fig:motivation}, our method introduces two class-adaptive components driven by a prompt-to-class router:
(i) \emph{class-adaptive layer fusion}, which forms a category-conditioned mixture of multi-layer vision representations, and
(ii) \emph{class-adaptive pruning}, which allocates each stage budget between attention-salient pivots (relevance) and low-redundancy completion tokens (coverage).
The complete procedure is summarized in Algorithm~\ref{alg:cap}.

% =========================
% Main paper (Sec. 4): Method
% Replace/overwrite Sec. 4.1 and Sec. 4.2 with the following LaTeX.
% =========================

\subsection{Class-Adaptive Layer Fusion}
\label{sec:class-fusion}

\paragraph{Motivation and Design Principle.}
Vision backbones in MLLMs (\textit{e.g.}, ViT-style encoders) exhibit a well-known depth hierarchy:
shallower layers preserve local texture, edges, and fine spatial layouts, while deeper layers become increasingly semantic and invariant.
Our study (Sec.~\ref{sec:intro}, Fig.~\ref{fig:motivation-1}) indicates that \emph{the optimal feature granularity for pruning depends on the instruction category}:
detail-centric instructions (\textit{e.g.}, OCR, counting) benefit from retaining mid/shallow cues, whereas holistic instructions (\textit{e.g.}, scene understanding) prefer deeper abstractions.
Therefore, instead of extracting visual representations from a fixed depth, we construct a prompt-conditioned depth mixture that adaptively balances fine-grained and semantic representations.

\paragraph{Prompt-to-class Routing.}
We assume a discrete category space $\mathcal{C}=\{0,1,\ldots,C-1\}$ defined by a text taxonomy (\textit{e.g.}, benchmark question types or a user-defined schema).
A lightweight router $\mathrm{Route}(\cdot)$ is employed to map the input prompt $x_n$ to a category index:
\begin{equation}
c_n = \mathrm{Route}(x_n),\qquad c_n\in\mathcal{C}.
\label{eq:router}
\end{equation}
The text-only router runs before token reduction, incurring negligible overhead compared to multimodal decoding.

\paragraph{Category-conditioned Layer Mixture.}
To encode category-specific layer preferences, we maintain a score matrix $\mathbf{W}\in\mathbb{R}^{C\times L}$, 
where the $c$-th row $\mathbf{w}_c$ assigns unnormalized importance to each vision layer (detailed in~\ref{sec:appendix-method}).
For sample $n$, we convert $\mathbf{w}_{c_n}$ to mixture weights via a temperature-controlled softmax $\boldsymbol{\alpha}_n=\mathrm{softmax}\!\left(\tau\,\mathbf{w}_{c_n}\right)$:
\begin{equation}
\boldsymbol{\alpha}_n=\{\alpha_{n,l}\}_{l=1}^{L},\ \sum_{l=1}^{L}\alpha_{n,l}=1,
\label{eq:layer-mix}
\end{equation}
where $\tau>0$ controls the sharpness of the mixture (small $\tau$: uniform averaging; large $\tau$: near one-hot selection).
This formulation strictly generalizes fixed-layer heuristics: choosing a single layer corresponds to a one-hot $\boldsymbol{\alpha}_n$.

\paragraph{Token-wise Fusion.}
Let $\mathbf{Z}_n^{(l)}=\{\mathbf{z}^{(l)}_{n,t}\}_{t\in\mathcal{V}_n}$ be layer-wise \emph{patch representations} as in Eq.~\eqref{eq:prelim-vision}.
Directly mixing intermediate activations across layers can be sensitive to scale differences.
We then fuse multi-layer \emph{features} via adaptive weights by a convex combination:
\begin{equation}
\bar{\mathbf{Z}}_n \triangleq \{\bar{\mathbf{z}}_{n,t}\}_{t\in\mathcal{V}_n},
\qquad
\bar{\mathbf{z}}_{n,t}=\sum_{l=1}^{L}\alpha_{n,l}\,{\mathbf{z}}^{(l)}_{n,t}.
\label{eq:fusion}
\end{equation}
The convexity is deliberate: it provides stability and interpretability (see Appendix~\ref{sec:appendix-theory} for formal properties).
Finally, the original frozen MLP projector is employed to map the fused features into the decoder embedding space:
\begin{equation}
\tilde{\mathbf{Z}}_n = f_{\text{proj}}(\bar{\mathbf{Z}}_n),
\qquad
\tilde{\mathbf{Z}}_n=\{\tilde{\mathbf{z}}_{n,t}\}_{t\in\mathcal{V}_n},\ \tilde{\mathbf{z}}_{n,t}\in\mathbb{R}^{d}.
\label{eq:proj}
\end{equation}
While we assume ViT-style token alignment, hierarchical backbones are supported via a simple token-alignment mapping to a common grid, and fusion adds only an $O(L\,|\mathcal{V}_n|\,d_v)$ overhead (Appendix~\ref{app:hier-fusion} and Appendix~B.6).

\begin{algorithm}[t]
\caption{Class-adaptive layer fusion and pruning.}
\label{alg:cap}
\begin{algorithmic}[1]
\REQUIRE prompt $x_n$; layer-wise tokens $\{\mathbf{Z}^{(l)}_n\}_{l=1}^{L}$ with indices $\mathcal{V}_n$; final budget $R$;
progressive schedule $\{R^{(2)},R^{(6)},R^{(15)}\}$;
router $\mathrm{Route}(\cdot)$; layer-score matrix $\mathbf{W}$; split ratios $\mathbf{a}$;
temperature $\tau$; projector $f_{\text{proj}}$; reference set $\mathcal{S}$
\STATE $c_n \leftarrow \mathrm{Route}(x_n)$
\STATE $\boldsymbol{\alpha}_n \leftarrow \mathrm{softmax}(\tau\,\mathbf{w}_{c_n})$
\STATE $\tilde{\mathbf{Z}}_n \leftarrow f_{\text{proj}}\!\Big(\sum_{l=1}^{L}\alpha_{n,l}\,{\mathbf{Z}}^{(l)}_n\Big)$
\STATE $\mathcal{V} \leftarrow \mathcal{V}_n$
\FOR{$(l_s,R_s)\in\{(2,R^{(2)}),(6,R^{(6)}),(15,R^{(15)})\}$}
    \STATE Compute relevance $\{\phi_{n,t}\}_{t\in\mathcal{V}}$ at layer $l_s$
    \STATE $K_1 \leftarrow \lfloor a_{c_n}R_s \rfloor$;\quad $K_2 \leftarrow R_s-K_1$
    \STATE $\mathcal{P} \leftarrow \mathrm{Top}_{K_1}(\{\phi_{n,t}\}_{t\in\mathcal{V}})$ \hfill (Eq.~\eqref{eq:pivots})
    \STATE Initialize centers using $\mathrm{Bottom}_{K_2}(\{\rho_{n,t}(\mathcal{P})\})$ \hfill (Eq.~\eqref{eq:redundancy-inst})
    \STATE Refine centers via Spherical $K$-Means ($T$ iters) \hfill (Eqs.~\eqref{eq:kmeans-assign}, \eqref{eq:kmeans-update})
    \STATE $\mathcal{Q} \leftarrow$ Select cluster medoids \hfill (Eqs.~\eqref{eq:diverse-q}, \eqref{eq:diverse-set})
    \STATE $\mathcal{V} \leftarrow \mathcal{P}\cup\mathcal{Q}$
\ENDFOR
\STATE \textbf{return} $\{\tilde{\mathbf{z}}_{n,t}\}_{t\in\mathcal{V}}$
\end{algorithmic}
\end{algorithm}

\subsection{Class-Adaptive Pruning}
\label{sec:class-prune}

Given aligned visual tokens $\tilde{\mathbf{Z}}_n=\{\tilde{\mathbf{z}}_{n,t}\}_{t\in\mathcal{V}_n}$, we retain exactly $R$ tokens. 
Conceptually, we want a subset that (i) preserves instruction-critical evidence (\emph{relevance}) and (ii) avoids wasting budget on redundant patches (\emph{coverage}). 
A single scalar scoring rule cannot reliably satisfy both criteria across heterogeneous instruction types. 
We therefore instantiate pruning as an explicit \emph{two-stage} procedure, and make the relevance--coverage budget split category-dependent.

\paragraph{Category-dependent Split Ratio.}
We maintain a class-wise ratio vector $\mathbf{a}\in[0,1]^C$ and set $a_n\triangleq a_{c_n}$. 
Intuitively, large $a_n$ allocates more slots to attention-salient tokens (useful for detail-centric prompts), 
while small $a_n$ allocates more slots to coverage completion (useful for holistic prompts). 
Given total budget $R$, we compute:
\begin{equation}
K_1=\left\lfloor a_n R\right\rfloor,\qquad K_2=R-K_1.
\label{eq:split}
\end{equation}

\paragraph{Stage I: Attention Pivots for Relevance Preservation.}
We select $K_1$ \emph{pivot} tokens that maximize total relevance:
\begin{equation}
\mathcal{P}_n=\mathrm{Top}_{K_1}\bigl(\{\phi_{n,t}\}_{t\in\mathcal{V}_n}\bigr).
\label{eq:pivots}
\end{equation}
This stage is intentionally aggressive in preserving query- or summary-attended evidence and acts as an ``anchor'' set that protects small but critical local visual regions (\textit{e.g.}, text glyphs, counting targets). 
Appendix~\ref{sec:appendix-theory} formalizes the optimality of $\mathrm{Top}_{K_1}$ for additive relevance.

\paragraph{Stage II: Coverage Completion via Redundancy-aware Clustering.}
Attention pivots alone can be redundant: multiple high-attention tokens may correspond to near-identical patches. 
To complement them, we allocate $K_2$ slots to a clustering-based completion stage that explicitly improves coverage, consistent with the overview depicted in Fig.~2. 
Let us define $\mathcal{U}_n=\mathcal{V}_n\setminus\mathcal{P}_n$ and the normalized features $\mathbf{u}_{n,t}=\tilde{\mathbf{z}}_{n,t}/\|\tilde{\mathbf{z}}_{n,t}\|_2$. 
We first measure redundancy of a candidate token in the pivot set as:
\begin{equation}
\rho_{n,t}(\mathcal{P}_n)=\max_{j\in\mathcal{P}_n}\mathbf{u}_{n,t}^\top\mathbf{u}_{n,j},
\qquad t\in\mathcal{U}_n.
\label{eq:redundancy-inst}
\end{equation}

Instead of directly taking $\mathrm{Bottom}_{K_2}$, we use $\rho_{n,t}$ to construct a deterministic, redundancy-minimizing initialization for the clustering phase: we choose the $K_2$ least redundant tokens as seeds and set their features as initial centers. We then run $T$ iterations of spherical $K$-means~\citep{hornik2012spherical} on $\{\mathbf{u}_{n,t}\}_{t\in\mathcal{U}_n}$ computed with cosine similarity:
\begin{align}
s^{(r)}_{n,t} &= \arg\max_{k\in\{1,\dots,K_2\}}\mathbf{u}_{n,t}^\top\boldsymbol{\mu}^{(r-1)}_{n,k}, \label{eq:kmeans-assign} \\[1ex]
\boldsymbol{\mu}^{(r)}_{n,k} &= \frac{\sum_{t:s^{(r)}_{n,t}=k}\mathbf{u}_{n,t}}{\left\|\sum_{t:s^{(r)}_{n,t}=k}\mathbf{u}_{n,t}\right\|_2}, \label{eq:kmeans-update}
\end{align}
for $r=1,\dots,T$. Finally, we select one representative (medoid) token per resulting cluster:
\begin{align}
q_{n,k} &= \arg\max_{t:s^{(T)}_{n,t}=k}\mathbf{u}_{n,t}^\top\boldsymbol{\mu}^{(T)}_{n,k}, \label{eq:diverse-q} \\
\mathcal{Q}_n &= \{q_{n,k}\}_{k=1}^{K_2}. \label{eq:diverse-set}
\end{align}
Geometrically, the redundancy-aware seeding enforces angular separation from the pivot set (low $\rho_{n,t}$), while the clustering refinement reduces mutual duplication among completion tokens by encouraging them to cover multiple modes on the unit sphere. 
We use $T{=}5$ iterations by default following the ablation in Table~13.

\subsection{Retained Tokens and Decoder Inference.}
The final retained index set is constructed as the union $\mathcal{P}_n\cup\mathcal{Q}_n$, thereby yielding a total of exactly $R$ tokens:
\begin{equation}
\{\tilde{\mathbf{z}}_{n,t}\mid t\in \mathcal{P}_n\cup\mathcal{Q}_n\},
\qquad \bigl|\mathcal{P}_n\cup\mathcal{Q}_n\bigr|=R.
\label{eq:keep}
\end{equation}
We feed the retained visual tokens as a compact prefix to the decoder. Conditioned on $x_n$ and the retained tokens, the decoder generates the response $\mathbf{y}_n$ autoregressively:
\begin{equation}
\begin{aligned}
p\bigl(\mathbf{y}_n \mid x_n,&\{\tilde{\mathbf{z}}_{n,t}\}_{t\in \mathcal{P}_n\cup\mathcal{Q}_n}\bigr) \\
&=\prod_i p\Bigl(y_{n,i}\mid y_{n,<i},x_n, \{\tilde{\mathbf{z}}_{n,t}\}_{t\in \mathcal{P}_n\cup\mathcal{Q}_n}\Bigr).
\end{aligned}
\label{eq:decode}
\end{equation}
This hybrid formulation ensures that the language decoder receives a compressed yet representative visual context, where $\mathcal{P}_n$ safeguards task-critical \emph{relevance} while $\mathcal{Q}_n$ maintains global representational \emph{coverage}.

\section{Experiment}

% 备注：请确保已定义 MidnightBlue（例如：\usepackage[dvipsnames]{xcolor}）
\begin{table*}[t]
    \centering
    \footnotesize
    \caption{Performance comparison of various methods on LLaVA-v1.5-7B across different benchmarks. Results are shown for different pruning ratios, with accuracy and average performance highlighted. Best results in \textcolor{MidnightBlue}{\textbf{blue}}.}
     \label{tab1:main_table1}
    \resizebox{0.9\linewidth}{!}{
    \begin{tabular}{l | *{7}{>{\centering\arraybackslash}p{0.92cm}} |>{\centering\arraybackslash}p{1.15cm}}
        \midrule
        \textbf{\;Methods}
        & \textbf{GQA}
        & \textbf{MMB}
        & \textbf{MME}
        & \textbf{POPE}
        & \textbf{SQA}
        & \textbf{VQA}$_{\text{V2}}$
        & \textbf{VQA}$_{\text{Text}}$
        & \makecell[c]{\textbf{Average}}\\
        \midrule
        
        \textcolor{gray}{Upper Bound, 576 Tokens} & \textcolor{gray}{61.9} & \textcolor{gray}{64.7} & \textcolor{gray}{1862} & \textcolor{gray}{85.9} & \textcolor{gray}{69.5} & \textcolor{gray}{78.4} & \textcolor{gray}{58.2} & \multirow{1}*{\textcolor{gray}{100.0\%}} \\
        \midrule

        \rowcolor{mygray}
        LLaVA-1.5 \textcolor{gray}{7B} & \multicolumn{8}{c}{\textit{Retain 192 Tokens} \ $\fg{(\downarrow 66.7\%)}$}\\
        ToMe \texttt{\scriptsize{(ICLR23)}}\citep{bolya2022token} & 54.3 & 60.5 & 1563 & 72.4 & 65.2 & 68.0 & 52.1 & 88.5\% \\
        FastV \texttt{\scriptsize{(ECCV24)}}\citep{chen2024image} & 52.7 & 61.2 & 1612 & 64.8 & 67.3 & 67.1 & 52.5 & 87.8\% \\
        MustDrop \texttt{\scriptsize{(2024.11)}}\citep{liu2024multi} & 58.2 & 62.3 & 1787 & 82.6 & 69.2 & 76.0 & 56.5 & 96.6\% \\
        LLaVA-PruMerge \texttt{\scriptsize{(ICCV25)}}\citep{shang2025llava}\;\; & 54.3 & 59.6 & 1632 & 71.3 & 67.9 & 70.6 & 54.3 & 90.2\% \\
        PDrop \texttt{\scriptsize{(CVPR25)}}\citep{xing2024pyramiddrop} & 57.1 & 63.2 & 1766 & 82.3 & 68.8 & 75.1 & 56.1 & 96.0\% \\
        FiCoCo-V \texttt{\scriptsize{(2025.03)}}\citep{han2025filtercorrelatecompresstrainingfree} & 58.5 & 62.3 & 1732 & 82.5 & 67.8 & 74.4 & 55.7 & 95.4\% \\
        HiRED \texttt{\scriptsize{(AAAI25)}}\citep{arif2025hired} & 58.7 & 62.8 & 1737 & 82.8 & 68.4 & 74.9 & 47.4 & 93.9\%     \\
        VisionZip \texttt{\scriptsize{(CVPR25)}}\citep{yang2025visionzip} 
            & 59.3 
            & \textcolor{MidnightBlue}{\textbf{64.5}} 
            & 1767 
            & \textcolor{MidnightBlue}{\textbf{86.4}} 
            & 68.9 
            & 76.8 
            & 57.3 
            & 98.1\%  \\
        SparseVLM \texttt{\scriptsize{(ICML25)}}\citep{zhang2024sparsevlm} & 57.6 & 62.5 & 1721 & {83.6} & 69.1 & 75.6 & 56.1 & 95.9\% \\
        DART \texttt{\scriptsize{(EMNLP25)}}\citep{wen2025stop} 
            & 58.9 
            & 63.6 
            & \textcolor{MidnightBlue}{\textbf{1856}} 
            & 82.8 
            & \textcolor{MidnightBlue}{\textbf{69.8}} 
            & 76.7 
            & 57.4
            & 98.1\% \\
        % HoloV \scriptsize{(NeurIPS25)}\citep{zou2025don} & 59.0 & \textcolor{MidnightBlue}{\textbf{65.4}} & {1820} & \textcolor{MidnightBlue}{\textbf{85.6}} & \textcolor{MidnightBlue}{\textbf{69.8}} & 76.7 & \textcolor{MidnightBlue}{\textbf{57.4}} & \textcolor{MidnightBlue}{\textbf{98.7\%}} \\
        \rowcolor{mygreen2}
        CLASP (ours) 
            & \textcolor{MidnightBlue}{\textbf{60.4}}
            & 61.3
            & 1848
            & 85.6
            & 69.6
            & \textcolor{MidnightBlue}{\textbf{77.1}}
            & \textcolor{MidnightBlue}{\textbf{57.6}}
            & \textcolor{MidnightBlue}{\textbf{98.4\%}} \\
        
        \midrule

        \rowcolor{mygray}
        LLaVA-1.5 \textcolor{gray}{7B} & \multicolumn{8}{c}{\textit{Retain 128 Tokens} \ $\fg{(\downarrow 77.8\%)}$}\\
        ToMe \texttt{\scriptsize{(ICLR23)}}\citep{bolya2022token} & 52.4 & 53.3 & 1343 & 62.8 & 59.6 & 63.0 & 49.1 & 80.4\% \\
        FastV \texttt{\scriptsize{(ECCV24)}}\citep{chen2024image} & 49.6 & 56.1 & 1490 & 59.6 & 60.2 & 61.8 & 50.6 & 81.2\%\\
        MustDrop \texttt{\scriptsize{(2024.11)}}\citep{liu2024multi} & 56.9 & 61.1 & 1745 & 78.7 & 68.5 & 74.6 & 56.3 & 94.6\% \\
        LLaVA-PruMerge \texttt{\scriptsize{(ICCV25)}}\citep{shang2025llava} & 53.3 & 58.1 & 1554 & 67.2 & 67.1 & 68.8 & 54.3 & 87.9\%  \\
        PDrop \texttt{\scriptsize{(CVPR25)}}\citep{xing2024pyramiddrop} & 56.0 & 61.1 & 1644 & {82.3} & 68.3 & 72.9 & 55.1 & 93.6\% \\
        FiCoCo-V \texttt{\scriptsize{(2025.03)}}\citep{han2025filtercorrelatecompresstrainingfree} & 57.6 & 61.1 & 1711 & 82.2 & 68.3 & 73.1 & 55.6 & 94.6\% \\
        HiRED \texttt{\scriptsize{(AAAI25)}}\citep{arif2025hired} & 57.2 & 61.5 & 1710 & 79.8 & 68.1 & 73.4 & 46.1 & 91.9\% \\
        VisionZip \texttt{\scriptsize{(CVPR25)}}\citep{yang2025visionzip} 
            & 57.6 
            & \textcolor{MidnightBlue}{\textbf{63.4}} 
            & 1768 
            & 84.7 
            & 68.8 
            & 75.6 
            & \textcolor{MidnightBlue}{\textbf{56.8}} 
            & 96.8\% \\
        SparseVLM \texttt{\scriptsize{(ICML25)}}\citep{zhang2024sparsevlm} & 56.0 & 60.0 & 1696 & 80.5 & 67.1 & 73.8 & 54.9 & 93.3\% \\
        DART \texttt{\scriptsize{(EMNLP25)}}\citep{wen2025stop} 
            & 57.9 
            & 63.2 
            & \textcolor{MidnightBlue}{\textbf{1845}} 
            & 80.1 
            & \textcolor{MidnightBlue}{\textbf{69.1}} 
            & 75.9 
            & 56.4 
            & 96.7\% \\
        % HoloV \texttt{\scriptsize{(NeurIPS25)}}\citep{zou2025don} & 57.7 & \textcolor{MidnightBlue}{\textbf{63.9}} & 1802 & \textcolor{MidnightBlue}{\textbf{84.0}} & \textcolor{MidnightBlue}{\textbf{69.8}} & 75.5 & \textcolor{MidnightBlue}{\textbf{56.8}} & \textcolor{MidnightBlue}{\textbf{97.3\%}} \\
        \rowcolor{mygreen2}
        CLASP (ours) 
            & \textcolor{MidnightBlue}{\textbf{58.9}}
            & 60.7
            & 1790
            & \textcolor{MidnightBlue}{\textbf{85.2}}
            & 69.0
            & \textcolor{MidnightBlue}{\textbf{76.7}}
            & 56.7
            & \textcolor{MidnightBlue}{\textbf{97.0\%}} \\
        
        \midrule

        \rowcolor{mygray}
        LLaVA-1.5 \textcolor{gray}{7B} & \multicolumn{8}{c}{\textit{Retain 64 Tokens} \ $\fg{(\downarrow 88.9\%)}$}\\
        ToMe \texttt{\scriptsize{(ICLR23)}}\citep{bolya2022token} & 48.6 & 43.7 & 1138 & 52.5 & 50.0 & 57.1 & 45.3 & 70.1\%\\
        FastV \texttt{\scriptsize{(ECCV24)}}\citep{chen2024image} & 46.1 & 48.0 & 1256 & 48.0 & 51.1 & 55.0 & 47.8 & 71.1\% \\
        MustDrop \texttt{\scriptsize{(2024.11)}}\citep{liu2024multi} & 53.1 & 60.0 & 1612 & 68.0 & 63.4 & 69.3 & 54.2 & 88.1\%    \\
        LLaVA-PruMerge \texttt{\scriptsize{(ICCV25)}}\citep{shang2025llava} & 51.9 & 55.3 & 1549 & 65.3 & 68.1 & 67.4 & 54.0 & 86.5\% \\
        PDrop \texttt{\scriptsize{(CVPR25)}}\citep{xing2024pyramiddrop} & 41.9 & 33.3 & 1092 & 55.9 & 68.6 & 69.2 & 45.9 & 72.7\% \\
        FiCoCo-V \texttt{\scriptsize{(2025.03)}}\citep{han2025filtercorrelatecompresstrainingfree} & 52.4 & 60.3 & 1591 & {76.0} & 68.1 & 71.3 & 53.6 & 90.4\% \\
        HiRED \texttt{\scriptsize{(AAAI25)}}\citep{arif2025hired} & 54.6 & 60.2 & 1599 & 73.6 & 68.2 & 69.7 & 44.2 & 88.0\% \\
        VisionZip \texttt{\scriptsize{(CVPR25)}}\citep{yang2025visionzip} 
            & 55.1 
            & 60.1 
            & 1690 
            & 77.0 
            & 69.0 
            & 72.4 
            & \textcolor{MidnightBlue}{\textbf{55.5}} 
            & 92.8\% \\
        SparseVLM \texttt{\scriptsize{(ICML25)}}\citep{zhang2024sparsevlm} & 52.7 & 56.2 & 1505 & 75.1 & 62.2 & 68.2 & 51.8 & 86.5\% \\
        DART \texttt{\scriptsize{(EMNLP25)}}\citep{wen2025stop} 
            & 55.9 
            & \textcolor{MidnightBlue}{\textbf{60.6}} 
            & \textcolor{MidnightBlue}{\textbf{1765}} 
            & 73.9 
            & \textcolor{MidnightBlue}{\textbf{69.8}} 
            & 72.4 
            & 54.4 
            & 93.0\% \\
        % HoloV \scriptsize{(NeurIPS25)}\citep{zou2025don} & 55.3 & \textcolor{MidnightBlue}{\textbf{63.3}} & 1715 & \textcolor{MidnightBlue}{\textbf{80.3}} & 69.5 & \textcolor{MidnightBlue}{\textbf{72.8}} & 55.4 & \textcolor{MidnightBlue}{\textbf{94.4\%}} \\
        \rowcolor{mygreen2}
        CLASP (ours) 
            & \textcolor{MidnightBlue}{\textbf{57.0}}
            & 59.1
            & 1709
            & \textcolor{MidnightBlue}{\textbf{82.8}}
            & \textcolor{MidnightBlue}{\textbf{69.8}}
            & \textcolor{MidnightBlue}{\textbf{75.2}}
            & 55.2
            & \textcolor{MidnightBlue}{\textbf{94.7\%}} \\

        \midrule
	\end{tabular}
    }
\end{table*}

\textbf{Benchmarks.}
We evaluate multimodal capability on ten established benchmarks spanning general visual understanding, compositional reasoning, OCR-centric reasoning, real-world robustness, object hallucination, and video comprehension. 
For image understanding, we use 
GQA~\citep{hudson2019gqa}, 
MMBench~\citep{liu2024mmbench}, 
MME~\citep{fu2025mme}, 
POPE~\citep{Li-hallucination-2023}, 
VQAv2~\citep{goyal2017making}, 
ScienceQA~\citep{lu2022learn}, 
and TextVQA~\citep{singh2019towards}. 
To assess temporal reasoning and video capability, we further include TGIF~\citep{li2016tgif} for animated GIF description, MSVD~\citep{chen2011collecting} for video captioning, and MSRVTT~\citep{xu2016msr} for open-domain video description.
We follow the official dataset splits and evaluation protocols, reporting the standard metrics for each benchmark to ensure reproducibility and fair comparison.

\textbf{Models and Comparison Methods.}
We apply our method to various MLLM architectures, including the LLaVA
series: LLaVA-1.5~\citep{liu2023improvedllava} for image understanding, LLaVA-NeXT~\citep{liu2024llava}  for high-resolution inputs, and Video-LLaVA~\citep{zhang2024videoinstructiontuningsynthetic} for video understanding, as well as the open-source model Qwen2.5-VL~\citep{Qwen2.5-VL}. 
For efficiency-oriented comparisons, we benchmark against state-of-the-art token reduction methods for MLLMs, including ToMe~\citep{bolya2022token}, LLaVA-PruMerge~\citep{shang2025llava}, FastV~\citep{chen2024image}, HiRED~\citep{arif2025hired}, PDrop~\citep{xing2024pyramiddrop}, 
Multi-Stage Vision Token Dropping~\citep{liu2024multi}, SparseVLM~\citep{zhang2024sparsevlm}, VisionZip~\citep{yang2025visionzip}, and DART~\citep{wen2025stop}.
More details regarding the model architectures and additional experimental results are provided in Appendix~\ref{Detailed_experiment_settings} and Appendix~\ref{sec:appendix-method-2}.
% \textcolor{red}{modified}

\begin{table*}[t]
    \centering
    \footnotesize
    \caption{Performance comparison of various methods on LLaVA-NeXT-7B across different benchmarks. Results are shown for different pruning ratios, with accuracy and average performance highlighted. Best results in \textcolor{MidnightBlue}{\textbf{blue}}.}
    \label{tab2:main}

    \resizebox{0.9\linewidth}{!}{%
    \begin{tabular}{l !{\vrule width 0.6pt} *{7}{>{\centering\arraybackslash}p{0.92cm}} !{\vrule width 0.6pt} >{\centering\arraybackslash}p{1.15cm}}
        \toprule
        \textbf{Methods} & \textbf{GQA} & \textbf{MMB} & \textbf{MME} & \textbf{POPE} & \textbf{SQA} & \textbf{VQA}$_{\text{V2}}$ & \textbf{VQA}$_{\text{Text}}$ & \makecell[c]{\textbf{Average}} \\
        \midrule
        
        \textcolor{gray}{Upper Bound, 2880 Tokens}
        & \textcolor{gray}{64.2} & \textcolor{gray}{67.4} & \textcolor{gray}{1851} & \textcolor{gray}{86.5}
        & \textcolor{gray}{70.1} & \textcolor{gray}{81.8} & \textcolor{gray}{64.9} & \textcolor{gray}{100.0\%} \\
        \midrule

        \rowcolor{mygray}
        LLaVA-NeXT \textcolor{gray}{7B} & \multicolumn{8}{c}{\textit{Retain 320 Tokens} \ $\fg{(\downarrow 88.9\%)}$} \\
        
        FastV \texttt{\scriptsize{(ECCV24)}}\citep{chen2024image} 
        & 55.9 & 61.6 & 1661 & 71.7 & 62.8 & 71.9 & 55.7 & 88.0\% \\
        LLaVA-PruMerge \texttt{\scriptsize{(ICCV25)}}\citep{shang2025llava} 
        & 53.6 & 61.3 & 1534 & 60.8 & 66.4 & 69.7 & 50.6 & 85.6\% \\
        PDrop \texttt{\scriptsize{(CVPR25)}}\citep{xing2024pyramiddrop} 
        & 56.4 & 63.4 & 1663 & 77.6 & 67.5 & 73.5 & 54.4 & 90.9\% \\
        MustDrop \texttt{\scriptsize{(2024.11)}}\citep{liu2024multi} 
        & 57.3 & 62.8 & 1641 & 82.1 & 68.0 & 73.7 & 59.9 & 92.2\% \\
        FasterVLM\texttt{\scriptsize{(ICCV25)}}\citep{zhang2024fastervlm}  
        & 56.9 & 61.6 & 1701 & 83.6 & 66.5 & 74.0 & 56.5 & 91.1\% \\
        HiRED \texttt{\scriptsize{(AAAI25)}}\citep{arif2025hired} 
        & 59.3 & 64.2 & 1690 & 83.3 & 66.7 & 75.7 & 58.8 & 93.3\% \\
        SparseVLM \texttt{\scriptsize{(ICML25)}}\citep{zhang2024sparsevlm} 
        & 56.1 & 60.6 & 1533 & 82.4 & 66.1 & 71.5 & 58.4 & 89.6\% \\
        % GlobalCom$^2$ \texttt{\scriptsize{(2025.3)}} 
        % & 57.1 & 61.8 & 1698 & 83.8 & 67.4 & 76.7 & 57.2 & 92.2\% \\
        DART \texttt{\scriptsize{(EMNLP25)}}\citep{wen2025stop} 
        & 61.7 & \textcolor{MidnightBlue}{\textbf{65.3}} & 1710 & 84.1 & \textcolor{MidnightBlue}{\textbf{68.4}} & \textcolor{MidnightBlue}{\textbf{79.1}} & 58.7 & 93.9\% \\
        
        \rowcolor{mygreen2}
        CLASP (ours)
        & \textcolor{MidnightBlue}{\textbf{62.7}}
        & 61.2
        & \textcolor{MidnightBlue}{\textbf{1723}}
        & \textcolor{MidnightBlue}{\textbf{85.8}}
        & 67.0
        & 78.0
        & \textcolor{MidnightBlue}{\textbf{61.7}}
        & \textcolor{MidnightBlue}{\textbf{95.2\%}} \\
        
        \bottomrule
    \end{tabular}%
    }
\end{table*}

\textbf{Main Results.}
In Table~\ref{tab1:main_table1}, we compare our method with representative token merging and dropping methods on LLaVA-1.5-7B across seven diverse and challenging image-understanding benchmarks.
For a fair comparison, we report the raw task scores and additionally normalize the vanilla 576-token model to a 100\% upper bound, using the resulting normalized average to compare methods.
Across three retention budgets (192/128/64 tokens; $\downarrow$66.7\%/$\downarrow$77.8\%/$\downarrow$88.9\%), our method consistently achieves the best normalized average (98.4\%/97.0\%/94.7\%).
With 192 tokens, we retain 98.4\% overall and obtain the strongest GQA result (60.4) together with competitive performance on the remaining benchmarks, while improving over classical reduction baselines by large margins (\textit{e.g.}, +9.9\% and +10.6\% over ToMe~\citep{bolya2022token} and FastV~\citep{chen2024image}).
At 128 tokens, we maintain 97.0\% overall and reach the best POPE accuracy (85.2), suggesting improved faithfulness under tight budgets.
Even with 64 tokens, our method preserves 94.7\% overall, outperforming the similarity-based method DART~\citep{wen2025stop} by +1.7\%, while further extending the margin over attention-based methods (+22.0\% over PDrop~\citep{xing2024pyramiddrop} and +23.6\% over FastV~\citep{chen2024image}).
Overall, these results show that our method preserves critical visual evidence more effectively, especially in the extreme low-token setting.

\textbf{Main Results on Higher Resolution.}
For further comprehensive evaluation, we report results on LLaVA-NeXT 7B when reducing the visual tokens from 2880 to 320 (a $\,\downarrow 88.9\%$ token reduction) in Table~\ref{tab2:main}.
Following the same evaluation protocol as Table~\ref{tab1:main_table1}, we report raw task scores and a normalized average where the 2880-token model is set to 100\% for cross-benchmark comparison.
Under this aggressive budget, our method achieves the best normalized average of 95.2\%, improving over the similarity-based pruning method DART~\citep{wen2025stop} by 1.3\%, and outperforming recent competitors such as HiRED~\citep{arif2025hired} by an even larger margin (+1.9\%).
In addition to the overall ranking, our method attains the best scores on four out of seven benchmarks, including GQA (62.7), MME (1723), POPE (85.8), and TextVQA (61.7), indicating stronger retention of task-critical visual evidence under severe token reduction.
These results show that our pruning strategy scales well to higher-resolution token settings and remains highly effective, preserving accuracy while substantially shrinking the visual context length.

\textbf{Results on Qwen Architecture.}
Beyond LLaVA, we extend our method to Qwen2.5-VL-7B~\citep{Qwen2.5-VL} to validate its generalizability, evaluating it under three token pruning rates (66.7\%, 77.8\%, and 88.9\%). 
As shown in Table~\ref{tab:Qwen}, our method consistently retains a larger fraction of the upper-bound performance than SparseVLM~\citep{zhang2024sparsevlm} at every budget, and the advantage widens as pruning becomes more aggressive. 
We improve the normalized average from 94.1\% to 96.5\% at 66.7\% pruning (+2.4\%), from 90.8\% to 94.4\% at 77.8\% pruning (+3.6\%), and from 82.9\% to 89.0\% at 88.9\% pruning (+6.1\%). 
These results show that our pruning strategy transfers effectively to Qwen2.5-VL. 
The gains are most pronounced in high-compression settings, where preserving diverse yet instruction-relevant visual evidence is critical, highlighting strong generalizability on advanced MLLM architectures.
\renewcommand{\multirowsetup}{\centering}
\definecolor{mygray}{gray}{.92}
\definecolor{mygreen1}{RGB}{253, 244, 244}%255, 234, 234
\definecolor{mygreen2}{RGB}{238, 243, 243}%223, 240, 240
\definecolor{ForestGreen}{RGB}{34,139,34}
\definecolor{Forestred}{RGB}{220,50,50}

\begin{table}[t]
    \centering
    \footnotesize
    \caption{Performance comparison on Qwen2.5-VL-7B across widely-used benchmarks.
    Best results in \textcolor{MidnightBlue}{\textbf{blue}}.
    }
    \label{tab:Qwen}
    \resizebox{0.95\linewidth}{!}{%
    \begin{tabular}{l | *{6}{c}}
        \toprule
        \textbf{Methods}  & \textbf{MME}  & \textbf{POPE} & \textbf{SQA} & \textbf{VQA}$_{\text{Text}}$ & \textbf{MMB} & \textbf{Avg} \\
        \midrule
        
        \textcolor{gray}{Upper Bound} 
        & \textcolor{gray}{2308} 
        & \textcolor{gray}{86.1} 
        & \textcolor{gray}{78.0} 
        & \textcolor{gray}{77.8} 
        & \textcolor{gray}{82.2}
        & \textcolor{gray}{100\%} \\
        \midrule

        \rowcolor{mygray}
        Qwen2.5-VL-7B & \multicolumn{6}{c}{\textit{TokenPruningRate = 66.7\%}}\\
        SparseVLM & 2165 & 81.7 & 73.8 & 71.5 & 78.4 & 94.1\% \\
        \rowcolor{mygreen2}
        CLASP (ours) & \textcolor{MidnightBlue}{\textbf{2197}} 
            & \textcolor{MidnightBlue}{\textbf{85.1}} 
            & \textcolor{MidnightBlue}{\textbf{75.4}} 
            & \textcolor{MidnightBlue}{\textbf{74.7}} 
            & \textcolor{MidnightBlue}{\textbf{78.6}}
            & \textcolor{MidnightBlue}{\textbf{96.5\%}} \\
        \midrule

        \rowcolor{mygray}
        Qwen2.5-VL-7B & \multicolumn{6}{c}{\textit{TokenPruningRate = 77.8\%}}\\
        SparseVLM & 2086 & 77.7 & 72.8 & 68.7 & 75.6 & 90.8\% \\
        \rowcolor{mygreen2}
        CLASP (ours) & \textcolor{MidnightBlue}{\textbf{2154}} 
            & \textcolor{MidnightBlue}{\textbf{83.0}} 
            & \textcolor{MidnightBlue}{\textbf{74.5}} 
            & \textcolor{MidnightBlue}{\textbf{73.1}} 
            & \textcolor{MidnightBlue}{\textbf{76.2}}
            & \textcolor{MidnightBlue}{\textbf{94.4\%}} \\
        \midrule

        \rowcolor{mygray}
        Qwen2.5-VL-7B & \multicolumn{6}{c}{\textit{TokenPruningRate = 88.9\%}}\\
        SparseVLM & 1813 & 66.7 & 72.0 & 61.2 & 72.1 & 82.9\% \\
        \rowcolor{mygreen2}
        CLASP (ours) & \textcolor{MidnightBlue}{\textbf{2010}} 
            & \textcolor{MidnightBlue}{\textbf{76.6}} 
            & \textcolor{MidnightBlue}{\textbf{72.5}} 
            & \textcolor{MidnightBlue}{\textbf{68.0}} 
            & \textcolor{MidnightBlue}{\textbf{72.7}}
            & \textcolor{MidnightBlue}{\textbf{89.0\%}} \\
        \bottomrule
    \end{tabular}%
    }
    \vspace{-0.25em}
\end{table}

\textbf{Main Results on Video Benchmarks.} To verify the versatility of our framework beyond static images, we extend our evaluation to the video domain using Video-LLaVA. As presented in Table~\ref{tab:video_performance}, we compare CLASP against the full-token Upper Bound and the baseline SparseVLM on three representative benchmarks (TGIF, MSVD, and MSRVTT). Remarkably, CLASP achieves the highest average score of 52.95, surpassing not only SparseVLM (51.22) but also the unpruned Upper Bound (52.23). This performance gain over the full-token baseline is particularly evident on TGIF (45.65 vs. 43.47) and MSRVTT (51.58 vs. 51.28). We attribute this to the high temporal and spatial redundancy inherent in video streams. By aggressively filtering out repetitive or non-informative tokens, our method likely reduces noise and allows the model to focus more effectively on core temporal semantics. 
Overall, these findings confirm that our relevance-diversity pruning mechanism is highly effective for video understanding.

\begin{table}[t]
    \centering
    \footnotesize
    \caption{Performance comparison on Video-LLaVA-7B across video benchmarks (TGIF, MSVD, MSRVTT). Best results in \textcolor{MidnightBlue}{\textbf{blue}}.}
    \label{tab:video_performance}
    \resizebox{0.85\linewidth}{!}{
    % 修改此处：l | c ... 在第一列和第二列之间加竖线
    \begin{tabular}{l | c c c c}
        \toprule
        \textbf{Method} & \textbf{TGIF} & \textbf{MSVD} & \textbf{MSRVTT} & \textbf{Average} \\
        \midrule
        Upper Bound  & 43.47 & 61.93 & 51.28 & 52.23 \\
        \midrule
        SparseVLM & 44.67 & 59.29 & 49.69 & 51.22 \\
        \rowcolor{mygreen2}
        CLASP (ours) & \textcolor{MidnightBlue}{\textbf{45.65}} & \textcolor{MidnightBlue}{\textbf{61.62}} & \textcolor{MidnightBlue}{\textbf{51.58}} & \textcolor{MidnightBlue}{\textbf{52.95}} \\
        \bottomrule
    \end{tabular}
    }
\end{table}

\textbf{Efficiency Analysis.}
Table~\ref{tab:efficiency_analysis} compares the computational cost and performance of our method against the full-token upper bound and the baseline SparseVLM~\citep{zhang2024sparsevlm} on the POPE benchmark.
Our approach consistently achieves a superior trade-off between efficiency and accuracy.
With a token budget of $R=192$, we attain a $1.5\times$ end-to-end speedup while retaining $99.6\%$ of the upper-bound performance, significantly outperforming SparseVLM ($1.2\times$ speedup, $97.3\%$ accuracy) in both throughput and fidelity.
Notably, in the highly aggressive compression setting ($R=58$), our method demonstrates robust resilience, delivering a $2.1\times$ speedup with only a $4.6\%$ relative performance drop, whereas SparseVLM suffers a sharp degradation to $92.3\%$.
These results confirm that the synergy of class-adaptive layer fusion and dual-stage pruning effectively identifies and preserves the most information-dense tokens, enabling substantial latency reduction while maintaining high accuracy. 

\begin{table}[t]
    \centering
    \footnotesize
    \caption{Efficiency analysis on POPE at different retention levels. 
    $\Delta$ denotes the speedup ratio in Time.}
    \label{tab:efficiency_analysis}
    \resizebox{0.95\linewidth}{!}{
    \begin{tabular}{l | c c c c c c}
        \toprule
        \textbf{Method} & \textbf{Time} & \textbf{Prefill} & \textbf{Latency} & \textbf{Mem.} & \textbf{Acc.} & \textbf{$\Delta$} \\
        \midrule
        
        % --- R = 192 Section ---
        \rowcolor{mygray}
        \multicolumn{7}{c}{\textit{Retain 192 Tokens $\fg{(\downarrow 66.7\%)}$ }}  \\
        Upper Bound & 49:41 & 0.5ms & 0.334s & 19.0G & 100.0\% & - \\
        \midrule
        SparseVLM   & 40:51 & 0.6ms & 0.251s & \textcolor{MidnightBlue}{\textbf{15.8G}} & 97.3\% & 1.2$\times$ \\
        \rowcolor{mygreen2}
        CLASP (ours) & \textcolor{MidnightBlue}{\textbf{33:08}} & \textcolor{MidnightBlue}{\textbf{0.2ms}} & \textcolor{MidnightBlue}{\textbf{0.216s}} & 17.6G & \textcolor{MidnightBlue}{\textbf{99.6\%}} & \textcolor{MidnightBlue}{\textbf{1.5$\times$}} \\
        
        \midrule
        
        % --- R = 58 Section ---
        \rowcolor{mygray}
        \multicolumn{7}{c}{\textit{Retain 58 Tokens $\fg{(\downarrow 89.9\%)}$ }} \\
        Upper Bound & 49:41 & 0.5ms & 0.334s & 19.0G & 100.0\% & - \\
        \midrule
        SparseVLM   & 31:28 & 0.6ms & 0.212s & \textcolor{MidnightBlue}{\textbf{14.6G}} & 92.3\% & 1.6$\times$ \\
        \rowcolor{mygreen2}
        CLASP (ours) & \textcolor{MidnightBlue}{\textbf{24:10}} & \textcolor{MidnightBlue}{\textbf{0.1ms}} & \textcolor{MidnightBlue}{\textbf{0.155s}} & 17.6G & \textcolor{MidnightBlue}{\textbf{95.4\%}} & \textcolor{MidnightBlue}{\textbf{2.1$\times$}} \\
        \bottomrule
    \end{tabular}
    }
    \vspace{-0.5em}
\end{table}

% \textbf{Ablation of our method.}
% In this ablation study, we examine the individual contributions of two key components in our method: attention-based scoring and similarity-based scoring. As shown in \textcolor{red}{Table}, the results demonstrate that both components enhance the model’s performance across different pruning rates. For the MME dataset, SparseVLM achieves a performance of 1721 at a pruning rate of 66.7\%, while incorporating attention-based scoring improves the score to 1814, and adding similarity-based scoring further enhances it to 1848, showing a significant performance boost. Similar improvements are observed on the GQA dataset, where attention-based scoring raises the performance from 57.6 to 59.57, and similarity-based scoring results in a further increase to 60.44. On the MMEvet dataset, the combination of both components yields the highest score of 33.3, compared to the initial 33.1 from SparseVLM. For the POPE dataset, attention-based scoring improves the SparseVLM score of 83.6 to 85.02, and similarity-based scoring results in the highest performance of 85.55. These findings highlight that both attention-based scoring and similarity-based scoring play crucial roles in improving performance, with similarity-based scoring providing the most significant gains across various datasets and pruning rates.

\textbf{Token Pruning Visualization.}
Complementing the quantitative results, Fig.~\ref{fig-Visual-visual} visualizes the spatial distribution of retained tokens across samples. Unlike baselines that often fail to retain salient regions, our method consistently preserves critical areas. We further categorize the retained visual regions: blue regions represent tokens selected based on attention scores, indicating high relevance, while red regions represent tokens selected based on similarity metrics, providing contextual coverage. This distinction emphasizes that our method balances different visual aspects: blue regions capture the primary foreground, while red regions ensure background completeness. As pruning rates increase, our method robustly retains these essential cues, effectively filtering redundancy while preserving the core semantics needed for reasoning. Overall, this qualitative evidence confirms that our pruning strategy preserves pivotal information for visual understanding, ensuring robust alignment for downstream tasks.

\begin{figure}[t]
    \centering
    \includegraphics[width=0.98\linewidth]{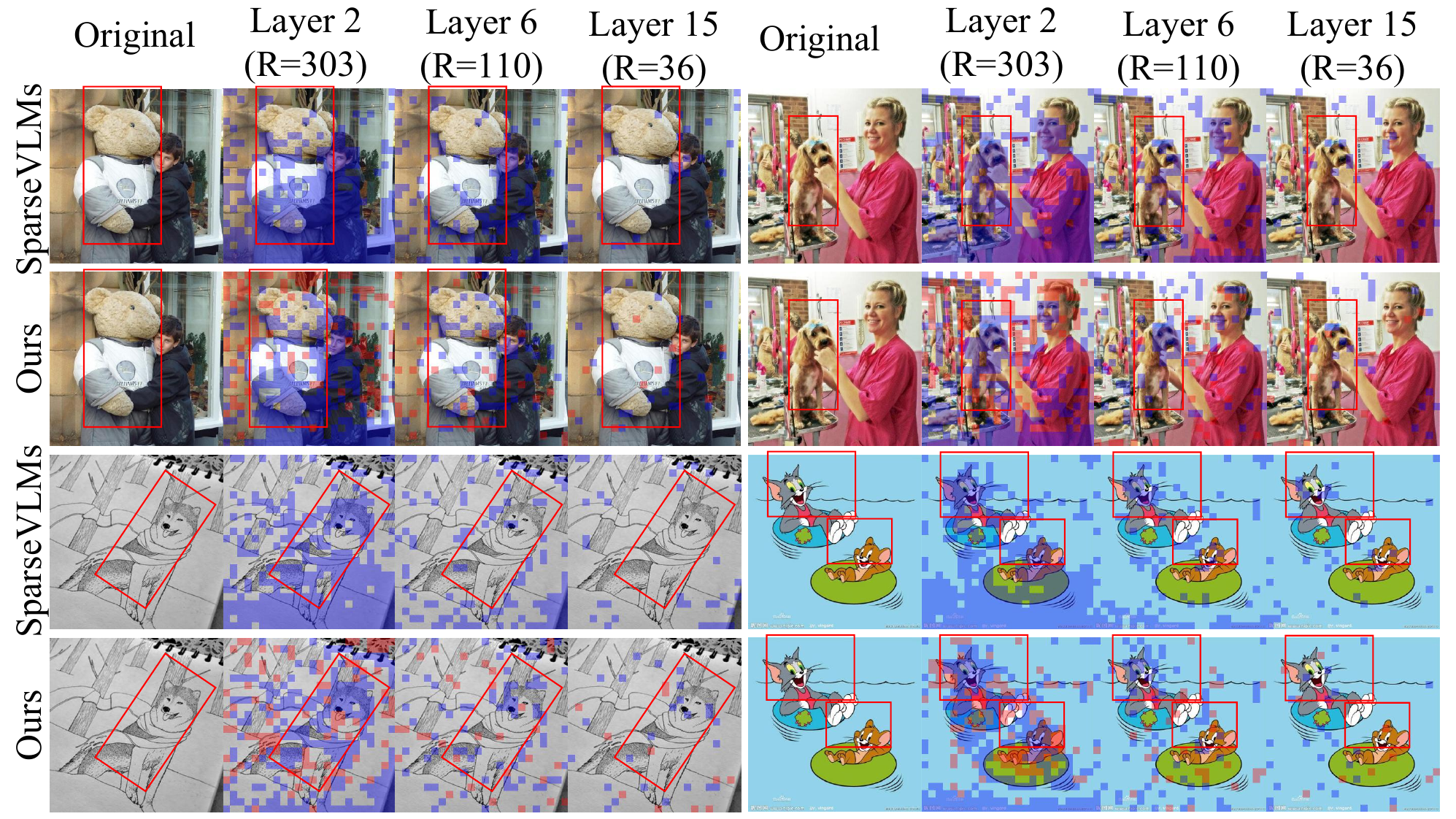}
    \caption{Example visualization of the original image and the corresponding token-retention map.}
    \label{fig-Visual-visual}
\end{figure}

\section{Conclusion} In this paper, we proposed CLASP, a framework synergizing class-adaptive layer fusion with dual-stage pruning to balance token relevance and spatial coverage. Our results demonstrate that dynamic reduction minimizes redundancy without compromising fine-grained perception, marking a fundamental shift from static to conditional visual encoding. Future work should aim to transition from manual heuristic priors to end-to-end learnable routing policies, alongside tighter hardware-aware co-design to fully realize efficient multimodal intelligence on edge and embodied systems.

\newpage

\clearpage

\section*{Impact Statement}
This work reduces compute, memory, and latency of MLLM inference by pruning redundant visual tokens without retraining, which can lower cost and energy per query and enable deployment on resource-constrained hardware. Improved efficiency may broaden access to multimodal systems in assistive and on-device settings. However, lowering the marginal cost of multimodal inference can also facilitate harmful uses, including privacy-invasive analytics, large-scale surveillance, and high-volume content generation. CLASP does not add new capabilities beyond faster inference, but responsible deployment remains important: practitioners should follow data minimization and consent practices, apply secure storage and access control for visual inputs, monitor misuse, and validate robustness under aggressive pruning before use in sensitive domains.

% =========================================================
% 参考文献
% =========================================================
\bibliographystyle{icml2026} % 【必须添加】指定引用格式
\bibliography{custom}

% =========================================================
% 附录
% =========================================================
\newpage
\appendix
\onecolumn % ICML 附录通常建议单栏，方便放宽图表

\appendix

\section{Detailed Experiment Settings.}
\label{Detailed_experiment_settings}

\paragraph{Benchmarks and Metrics.}
We conduct experiments on a comprehensive suite of benchmarks to evaluate our model's multimodal capabilities across both static and dynamic scenarios.
For \textbf{image understanding}, we assess general perception, compositional reasoning, OCR, and hallucination using eight widely used datasets in our main evaluation:
GQA~\citep{hudson2019gqa}, MMBench (MMB)~\citep{liu2024mmbench}, MME~\citep{fu2025mme}, POPE~\citep{Li-hallucination-2023}, ScienceQA (SQA)~\citep{lu2022learn}, VQAv2~\citep{goyal2017making}, TextVQA~\citep{singh2019towards}, and MMVet~\citep{yu2023mm}. Additionally, to further demonstrate the robustness of our approach, we provide supplementary evaluations on LLaVA-Bench (LLaVA-B)~\citep{liu2023visual}, VizWiz~\citep{gurari2018vizwiz}, MMBench-Chinese (MMB-CN)~\citep{liu2024mmbench}, and SEED-Bench (SEED)~\citep{li2024seed} in the appendix.
Furthermore, to verify that our pruning method maintains \textbf{temporal reasoning and video comprehension} capabilities, we evaluate on three standard video benchmarks:
TGIF~\citep{li2016tgif}, MSVD~\citep{chen2011collecting}, and MSRVTT~\citep{xu2016msr}.
Unless otherwise specified, we report accuracy following the official evaluation protocols to ensure a fair comparison.

\paragraph{GQA.} 
GQA~\citep{hudson2019gqa} evaluates scene understanding and compositional reasoning. It is built upon images, structured scene graphs, and questions designed to probe fine-grained object attributes, spatial relations, and multi-step reasoning. Generated from the Visual Genome dataset, GQA contains over 22 million questions across approximately 113K images. A key characteristic of GQA is its rigorous balancing of the answer distribution for each question group, which significantly mitigates language priors and forces models to rely on visual evidence rather than statistical correlations. In the context of model compression, GQA provides a diagnostic view of a model's ability to jointly perceive and reason over complex visual scenes, serving as a stress test for preserving multi-hop reasoning capabilities after pruning.

\paragraph{MMBench.} 
MMBench~\citep{liu2024mmbench} provides a multi-dimensional evaluation framework organized in a three-level hierarchy of capabilities. Level-1 targets the two core abilities of perception and reasoning; Level-2 expands these into six sub-abilities; and Level-3 further refines the assessment into 20 fine-grained ability dimensions. Constructed from various sources, MMBench contains approximately 3,000 multiple-choice questions covering diverse domains. A distinguishing feature of MMBench is its CircularEval strategy, which inputs the same question with shifted options multiple times to the model. This mechanism effectively mitigates the model's sensitivity to option ordering and random guessing. For VLM pruning, MMBench serves as a critical benchmark to verify that the compressed model maintains robust instruction-following and reasoning capabilities across a broad spectrum of tasks, rather than overfitting to specific patterns.

\paragraph{MME.} 
MME~\citep{fu2025mme} is a comprehensive and quantitative benchmark designed to evaluate multimodal LLMs across 14 distinct subtasks. These subtasks are structurally categorized into two primary axes: Perception (\textit{e.g.}, coarse-grained recognition, OCR, color) and Cognition (\textit{e.g.}, commonsense reasoning, numerical calculation). MME employs concise instruction--answer pairs, predominantly in a "Yes/No" format, to minimize the influence of prompt engineering and reduce potential data leakage. With manually annotated samples, it facilitates a reliable measurement of multimodal performance. For pruning studies, MME is particularly useful for diagnosing whether parameter reduction disproportionately impacts basic visual perception or higher-order cognitive reasoning.

\paragraph{POPE.} 
POPE~\citep{Li-hallucination-2023} is a specialized benchmark focusing on evaluating object hallucination in VLMs. It reformulates hallucination assessment as a series of binary "Yes/No" questions about the presence of specific objects in an image (based on MSCOCO validation sets). Crucially, POPE evaluates models under three distinct sampling settings to disentangle visual perception from language priors: Random (objects not in the image), Popular (frequent objects in the dataset), and Adversarial (objects that often co-occur with present objects but are absent). By reporting Accuracy, Recall, Precision, and F1 across these settings, POPE provides a robust quantification of hallucination tendencies. For pruned models, maintaining high performance on the Adversarial setting is a strong indicator that the compression has not compromised the model's ability to ground answers in actual visual evidence.

\paragraph{ScienceQA.}
ScienceQA~\citep{lu2022learn} spans a wide array of domains, including natural, social, and language sciences. Questions are hierarchically categorized into 26 topics, 127 categories, and 379 skills, providing a diverse and comprehensive testbed for evaluating multimodal understanding and multi-step reasoning. Distinctively, ScienceQA includes annotated lectures and explanations, facilitating the assessment of Chain-of-Thought (CoT) capabilities. In the context of VLM pruning, ScienceQA is essential for evaluating whether the compressed model retains domain-specific knowledge and the ability to perform interpretable reasoning, ensuring that parameter reduction does not compromise the model's capacity to handle complex, knowledge-intensive tasks.

\paragraph{VQAv2.} 
VQAv2~\citep{goyal2017making} is a standard open-ended VQA benchmark covering diverse real-world images and questions derived from the MS COCO dataset. A critical feature of VQAv2 is its design to minimize language bias through balanced pairs: for every question, the dataset includes complementary images that result in different answers (\textit{e.g.}, "Yes" vs. "No"). This structure compels the model to rely on visual evidence rather than exploiting statistical language correlations. With over 1.1 million questions and multiple human annotations per question, VQAv2 serves as a large-scale testbed for general visual question answering. For pruning, it provides a fundamental baseline to ensure that the compressed model retains generalized visual recognition capabilities and aligns visual features correctly with textual queries.

\paragraph{TextVQA.} TextVQA~\citep{singh2019towards} evaluates reasoning over embedded text in images, addressing a specific capability often overlooked by general VQA benchmarks. Built upon images from the Open Images dataset, it comprises 45,336 questions that require the model to detect, recognize, and reason about text appearing in diverse scenes (\textit{e.g.}, signboards, book covers). Successful performance demands a tight integration of optical character recognition (OCR) and semantic reasoning, as the model must often transcribe specific text strings from the image to formulate the answer. For VLM pruning, TextVQA serves as a rigorous test of fine-grained feature preservation, verifying that the compressed model retains sufficient resolution and local attention to process small, symbol-rich visual elements without degradation.

\paragraph{MM-Vet.} MM-Vet~\citep{yu2023mm} targets complex multimodal problem solving by defining six core vision--language capabilities: Recognition, Knowledge, OCR, Spatial Awareness, Language Generation, and Math. Beyond testing these in isolation, MM-Vet evaluates 16 distinct integrations of these capabilities (\textit{e.g.}, recognizing an object and then answering a knowledge-based question about it). The benchmark employs an LLM-based evaluation system (typically GPT-4) to score open-ended responses, offering a nuance that rigid string-matching metrics lack. For pruning, MM-Vet is instrumental in verifying that the compressed model retains the synergy required to chain multiple reasoning steps, ensuring that the removal of parameters does not sever the functional connections between different cognitive modules.

\paragraph{LLaVA-Bench.}
LLaVA-Bench~\citep{liu2023visual} is designed to evaluate the capability of MLLMs in handling complex, open-ended visual instructions in real-world scenarios. Comprising a diverse set of ``in-the-wild'' images and detailed queries, it relies on GPT-4 as an impartial judge to score the model's generated responses against reference answers. Unlike standard objective benchmarks, LLaVA-Bench captures the nuances of conversational fluency, helpfulness, and visual grounding. In the context of token pruning, it is crucial for verifying that aggressive visual compression does not degrade the model's capacity for rich, free-form multimodal generation and its adherence to unpredictable human instructions.

\paragraph{VizWiz.}
VizWiz~\citep{gurari2018vizwiz} originates from authentic visual questions asked by blind and visually impaired individuals. Consequently, the dataset is characterized by images with severe real-world noise, such as poor lighting, blurriness, and occlusions, coupled with occasionally unanswerable queries. This benchmark rigorously tests a model's robustness and zero-shot generalization under suboptimal visual conditions. For VLM pruning, VizWiz serves as a unique diagnostic tool: it assesses whether the pruning algorithm remains robust when visual cues are inherently scarce or distorted, ensuring that the token reduction process does not disproportionately discard critical, albeit noisy, visual information.

\paragraph{MMBench-Chinese.}
MMBench-Chinese (MMB-CN)~\citep{liu2024mmbench} is the bilingual extension of the MMBench evaluation framework, translating the rigorous, multi-dimensional assessment into the Chinese linguistic context. Retaining the CircularEval strategy to mitigate option-selection bias, it evaluates the same broad spectrum of perception and reasoning capabilities but fundamentally requires robust cross-lingual alignment. Within our pruning study, MMB-CN is utilized to confirm that the class-adaptive layer fusion and token reduction mechanisms generalize effectively across diverse linguistic spaces, ensuring that cross-modal semantic alignment is not structurally impaired when evaluated in non-English languages.

\paragraph{SEED-Bench.}
SEED-Bench~\citep{li2024seed} is a comprehensive multimodal benchmark specifically designed to evaluate fine-grained visual understanding and spatial reasoning. It comprises thousands of multiple-choice questions with meticulously crafted distractors, targeting specific dimensions such as object attributes, instance locations, and spatial relations. The precise nature of these queries forces models to rely on exact visual evidence rather than holistic semantic guessing. In the evaluation of pruned MLLMs, SEED-Bench acts as a stringent stress test for fine-grained token retention; it reveals whether our dual-stage pruning strategy successfully preserves the crucial, highly localized visual tokens required to distinguish between subtle, misleading options.

\paragraph{TGIF.}
TGIF~\citep{li2016tgif} is a dataset focusing on animated GIFs, which captures motion semantics and temporal dynamics. It contains over 100K animated images sourced from social media, paired with natural language descriptions. Unlike static image datasets, TGIF challenges models to reason about actions, repetitions, and state transitions over time. In the context of VLM pruning, TGIF serves as a critical benchmark to assess whether the compressed model preserves \textit{temporal consistency}. Since pruning often involves reducing redundancy in visual tokens, performance on TGIF indicates whether the method effectively distinguishes between essential motion cues and redundant background frames.

\paragraph{MSVD.}
MSVD~\citep{chen2011collecting} constitutes a fundamental benchmark for video question answering, consisting of approximately 1,970 short video clips collected from YouTube. Unlike the captioning task, MSVD requires the model to provide precise answers to questions regarding video content based on spatio-temporal reasoning. For pruning studies, MSVD acts as a baseline sensitivity test. High performance here ensures that the reduction in model parameters has not destroyed the basic alignment between spatio-temporal visual features and the language decoder, serving as a sanity check for general video-to-text alignment. We report the accuracy following the official evaluation protocol to ensure consistency with other video benchmarks.

\paragraph{MSRVTT.}
MSRVTT~\citep{xu2016msr} is a dataset designed for open-domain video question answering, containing 10,000 video clips categorized into 20 distinct classes (\textit{e.g.}, music, sports, news). This diversity requires models to handle a wide vocabulary and complex visual scenes to answer questions regarding fine-grained events. In the context of model compression, MSRVTT is particularly useful for diagnosing "catastrophic forgetting" in the video domain. Maintaining high accuracy on MSRVTT implies that the pruned model retains robust feature representations across diverse categories and is not overfitting to specific visual patterns. Consistent with the official settings, we report accuracy as the primary metric.

\paragraph{Models.} 
We instantiate and evaluate CLASP on top of several representative \emph{open-source} multimodal large language models (MLLMs). For \textbf{image understanding}, we primarily build on the LLaVA family, including the standard LLaVA-1.5~\citep{liu2023improvedllava} and the stronger LLaVA-NeXT~\citep{liu2024llava}. We specifically adopt LLaVA-NeXT to validate performance under \textbf{high-resolution} visual inputs, following the official inference settings. To demonstrate the generalization of our method across different architectures, we further include Qwen2.5-VL~\citep{Qwen2.5-VL}. As a state-of-the-art model, it introduces Naive Dynamic Resolution mechanisms to handle arbitrary aspect ratios, serving as a rigorous testbed for our pruning strategy on variable-length visual tokens. Finally, to assess capabilities in the \textbf{temporal domain}, we extend our evaluation to Video-LLaVA~\citep{lin2023video}. This model unifies image and video feature alignment, allowing us to verify whether \textsc{CLASP} can effectively reduce temporal redundancies without compromising motion understanding.

\paragraph{ToMe.} 
ToMe~\citep{bolya2022token} accelerates vision transformers by merging similar tokens within transformer layers using lightweight token matching. Unlike traditional pruning methods that discard "unimportant" tokens, ToMe employs a bipartite matching algorithm based on feature similarity (typically using the Key or Query projections) to identify redundant tokens. It then aggregates these tokens via weighted averaging, progressively reducing the sequence length across the network depth. Since it requires no additional parameters or retraining, ToMe serves as a highly efficient, plug-and-play baseline for evaluating the trade-off between inference speed and information preservation in the VLM's visual encoder.

\paragraph{LLaVA-PruMerge.} 
LLaVA-PruMerge~\citep{shang2025llava} introduces an adaptive hybrid strategy that combines the benefits of both pruning and merging. It first evaluates the importance of visual tokens based on the attention scores from the special [CLS] token to image patches, effectively identifying regions most relevant to the global semantic context. Based on these scores, the method divides tokens into subsets: tokens with low importance scores are directly pruned to remove background noise, while highly correlated foreground tokens are merged based on key similarity. This approach significantly reduces the sequence length of the visual encoder without retraining, ensuring that critical visual details are preserved while spatial redundancy is efficiently eliminated before the tokens enter the LLM.

\paragraph{FastV.} 
FastV~\citep{chen2024image} identifies the inefficiency of visual tokens within Large Language Models (LLMs) and performs early-stage token pruning. It is grounded in the observation that while visual tokens are crucial in the initial transformer layers, the attention mechanism in deeper layers tends to ignore the vast majority of them. Leveraging this, FastV ranks visual tokens based on their average attention weights in the early layers (\textit{e.g.}, the second layer) and discards the least significant ones for all subsequent layers. This strategy significantly reduces the KV-cache memory footprint and FLOPs during inference, allowing the model to maintain high performance while processing significantly fewer tokens in the computation-heavy deep layers.

\paragraph{HiRED.} 
HiRED~\citep{arif2025hired} introduces a spatially-aware pruning strategy designed to mitigate the "tunnel vision" problem often observed in global ranking methods. Instead of competing all visual tokens in a single pool, HiRED divides the image into spatial partitions. It then dynamically allocates a specific token budget to each partition based on the attention distribution of the [CLS] token. By selecting the most informative tokens within these localized budgets, HiRED ensures that the pruned model maintains broad spatial coverage across the image. This hierarchical approach allows for aggressive compression while preventing the complete suppression of background regions or secondary objects, which are essential for tasks requiring holistic scene understanding.

\paragraph{PDrop.} 
PDrop~\citep{xing2024pyramiddrop}, or PyramidDrop, implements a progressive pruning strategy that mimics the hierarchical structure of CNNs within a standard Vision Transformer. Instead of maintaining a constant sequence length throughout the encoder, PDrop drastically reduces the token count at specific intermediate layers based on attention importance scores. This constructs a "pyramid-like" information flow, where high-resolution details are processed in shallow layers, while only the most semantically salient tokens are retained for the computation-heavy deep layers. By effectively funneling visual information, PDrop achieves a significant reduction in FLOPs and latency, serving as a representative baseline for structural pruning methods.

\paragraph{Multi-Stage Vision Token Dropping (MustDrop).} 
Multi-Stage Vision Token Dropping~\citep{liu2024multi} adopts a progressive compression strategy within the vision encoder to optimize computational efficiency. Instead of performing a single-step reduction, it executes token dropping at multiple intermediate layers (stages) based on attention significance or learned policies. This hierarchical approach allows the model to retain fine-grained low-level details in the shallow stages while aggressively reducing spatial redundancy in the deeper, more semantic layers. Consequently, it significantly lowers the FLOPs of the vision backbone and minimizes the sequence length passed to the multimodal projector, achieving a favorable balance between inference speed and task performance.

\paragraph{SparseVLM.} 
SparseVLM~\citep{zhang2024sparsevlm} presents a dynamic pruning framework that goes beyond uni-modal visual salience. It ranks token importance using cross-modal attention, ensuring that visual tokens are preserved based on their relevance to the specific textual query or instruction. Recognizing that different images contain varying amounts of information, SparseVLM introduces adaptive sparsity ratios, dynamically adjusting the token budget for each input. Furthermore, to mitigate the risks of aggressive pruning, it proposes a novel token recycling mechanism. Instead of permanently discarding pruned tokens, this mechanism aggregates or buffers them, allowing the model to retrieve context from these "discarded" regions if necessary. This significantly improves efficiency under varying input complexity without compromising the semantic integrity of the visual representation.

\paragraph{VisionZip.}
VisionZip~\citep{yang2025visionzip} addresses the spatial redundancy problem often found in attention-based pruning methods. While standard approaches simply select the top-$k$ tokens with the highest attention scores, VisionZip observes that these high-scoring tokens tend to cluster around the same visual object, leading to repetitive information.
To mitigate this, it implements a two-step pipeline: first, it evaluates token saliency via encoder attention to filter out background noise; second, it clusters the remaining high-saliency tokens based on key similarity.
By selecting representative tokens from these clusters, VisionZip ensures that the final compressed sequence maintains high semantic diversity, covering various distinct objects and regions within the image for downstream multimodal reasoning.

\paragraph{DART.} 
DART~\citep{wen2025stop} proposes a Duplication-Aware Reduction Transformer strategy designed to maximize information diversity while maintaining hardware efficiency. Unlike magnitude-based methods that may select redundant high-norm tokens, DART explicitly filters out redundancy. It operates by selecting a small set of "pivot" tokens and computing the cosine similarity between these pivots and the remaining tokens. Tokens with high similarity (high duplication) are discarded, ensuring that the retained set covers a broad semantic range. Crucially, DART is designed to be compatible with FlashAttention and other efficient attention operators. By avoiding complex gather-scatter operations or irregular memory access patterns, it ensures that the theoretical reduction in FLOPs translates directly into significant wall-clock speedups during inference.

\paragraph{Implementation details.}
For image-based benchmarks, we run experiments using the official LLaVA implementation. For video benchmarks, we build the model following the official LLaVA-NeXT codebase and conduct evaluation via \texttt{lmms-eval}. For more recent VLM architectures (\textit{e.g.}, Qwen2.5-VL), we rely on \texttt{VLMEvalKit} as the evaluation toolkit. All inference evaluations were executed on NVIDIA A100 (80GB) GPUs. To ensure reproducibility, we detail the specific configurations for dataset construction, parameter validation, and model-specific architectures below.

\textbf{Dataset Categorization and Intent Mapping.} 
To enable class-adaptive routing, we constructed a categorized dataset using the Qwen3-8B model to identify visual question answering intents. We designed a system prompt defining 15 fine-grained intents (\textit{e.g.}, \textit{arithmetic reasoning}, \textit{temporal order}) and instructed the model to output classifications in a strict JSON format. To enhance robustness, we applied deterministic heuristic overrides (\textit{e.g.}, mapping keywords like ``who wrote'' to \textit{Text/Symbol Recognition}). These labels were merged into the 9 core categories (Table~\ref{tab:hyperparams}) to ensure sufficient sample density for routing.

\textbf{Hyperparameter Search and Optimization.} 
To determine the optimal fusion weights and pruning ratios, we employed the \textit{Discrete Subspace Search Algorithm} detailed in Appendix~\ref{sec:appendix-method}. Specifically, we performed this calibration on a balanced dataset of 2,000 samples per category. Although fully unconstrained gradient-based optimization can risk overfitting on limited calibration data, we found that our approach---which restricts the search space to discrete, semantically meaningful layer prototypes---effectively prevents such overfitting. Crucially, our experiments indicated that this \textit{data-driven calibration} yielded superior zero-shot generalization compared to task-specific manual heuristics. By automatically adapting to the level of visual abstraction required for each intent (e.g., balancing fine-grained details and high-level semantics), the search algorithm identifies more robust configurations. Consequently, all results reported in this paper utilize the calibrated configurations obtained via this search process, as provided in Table~\ref{tab:hyperparams}.

\textbf{Model-Specific Configurations.} 
We devised distinct fusion strategies tailored to the architecture of the vision encoders. The LLaVA family (LLaVA-v1.5, LLaVA-NeXT, and Video-LLaVA) shares a unified set of fusion parameters, while Qwen2.5-VL employs a separate configuration adapted to its layer structure. The split ratios ($a_c$) for the two-stage pruning were kept consistent across all models.

\paragraph{Pruning Schedule and Budgets.}
We adopt the progressive sparsification strategy from SparseVLM, executing token pruning at three designated intermediate layers of the multimodal decoder on the \textbf{projected (decoder-space) visual tokens}: \textbf{Layer 2, Layer 6, and Layer 15}. 
Following SparseVLM~\citep{zhang2024sparsevlm}, we report an \textbf{effective token budget} $R$ under progressive pruning; the stage-wise budgets $[K_{L2}, K_{L6}, K_{L15}]$ are chosen to match the same cumulative computational cost (FLOPs) as maintaining a constant sequence length of $R$ throughout the decoder. 
Taking the standard input resolution ($N=576$ projected visual tokens) as a baseline, the stepwise retention counts are configured as follows:
(1) For the \textbf{192-token} setting ($R=192$), we retain $[300, 200, 110]$ tokens respectively;
(2) For the \textbf{128-token} setting ($R=128$), the schedule is set to $[303, 110, 36]$;
(3) For the \textbf{64-token} setting ($R=64$), we employ an aggressive schedule of $[66, 30, 17]$.
For high-resolution architectures (\textit{e.g.}, LLaVA-NeXT) processing $N=2880$ tokens, these budgets are scaled proportionally.

\paragraph{Relation to SparseVLM.}
For fair efficiency comparisons, we follow SparseVLM's~\citep{zhang2024sparsevlm} progressive three-stage pruning schedule in the multimodal decoder (Layer 2/6/15), while differing in how the prompt conditions the visual representation and how the token budget is allocated. SparseVLM ranks visual tokens using query-conditioned cross-modal attention and further mitigates aggressive sparsification via token recycling and adaptive sparsity ratios~\citep{zhang2024sparsevlm}. In contrast, we adopt a lightweight text-only router and use its predicted category to drive two class-adaptive components: (i) \emph{class-adaptive multi-layer feature fusion}, where token representations are formed by a temperature-controlled mixture over vision layers to match the level of visual abstraction required by the prompt; and (ii) \emph{attention--similarity adaptive-ratio pruning}, where a fixed budget $R$ is split into $K_1=\lfloor a_c R\rfloor$ attention-salient pivots (relevance) and $K_2=R-K_1$ similarity-based completion tokens selected by minimizing redundancy under cosine similarity (diversity). This formulation explicitly combines relevance-driven attention selection (\textit{e.g.}, FastV~\citep{chen2024image}, HiRED~\citep{arif2025hired}, PyramidDrop~\citep{xing2024pyramiddrop}) with redundancy-aware retention strategies (\textit{e.g.}, ToMe~\citep{bolya2022token}, VisionZip~\citep{yang2025visionzip}, DART~\citep{wen2025stop}, DivPrune~\citep{alvar2025divprune}, conditional-diversity pruning~\citep{zhang2025beyond}), while keeping the backbone frozen and introducing negligible inference overhead.

% preamble:
% \usepackage{booktabs}
% \usepackage{tabularx}   % 可选，但推荐用于 table* 自动适配宽度

\begin{table*}[t]
  \centering
  \small
  \caption{Hyperparameters for class-adaptive layer fusion and pruning across MLLM backbones.
  \textbf{Fusion weights} are reported as \{layer: coefficient\}.
  The \textbf{split ratio} $a_c$ governs the prompt-conditioned budget allocation between attention-salient pivots for relevance and redundancy-aware completion tokens for coverage.}
  \label{tab:hyperparams}

  \resizebox{0.9\textwidth}{!}{%
  \begin{tabular}{c l l l c}
    \toprule
    \textbf{ID} & \textbf{Category} &
    \textbf{Fusion weights (LLaVA family)} &
    \textbf{Fusion weights (Qwen2.5-VL)} &
    \textbf{Split ratio $a_c$} \\
    \midrule
    0 & \mbox{Object identification}
      & $\{L_5\!:\!0.2,\; L_{15}\!:\!0.3,\; L_{22}\!:\!0.5\}$
      & $\{L_9\!:\!0.2,\; L_{22}\!:\!0.3,\; L_{31}\!:\!0.5\}$
      & 0.8 \\
    1 & \mbox{Attribute / breed identification}
      & $\{L_5\!:\!0.2,\; L_{22}\!:\!0.8\}$
      & $\{L_9\!:\!0.2,\; L_{31}\!:\!0.8\}$
      & 0.4 \\
    2 & \mbox{Text / symbol recognition}
      & $\{L_5\!:\!0.2,\; L_{22}\!:\!0.8\}$
      & $\{L_9\!:\!0.2,\; L_{31}\!:\!0.8\}$
      & 0.7 \\
    3 & \mbox{Scene understanding}
      & $\{L_{20}\!:\!0.2,\; L_{22}\!:\!0.8\}$
      & $\{L_{28}\!:\!0.2,\; L_{31}\!:\!0.8\}$
      & 0.7 \\
    4 & \mbox{Spatial relations}
      & $\{L_{14}\!:\!0.2,\; L_{17}\!:\!0.3,\; L_{22}\!:\!0.5\}$
      & $\{L_{21}\!:\!0.2,\; L_{24}\!:\!0.3,\; L_{31}\!:\!0.5\}$
      & 0.7 \\
    5 & \mbox{Counting}
      & $\{L_5\!:\!0.2,\; L_{15}\!:\!0.3,\; L_{22}\!:\!0.5\}$
      & $\{L_9\!:\!0.2,\; L_{22}\!:\!0.3,\; L_{31}\!:\!0.5\}$
      & 0.6 \\
    6 & \mbox{Action / interaction}
      & $\{L_{12}\!:\!0.2,\; L_{15}\!:\!0.3,\; L_{19}\!:\!0.5\}$
      & $\{L_{18}\!:\!0.2,\; L_{22}\!:\!0.3,\; L_{28}\!:\!0.5\}$
      & 0.8 \\
    7 & \mbox{Intention / function}
      & $\{L_3\!:\!0.2,\; L_{12}\!:\!0.3,\; L_{18}\!:\!0.5\}$
      & $\{L_6\!:\!0.2,\; L_{18}\!:\!0.3,\; L_{26}\!:\!0.5\}$
      & 0.2 \\
    8 & \mbox{Default}
      & $\{L_{20}\!:\!0.2,\; L_{22}\!:\!0.8\}$
      & $\{L_{29}\!:\!0.2,\; L_{31}\!:\!0.8\}$
      & 0.9 \\
    \bottomrule
  \end{tabular}%
  }
\end{table*}

\begin{table*}[t]
  \centering
  \small
  \caption{Detailed statistical distribution of sample counts per task category across all evaluated multimodal benchmarks.}
  \label{tab:dataset_by_class}
  \resizebox{0.9\textwidth}{!}{%
  \begin{tabular}{c l r r r r r r r r}
    \toprule
    \textbf{ID} & \textbf{Category} & \textbf{GQA} & \textbf{MMB} & \textbf{MME} & \textbf{MMVet} & \textbf{POPE} & \textbf{SQA} & \textbf{VQA}$_{\text{v2}}$ & \textbf{VQA}$_{\text{Text}}$ \\
    \midrule
    0 & Object identification            & 2095 & 272  & 355  & 21 & 2277 & 178  & 18263 & 239  \\
    1 & Attribute / breed identification & 4015 & 268  & 85   & 19 & 0    & 410  & 19380 & 1037 \\
    2 & Text / symbol recognition        & 75   & 111  & 147  & 12 & 0    & 465  & 4368  & 2856 \\
    3 & Scene understanding              & 2319 & 1950 & 1209 & 58 & 6633 & 521  & 35553 & 388  \\
    4 & Spatial relations                & 3632 & 379  & 70   & 18 & 0    & 193  & 7378  & 39   \\
    5 & Counting                         & 9    & 100  & 87   & 5  & 0    & 54   & 10007 & 53   \\
    6 & Action / interaction             & 200  & 74   & 2    & 4  & 0    & 1    & 5850  & 8    \\
    7 & Intention / function             & 43   & 287  & 79   & 9  & 0    & 8    & 3440  & 121  \\
    8 & Default                          & 190  & 936  & 340  & 72 & 0    & 2411 & 3155  & 259  \\
    \bottomrule
  \end{tabular}%
  }
\end{table*}

\section{Additional Method Details}
\label{sec:appendix-method}

\subsection{Calibration of $\mathbf{W}$ and $\mathbf{a}$}
\label{sec:appendix-calibration}

The adaptive parameters in our method are low dimensional ($C\times L$ layer scores and $C$ split ratios) and can be set
without finetuning any parameters of the underlying MLLM (vision encoder, projector, or LLM).
We outline a practical calibration protocol to obtain stable per-category configurations.

\textbf{Data and objective.}
Assume access to a small calibration set
$\mathcal{D}_{\text{cal}}=\{(x_n, \mathrm{img}_n, y_n^\star, c_n)\}_{n=1}^{N_{\text{cal}}}$
that contains (i) prompts $x_n$, (ii) images $\mathrm{img}_n$, (iii) reference outputs $y_n^\star$
(\textit{e.g.}, gold answers for QA / multiple-choice labels), and (iv) category labels $c_n\in\mathcal{N}_c$.
The category labels can be obtained from benchmark taxonomies (\textit{e.g.}, question-type tags) or a lightweight manual/automatic
annotation effort.

For a fixed effective token budget $R$, we select $(\mathbf{W},\mathbf{a})$ to maximize the task score under pruning:
\begin{equation}
\max_{\mathbf{W},\mathbf{a}} \ \ \mathrm{Score}\!\left(\mathcal{D}_{\text{cal}}; R, \mathbf{W}, \mathbf{a}\right)
\quad
\text{s.t.}\ \ a_c\in[0,1],\ \forall c.
\label{eq:cal-obj}
\end{equation}
Here $\mathrm{Score}(\cdot)$ uses the same decoding and evaluation protocol as in the main experiments, but restricted to
$\mathcal{D}_{\text{cal}}$. Concretely, we write
\begin{equation}
\mathrm{Score}\!\left(\mathcal{D}_{\text{cal}}; R, \mathbf{W}, \mathbf{a}\right)
=
\frac{1}{|\mathcal{D}_{\text{cal}}|}
\sum_{(x_n,\mathrm{img}_n,y_n^\star,c_n)\in \mathcal{D}_{\text{cal}}}
\mathrm{Eval}\!\left(\hat{y}_n(R,\mathbf{W},\mathbf{a}),\ y_n^\star\right),
\label{eq:score-def}
\end{equation}
where $\hat{y}_n(R,\mathbf{W},\mathbf{a})$ is the model prediction produced when applying our class-adaptive layer fusion
and dual-stage pruning under budget $R$, and $\mathrm{Eval}(\cdot)$ is the benchmark metric (\textit{e.g.}, exact-match accuracy for QA,
multiple-choice accuracy for MMBench-style evaluation, or any dataset-specific scorer).

\textbf{Why calibration does not require MLLM finetuning.}
The calibration variables $(\mathbf{W},\mathbf{a})$ are extremely low-dimensional and represent high-level architectural priors rather than dense feature mappings. Because the MLLM weights are frozen, the scoring function $\mathrm{Score}$ can be treated as a lightweight objective that does not require computing gradients through the large-scale backbone. While $(\mathbf{W},\mathbf{a})$ could theoretically be optimized via \emph{forward-only} black-box search or lightweight gradient updates, our empirical analysis reveals that the system is remarkably robust to these parameters. Specifically, we find that our discrete subspace search—which selects from structured prototypes tailored to the required level of visual abstraction (\textit{e.g.}, semantic-rich deep layers for identification vs. detail-oriented shallow layers for OCR)—consistently outperforms unconstrained continuous calibration in terms of zero-shot generalization. This inherent robustness allows \textsc{CLASP} to function as a true plug-and-play solution, bypassing the need for expensive calibration cycles or backward activations entirely.

\textbf{Per-class decomposition.}
Since $c_n$ selects the row $\mathbf{w}_{c_n}$ and ratio $a_{c_n}$, the calibration naturally decomposes over categories:
each class $c$ can be calibrated on the subset $\mathcal{D}_{\text{cal}}^{(c)}=\{(x_n,\mathrm{img}_n,y_n^\star,c_n)\in\mathcal{D}_{\text{cal}}:\ c_n=c\}.
$ independently,
which improves stability and reduces search complexity. Algorithm~\ref{alg:calibration} summarizes a simple per-class procedure.

\textbf{Efficient evaluation in practice.}
Calibration can be further accelerated by:
(i) caching the frozen vision-encoder layer outputs $\{Z^{(l)}_n\}_{l=1}^L$ for each $(x_n,\mathrm{img}_n)$ once,
so re-evaluating a candidate $\mathbf{w}_c$ only requires a cheap weighted sum (fusion) plus the standard forward decode;
(ii) using a small, balanced $\mathcal{D}_{\text{cal}}$ per category (hundreds to a few thousand samples typically suffice);
and (iii) optionally using early rejection (evaluate candidates on a small subset first, and fully score only the top few).

\paragraph{Impact of Calibration Set Size and Distribution.}
To rigorously justify our choices regarding the calibration dataset ($\mathcal{D}_{\text{cal}}$), we investigate how the size of the calibration set and its categorical distribution (sampling strategy) affect the final pruned model's performance. 
We evaluate two data selection strategies: a \textit{Balanced} distribution (enforcing an equal number of samples per intent category) and a \textit{Random} distribution (naturally imbalanced based on the source dataset's available distribution). We test total calibration sizes of 4.5k, 9k, and 18k samples. To account for statistical variance, we evaluate each configuration across five random seeds ($[42, 43, 44, 45, 46]$) during the redundancy-aware initialization and report the mean and standard deviation.

As shown in Table~\ref{tab:calibration_size}, our offline search process is exceptionally sample-efficient. Performance strictly saturates around 9k samples, with negligible gains observed when doubling the size to 18k. Furthermore, in the data-scarce regime (4.5k), the \textit{Balanced} sampling strategy significantly reduces variance compared to \textit{Random} sampling (\textit{e.g.}, $\pm6$ vs. $\pm15$ on MME, and $\pm0.3$ vs. $\pm0.6$ on TextVQA). This empirically proves that maintaining a balanced category distribution is crucial for stabilizing the search space and avoiding overfitting to dominant task categories. Consequently, we confirm that using a moderately sized (around 9k), class-balanced calibration set is the optimal strategy to achieve high-performance configurations with minimal computational overhead.

\begin{table}[ht]
    \centering
    \footnotesize
    \caption{\textbf{Ablation on calibration set size and sampling strategy.} Results are reported as mean $\pm$ standard deviation across 5 random seeds. The framework achieves optimal, low-variance performance with a balanced dataset of merely 9k samples, demonstrating excellent sample efficiency.}
    \label{tab:calibration_size}
    \begin{tabular}{c l c c c}
        \toprule
        \textbf{Calibration Size} & \textbf{Sampling Strategy} & \textbf{MME} & \textbf{TextVQA} & \textbf{MMVet} \\
        \midrule
        4.5k & Balanced (Avg) & $1814 \pm 6$  & $56.6 \pm 0.3$ & $32.0 \pm 0.3$ \\
        4.5k & Random         & $1809 \pm 15$ & $56.4 \pm 0.6$ & $32.1 \pm 0.4$ \\
        \midrule
        9k   & Balanced (Avg) & $1845 \pm 3$  & $57.5 \pm 0.1$ & $33.2 \pm 0.1$ \\
        9k   & Random         & $1840 \pm 5$  & $57.6 \pm 0.2$ & $33.2 \pm 0.2$ \\
        \midrule
        18k  & Balanced (Avg) & $1845 \pm 2$  & $57.5 \pm 0.1$ & $33.3 \pm 0.0$ \\
        18k  & Random         & $1845 \pm 3$  & $57.6 \pm 0.1$ & $33.2 \pm 0.1$ \\
        \bottomrule
    \end{tabular}
\end{table}

\textbf{Structured parameterization for fast search.}
To avoid overfitting and to reduce search complexity, we recommend a structured parameterization:
(i) restrict $\mathbf{w}_c$ to have support on a small candidate layer set (\textit{e.g.}, 3--5 layers spanning shallow/mid/deep),
and (ii) restrict $a_c$ to a coarse grid (\textit{e.g.}, $\{0.2,0.4,0.6,0.8\}$ or a slightly denser grid when needed).
This yields a compact search space while preserving the key adaptivity signals revealed by our motivation study.

\textbf{Optional continuous refinement (if desired).}
If one wants to refine beyond a discrete candidate set, one can keep the same constraints but allow non-uniform weights on
the selected layers (still normalized by softmax), and use derivative-free optimizers (\textit{e.g.}, coordinate search / random search)
on the low-dimensional parameters. This is optional and not required for stable gains.

\paragraph{Empirical Cost of Offline Calibration and Online Routing.}
To quantitatively support the efficiency claims of our calibration procedure, we report the wall-clock time and peak GPU memory usage required for both the offline search and the online intent classification. 
The measurements were conducted on a single NVIDIA A100 (80GB) GPU.

As detailed in Table~\ref{tab:calibration_cost}, the offline calibration is highly resource-efficient. 
Thanks to the caching of frozen vision-encoder features (as discussed in our efficient evaluation strategy), the discrete subspace search for Layer Weights ($W$) and Split Ratios ($a$) takes approximately $1.5$ hours and $1.0$ hour, respectively. The peak memory footprint remains well under 20GB, making this calibration completely feasible on consumer-grade hardware without requiring large-scale distributed clusters or gradient backpropagation.
Furthermore, during online inference, the intent classifier introduces a negligible latency of less than $2$ms and requires only $15.5$GB of memory (which includes the loaded weights of the lightweight routing model), ensuring that dynamic class-adaptive routing does not become a system bottleneck.

\begin{table}[ht]
    \centering
    \footnotesize
    \caption{\textbf{Time and memory cost of offline calibration and online routing.} The offline calibration of layer weights and split ratios takes only a few hours on a single GPU without requiring backpropagation. The online intent classifier introduces negligible latency.}
    \label{tab:calibration_cost}
    \begin{tabular}{l c c}
        \toprule
        \textbf{Component} & \textbf{Time Cost} & \textbf{Memory Usage} \\
        \midrule
        Intent Classifier (Online Inference) & $<2$ ms & 15.5 GB \\
        Calibration of Layer Weights ($W$) (Offline) & $\approx 1.5$ h & 19.4 GB \\
        Calibration of Split Ratios ($a$) (Offline) & $\approx 1.0$ h & 17.6 GB \\
        \bottomrule
    \end{tabular}
\end{table}

\subsection{Fusion with hierarchical backbones}
\label{app:hier-fusion}

Our exposition assumes token indices align across layers (a standard property of ViT-style encoders).
For hierarchical encoders with resolution changes (\textit{e.g.}, feature pyramids), fusion can be supported by mapping each
layer to a common token set before applying Eq.~\eqref{eq:fusion}.

\textbf{Token alignment via spatial resampling.}
Let layer $l$ output a grid of patch tokens with spatial resolution $H_l\times W_l$ (excluding any special token),
and let the target resolution be $H_\star\times W_\star$ (typically the final layer resolution).
We define an alignment operator $\mathrm{Align}_l(\cdot)$ that maps
$Z^{(l)}_n \in \mathbb{R}^{(H_lW_l)\times d_v}$ to $\hat{Z}^{(l)}_n \in \mathbb{R}^{(H_\star W_\star)\times d_v}$:
\begin{equation}
\hat{Z}^{(l)}_n = \mathrm{Align}_l\!\left(Z^{(l)}_n\right),
\qquad
\mathrm{Align}_l \;=\;
\begin{cases}
\text{bilinear/nearest interpolation (upsample)} & \text{if } H_lW_l < H_\star W_\star,\\
\text{average pooling / strided pooling (downsample)} & \text{if } H_lW_l > H_\star W_\star,\\
\text{identity} & \text{if } H_lW_l = H_\star W_\star.
\end{cases}
\label{eq:hier-align}
\end{equation}
In practice, we reshape tokens back to a feature map of size $H_l\times W_l\times d_v$, apply standard 2D resampling, and
flatten back to a sequence. If the backbone includes a special token (\textit{e.g.}, [CLS]), we either (i) keep it separate and fuse it
with the same mixture weights, or (ii) drop it when it is not used by the downstream projector.

After alignment, layer fusion proceeds exactly as in Eq.~\eqref{eq:fusion} by mixing the aligned tokens
$\{\hat{Z}^{(l)}_n\}_{l=1}^L$.
The computational overhead of fusion remains linear in the number of visual tokens (and linear in the number of fused layers);
see Appendix~B.6 for a step-wise complexity breakdown.

\begin{algorithm}[t]
\caption{Discrete Subspace Search for Optimal Layer Prototypes and Pruning Ratios.}
\label{alg:calibration}
\begin{algorithmic}[1]
\REQUIRE calibration set $\mathcal{D}_{\text{cal}}$ with labels $c_n$; candidate layer sets $\{\mathcal{L}^{(m)}\}$;
candidate ratios $\mathcal{A}$; budget $R$
\FOR{$c\in\mathcal{N}_c$}
    \STATE Initialize best score $s^\star\leftarrow -\infty$
    \FOR{layer candidate $\mathcal{L}^{(m)}$}
        \FOR{$a\in\mathcal{A}$}
            \STATE Define $\mathbf{w}_c$ supported on $\mathcal{L}^{(m)}$ (\textit{e.g.}, uniform over $\mathcal{L}^{(m)}$, $-\infty$ elsewhere)
            \STATE Evaluate $\mathrm{Score}$ on $\mathcal{D}_{\text{cal}}^{(c)}$ using $(\mathbf{w}_c,a)$
            \IF{score $> s^\star$}
                \STATE Update $(\mathbf{w}_c^\star,a_c^\star)\leftarrow(\mathbf{w}_c,a)$; $s^\star\leftarrow$ score
            \ENDIF
        \ENDFOR
    \ENDFOR
\ENDFOR
\STATE Return $\mathbf{W}=[\mathbf{w}_0^\star;\ldots;\mathbf{w}_{C-1}^\star]$, $\mathbf{a}=[a_0^\star,\ldots,a_{C-1}^\star]$
\end{algorithmic}
\end{algorithm}

\section{Additional Theoretical Analysis}
\label{sec:appendix-theory}

\subsection{Notation and setup}
For sample $n$, let $M_n \triangleq |\mathcal{V}_n|$ denote the number of visual patch tokens.
Layer-wise tokens are $\mathbf{z}^{(l)}_{n,t}\in\mathbb{R}^{d_v}$ for $l\in\{1,\ldots,L\}$ and $t\in\mathcal{V}_n$.
Class-adaptive fusion produces $\bar{\mathbf{z}}_{n,t}\in\mathbb{R}^{d_v}$ and aligned tokens
$\tilde{\mathbf{z}}_{n,t}=f_{\mathrm{proj}}(\bar{\mathbf{z}}_{n,t})\in\mathbb{R}^{d}$.

\paragraph{Selection operators.}
For a collection of real-valued scores $\{s_t\}_{t\in\mathcal{I}}$ indexed by $\mathcal{I}$, we denote by
$\mathrm{Top}_{K}(\{s_t\}_{t\in\mathcal{I}})$ the index set of the $K$ largest scores, and by
$\mathrm{Bottom}_{K}(\{s_t\}_{t\in\mathcal{I}})$ the index set of the $K$ smallest scores.
When ties occur, any deterministic tie-breaking rule suffices and does not affect the statements below.

\paragraph{Simplex notation.}
We denote the probability simplex by
\begin{equation}
\Delta^{L-1}\triangleq \Bigl\{\boldsymbol{\gamma}\in\mathbb{R}^{L}\;:\;\gamma_l\ge 0~(\forall l),\ \sum_{l=1}^{L}\gamma_l=1\Bigr\}.
\label{eq:app-simplex}
\end{equation}

\paragraph{Unit-normalization for cosine similarity.}
Whenever cosine similarity is used, we work with $\ell_2$-normalized features
\begin{equation}
\mathbf{u}_{n,t}\triangleq \frac{\tilde{\mathbf{z}}_{n,t}}{\|\tilde{\mathbf{z}}_{n,t}\|_2}\in\mathbb{R}^{d}.
\label{eq:app-unit}
\end{equation}
This normalization isolates \emph{directional} information and makes inner products
$\mathbf{u}_{n,t}^{\top}\mathbf{u}_{n,j}$ equal cosine similarity, which is the natural notion of redundancy/novelty on the unit sphere. Consequently, Euclidean distance on this manifold becomes strictly monotonic with respect to angular separation, simplifying the geometric proofs in Section~\ref{sec:app-geometry}.

% ============================================================
\subsection{Properties of class-adaptive layer fusion}
\label{sec:app-fusion}

\paragraph{Fusion equations in execution order.}
For clarity, we restate the fusion module as a short sequence of equations.

\begin{enumerate}
\item \textbf{Routing and mixture weights.}
Given the prompt $x_n$, a text-only router predicts a class $c_n$, which selects a row $\mathbf{w}_{c_n}\in\mathbb{R}^{L}$ from the layer-score matrix $\mathbf{W}$.
We then convert it into a probability distribution over layers:
\begin{equation}
c_n \leftarrow \mathrm{Route}(x_n),
\qquad
\boldsymbol{\alpha}_n \triangleq \mathrm{softmax}(\tau\,\mathbf{w}_{c_n})\in\Delta^{L-1}.
\label{eq:app-alpha}
\end{equation}
Here $\tau>0$ is a temperature that controls how ``peaky'' the layer preference is.

\item \textbf{Token-wise convex fusion across layers.}
For each token index $t\in\mathcal{V}_n$, we fuse its layer-wise representations by a weighted sum. This dynamically integrates features from varying levels of visual abstraction:
\begin{equation}
\bar{\mathbf{z}}_{n,t}=\sum_{l=1}^{L}\alpha_{n,l}\mathbf{z}^{(l)}_{n,t}.
\label{eq:app-fusion}
\end{equation}

\item \textbf{Projection into the decoder space.}
The fused token is mapped to the decoder embedding space through $f_{\mathrm{proj}}$:
\begin{equation}
\tilde{\mathbf{z}}_{n,t}=f_{\mathrm{proj}}(\bar{\mathbf{z}}_{n,t})\in\mathbb{R}^{d}.
\label{eq:app-proj}
\end{equation}
\end{enumerate}

\begin{lemma}[To keep fused features inside the layer-wise convex hull, we fuse by a convex combination]
\label{lem:convex-fusion}
For any sample $n$ and token index $t\in\mathcal{V}_n$, the fused token in Eq.~\eqref{eq:app-fusion}
lies in the convex hull of $\{\mathbf{z}^{(l)}_{n,t}\}_{l=1}^{L}$.
\end{lemma}

\begin{proof}
By Eq.~\eqref{eq:app-alpha}, $\boldsymbol{\alpha}_n=\mathrm{softmax}(\tau \mathbf{w}_{c_n})$ satisfies $\alpha_{n,l}\ge 0$ and $\sum_{l=1}^{L}\alpha_{n,l}=1$,
hence Eq.~\eqref{eq:app-fusion} is a convex combination.
Geometrically, the fused vector lies inside the layer-wise feature polytope.
\end{proof}

\paragraph{Interpretation.}
Eq.~\eqref{eq:app-fusion} prevents ``out-of-manifold'' extrapolation across layers:
fusion interpolates between representations at different depths but cannot invent directions outside their convex span.
This is desirable when fusion weights are controlled by a text router, ensuring the resulting representation remains within the valid semantic latent space.

\begin{lemma}[To interpolate between uniform averaging and hard layer selection, we analyze the temperature limits]
\label{lem:temp-limits}
Let $\boldsymbol{\alpha}(\tau)=\mathrm{softmax}(\tau\mathbf{w})$ for any fixed $\mathbf{w}\in\mathbb{R}^{L}$.
(i) As $\tau\to 0$, $\boldsymbol{\alpha}(\tau)\to (1/L)\mathbf{1}$.
(ii) As $\tau\to\infty$, $\boldsymbol{\alpha}(\tau)$ converges to a one-hot distribution supported on $\arg\max_{l} w_l$ when the maximizer is unique.
\end{lemma}

\begin{proof}
(i) $\mathrm{softmax}(0\cdot \mathbf{w})$ is uniform; continuity gives the limit.
(ii) For large $\tau$, $\exp(\tau w_l)$ is dominated by the largest $w_l$, yielding the standard softmax one-hot asymptotics, effectively selecting the single maximal layer.
\end{proof}

\begin{assumption}[To enable norm-based perturbation bounds, we assume bounded layer-wise token norms]
\label{assm:bounded}
There exists $B>0$ such that for all $n,t,l$, $\|\mathbf{z}^{(l)}_{n,t}\|_2 \le B$, uniformly bounding the feature magnitude.
\end{assumption}

\begin{lemma}[To control how routing-score noise changes fusion weights, we show softmax is $\ell_1$-Lipschitz]
\label{lem:softmax-lipschitz}
Let $\boldsymbol{\alpha}(\mathbf{u})=\mathrm{softmax}(\mathbf{u})$.
Then for any $\mathbf{u},\mathbf{v}\in\mathbb{R}^{L}$, representing arbitrary input logit vectors,
\begin{equation}
\|\boldsymbol{\alpha}(\mathbf{u})-\boldsymbol{\alpha}(\mathbf{v})\|_1
\le \frac{1}{2}\|\mathbf{u}-\mathbf{v}\|_1.
\label{eq:app-softmax-lip}
\end{equation}
Consequently, with $\boldsymbol{\alpha}(\tau\mathbf{w})=\mathrm{softmax}(\tau\mathbf{w})$, applied to the temperature-scaled inputs,
\begin{equation}
\|\boldsymbol{\alpha}(\tau\mathbf{w})-\boldsymbol{\alpha}(\tau\mathbf{w}')\|_1
\le \frac{\tau}{2}\|\mathbf{w}-\mathbf{w}'\|_1.
\label{eq:app-softmax-lip-scaled}
\end{equation}
\end{lemma}

\begin{proof}
We make the dependence explicit in three steps, mirroring the execution logic.

\begin{enumerate}
\item[(a)] \textbf{Jacobian form.}
Let $J(\mathbf{u})=\nabla_{\mathbf{u}}\boldsymbol{\alpha}(\mathbf{u})$ with entries
$J_{ij}=\alpha_i(\delta_{ij}-\alpha_j)$.

\item[(b)] \textbf{Column-wise $\ell_1$ bound.}
For any column $j$, summing the absolute values of entries,
\[
\sum_i |J_{ij}|
= |\alpha_j(1-\alpha_j)|+\sum_{i\ne j}|\!-\alpha_i\alpha_j|
= \alpha_j(1-\alpha_j)+\alpha_j\sum_{i\ne j}\alpha_i
=2\alpha_j(1-\alpha_j)\le \tfrac12,
\]
since $\max_{x\in[0,1]} 2x(1-x)=1/2$.

\item[(c)] \textbf{Mean value theorem.}
The induced operator norm satisfies $\|J(\mathbf{u})\|_{1\to 1}\le 1/2$ for all $\mathbf{u}$.
Therefore,
\[
\|\boldsymbol{\alpha}(\mathbf{u})-\boldsymbol{\alpha}(\mathbf{v})\|_1
\le \sup_{\xi}\|J(\xi)\|_{1\to 1}\,\|\mathbf{u}-\mathbf{v}\|_1 \le \tfrac12\|\mathbf{u}-\mathbf{v}\|_1.
\]
Eq.~\eqref{eq:app-softmax-lip-scaled} follows by substituting $\mathbf{u}=\tau\mathbf{w}$ and $\mathbf{v}=\tau\mathbf{w}'$.
\end{enumerate}
\end{proof}

\begin{proposition}[To guarantee robustness to routing-score perturbations, we bound fused-token drift linearly in $\tau$]
\label{prop:end2end-stability}
Under Assumption~\ref{assm:bounded}, for any two layer-score vectors $\mathbf{w},\mathbf{w}'\in\mathbb{R}^{L}$ and fixed $\tau$,
\begin{equation}
\|\bar{\mathbf{z}}_{n,t}(\mathbf{w})-\bar{\mathbf{z}}_{n,t}(\mathbf{w}')\|_2
\le \frac{B\tau}{2}\|\mathbf{w}-\mathbf{w}'\|_1.
\label{eq:app-end2end}
\end{equation}
If $f_{\mathrm{proj}}$ is $L_{\mathrm{proj}}$-Lipschitz, then
$\|\tilde{\mathbf{z}}_{n,t}(\mathbf{w})-\tilde{\mathbf{z}}_{n,t}(\mathbf{w}')\|_2
\le \frac{B\tau L_{\mathrm{proj}}}{2}\|\mathbf{w}-\mathbf{w}'\|_1$.
\end{proposition}

\begin{proof}
By applying the triangle inequality with the norm bound $B$ and Lemma~\ref{lem:softmax-lipschitz}, we have:
\[
\|\bar{\mathbf{z}}(\mathbf{w})-\bar{\mathbf{z}}(\mathbf{w}')\|_2
\le B\|\boldsymbol{\alpha}(\tau\mathbf{w})-\boldsymbol{\alpha}(\tau\mathbf{w}')\|_1
\le \frac{B\tau}{2}\|\mathbf{w}-\mathbf{w}'\|_1.
\]
The projector bound follows from Lipschitzness of $f_{\mathrm{proj}}$, scaling the error by the constant.
\end{proof}

\paragraph{Practical reading.}
Eq.~\eqref{eq:app-end2end} makes the role of $\tau$ explicit:
larger $\tau$ yields sharper (more class-distinct) depth preferences (Lemma~\ref{lem:temp-limits}),
but also linearly amplifies sensitivity to score perturbations.
This justifies conservative $\tau$ (or annealing) to effectively dampen the amplification of routing noise when router uncertainty is non-negligible.

\begin{corollary}[To upper bound misrouting sensitivity, we relate class-to-class drift to inter-row distance in $\mathbf{W}$]
\label{cor:misroute}
Let $c$ and $c'$ be two classes with score vectors $\mathbf{w}_{c}$ and $\mathbf{w}_{c'}$.
Under Assumption~\ref{assm:bounded}, the fused-token drift caused by misrouting $c\rightarrow c'$, which corresponds to the feature shift induced by the incorrect selection, satisfies
\begin{equation}
\|\bar{\mathbf{z}}_{n,t}(\mathbf{w}_{c})-\bar{\mathbf{z}}_{n,t}(\mathbf{w}_{c'})\|_2
\le \frac{B\tau}{2}\|\mathbf{w}_{c}-\mathbf{w}_{c'}\|_1.
\label{eq:app-misroute}
\end{equation}
\end{corollary}

\begin{proof}
Apply Proposition~\ref{prop:end2end-stability} by substituting the specific class vectors $\mathbf{w}=\mathbf{w}_c$ and $\mathbf{w}'=\mathbf{w}_{c'}$.
\end{proof}

% ============================================================
\subsection{Pruning as an explicit relevance--coverage optimization with redundancy-aware seeding and clustering}
\label{sec:app-prune}

We analyze the class-adaptive two-stage pruning under a fixed budget split $K_1=\lfloor a_n R\rfloor$ and $K_2=R-K_1$ where $a_n \triangleq a_{c_n}$.
Stage~I selects \emph{relevance pivots} using saliency scores $\phi_{n,t}$, and Stage~II selects \emph{diversity completion} tokens
from the non-pivot pool $\mathcal{U}_n=\mathcal{V}_n\setminus \mathcal{P}_n$, designed to capture semantic information missed by the saliency-based selection.

\paragraph{Pruning equations in execution order.}
We restate the pruning pipeline as a numbered ``recipe'' for easier reference.

\begin{enumerate}
\item \textbf{Attention matrix and saliency (relevance score).}
Given an attention block with query/key/value $Q,K,V$, define
\begin{equation}
\mathrm{Attn}(Q,K,V) = \mathrm{softmax}\!\Bigl(\frac{QK^{\top}}{\sqrt{d_k}}\Bigr)V,
\qquad
\mathbf{A}_n\triangleq \mathrm{softmax}\!\Bigl(\frac{QK^{\top}}{\sqrt{d_k}}\Bigr),
\label{eq:app-attn}
\end{equation}
where $\mathbf{A}_n$ is the attention matrix (typically averaged over heads in practice).
For a compact reference index set $\mathcal{S}$ (\textit{e.g.}, visual [CLS] or selected text instruction tokens), we score a visual token $t\in\mathcal{V}_n$ by
\begin{equation}
\phi_{n,t}\triangleq \frac{1}{|\mathcal{S}|}\sum_{i\in\mathcal{S}} A_{n,i,t},
\qquad t\in\mathcal{V}_n.
\label{eq:app-saliency}
\end{equation}
Here $A_{n,i,t}$ measures how much the reference token $i$ attends to visual token $t$; larger $\phi_{n,t}$ indicates higher instruction relevance, effectively highlighting the visual regions most aligned with the user prompt.

\item \textbf{Budget split into relevance pivots vs.\ diversity completion.}
Given a class-dependent ratio $a_n\in[0,1]$ and total budget $R$, determining the specific allocation size for the relevance and diversity stages:
\begin{equation}
K_1=\lfloor a_n R\rfloor,\qquad K_2=R-K_1.
\label{eq:app-split}
\end{equation}

\item \textbf{Stage I: pivots by top-$K_1$ saliency.}
We keep the $K_1$ most salient tokens as pivots:
\begin{equation}
\mathcal{P}_n=\mathrm{Top}_{K_1}\bigl(\{\phi_{n,t}\}_{t\in\mathcal{V}_n}\bigr).
\label{eq:app-pivots}
\end{equation}
This stage is ``relevance-first'': it aggressively preserves tokens strongly queried by the instruction.

\item \textbf{Pivot-relative redundancy.}
Let $\mathcal{U}_n=\mathcal{V}_n\setminus\mathcal{P}_n$ be the non-pivot pool.
Using unit features $\mathbf{u}_{n,t}$ (Eq.~\eqref{eq:app-unit}), define redundancy of a candidate token $t\in\mathcal{U}_n$ as
\begin{equation}
\rho_{n,t}\triangleq \max_{j\in\mathcal{P}_n}\mathbf{u}_{n,t}^{\top}\mathbf{u}_{n,j},
\qquad t\in\mathcal{U}_n.
\label{eq:app-redundancy}
\end{equation}
Because $\mathbf{u}_{n,t}$ are unit vectors, $\rho_{n,t}\in[-1,1]$ is cosine similarity.
Large $\rho_{n,t}$ means $t$ is highly duplicated by some pivot; small $\rho_{n,t}$ means $t$ lies in a direction poorly covered by pivots.

\item \textbf{Deterministic redundancy-aware seeding for Stage II.}
We view $\rho_{n,t}$ as a \emph{pivot-overlap cost}: it measures how much a candidate token $t\in\mathcal{U}_n$ resembles the already-selected pivot set $\mathcal{P}_n$ (worst-case cosine overlap).
To make our seeding rule explicit, for any seed set $\mathcal{C}\subseteq\mathcal{U}_n$ with $|\mathcal{C}|=K_2$, we define the total pivot-relative redundancy as
\begin{equation}
D(\mathcal{C}\mid \mathcal{P}_n)\triangleq \sum_{t\in\mathcal{C}} \rho_{n,t}.
\label{eq:app-seedred}
\end{equation}
Since Eq.~\eqref{eq:app-seedred} is additive over tokens, the global minimizer under $|\mathcal{C}|=K_2$ is obtained by selecting the $K_2$ tokens with the smallest $\rho_{n,t}$, yielding a deterministic initialization that starts from directions complementary to $\mathcal{P}_n$.
Note that $D$ targets \emph{seed-to-pivot} redundancy; redundancy \emph{among} completion tokens is then reduced by the subsequent $K$-means refinement.

\item \textbf{Bottom-$K_2$ seeding (optimal for $D$).}
The minimizer of Eq.~\eqref{eq:app-seedred} is obtained by selecting the $K_2$ least redundant tokens, since the sum is minimized by the smallest individual terms:
\begin{equation}
\mathcal{C}^{(0)}_n = \mathrm{Bottom}_{K_2}\bigl(\{\rho_{n,t}\}_{t\in\mathcal{U}_n}\bigr).
\label{eq:app-seed-bottom}
\end{equation}
We use the corresponding unit features as initial spherical $K$-means centers, \textit{i.e.}, set $\boldsymbol{\mu}^{(0)}_{n,k}=\mathbf{u}_{n,c_k}$ for $c_k\in\mathcal{C}^{(0)}_n$.

\item \textbf{Spherical $K$-means refinement (coverage).}
With cosine similarity, spherical $K$-means iterates between:
\begin{enumerate}
\item[(a)] \textbf{Assignment (nearest-center by cosine).}
\begin{equation}
s^{(r)}(t) = \arg\max_{k\in\{1,\dots,K_2\}} \mathbf{u}_{n,t}^{\top}\boldsymbol{\mu}^{(r-1)}_{n,k},
\qquad t\in\mathcal{U}_n.
\label{eq:app-kmeans-assign}
\end{equation}
\item[(b)] \textbf{Update (normalized mean direction).}
\begin{equation}
\boldsymbol{\mu}^{(r)}_{n,k}
=\frac{\sum_{t:s^{(r)}(t)=k}\mathbf{u}_{n,t}}{\left\|\sum_{t:s^{(r)}(t)=k}\mathbf{u}_{n,t}\right\|_2},
\qquad k\in\{1,\dots,K_2\}.
\label{eq:app-kmeans-update}
\end{equation}
\end{enumerate}
We run $r=1,\dots,T$ iterations, where small $T$ (\textit{e.g.}, $5$) is typically sufficient in practice.

\item \textbf{Discrete completion tokens via cluster medoids.}
After $T$ iterations, we select one representative per cluster:
\begin{equation}
q_{n,k}=\arg\max_{t:s^{(T)}(t)=k}\mathbf{u}_{n,t}^\top\boldsymbol{\mu}^{(T)}_{n,k},
\qquad
\mathcal{Q}_n=\{q_{n,k}\}_{k=1}^{K_2}\subseteq \mathcal{U}_n.
\label{eq:app-medoid}
\end{equation}
Eq.~\eqref{eq:app-medoid} returns \emph{token indices} (not continuous centers), so the retained set is directly usable by the decoder.
\end{enumerate}

\begin{proposition}[To initialize clustering away from pivots, Bottom-$K_2$ seeding minimizes total seed redundancy]
\label{prop:seed-opt}
Fix $\mathcal{P}_n$ and let $\mathcal{U}_n=\mathcal{V}_n\setminus\mathcal{P}_n$.
Among all seed sets $\mathcal{C}\subseteq\mathcal{U}_n$ with $|\mathcal{C}|=K_2$, the minimizer of $D(\mathcal{C}\mid\mathcal{P}_n)$ in Eq.~\eqref{eq:app-seedred} is
$\mathcal{C}^{(0)}_n$ in Eq.~\eqref{eq:app-seed-bottom}, providing the global optimum for the additive redundancy cost function.
\end{proposition}

\begin{proof}
Eq.~\eqref{eq:app-seedred} is a sum of independent scalar costs $\rho_{n,t}$ over the chosen indices.
Thus the minimum over all $K_2$-subsets is achieved by selecting the $K_2$ smallest $\rho_{n,t}$ values, \textit{i.e.}, Eq.~\eqref{eq:app-seed-bottom}.
\end{proof}

\paragraph{A coverage objective for spherical $K$-means.}
Define the directional coherence (coverage) objective
\begin{equation}
\mathcal{J}(\{\boldsymbol{\mu}_{n,k}\}_{k=1}^{K_2})
\triangleq
\sum_{t\in\mathcal{U}_n}\max_{k\in\{1,\dots,K_2\}}\mathbf{u}_{n,t}^\top \boldsymbol{\mu}_{n,k}.
\label{eq:app-kmeans-obj}
\end{equation}
Each summand in Eq.~\eqref{eq:app-kmeans-obj} rewards a token by its cosine similarity to the nearest center.
Maximizing $\mathcal{J}$ therefore encourages the centers to cover multiple modes in $\mathcal{U}_n$ rather than collapsing to a single region.

\begin{proposition}[To monotonically improve coverage of $\mathcal{U}_n$, spherical $K$-means performs coordinate ascent on $\mathcal{J}$]
\label{prop:kmeans-mono}
The alternating updates in Eqs.~\eqref{eq:app-kmeans-assign}--\eqref{eq:app-kmeans-update} do not decrease $\mathcal{J}$ in Eq.~\eqref{eq:app-kmeans-obj}.
Moreover, since $\mathcal{J}$ is bounded above, the procedure converges to a stationary point (a local optimum) of $\mathcal{J}$.
\end{proposition}

\begin{proof}
We verify monotonicity in the same two-step order as the algorithm.

\begin{enumerate}
\item[(a)] \textbf{Assignment improves $\mathcal{J}$ for fixed centers.}
For fixed $\{\boldsymbol{\mu}_{n,k}\}$, choosing $s^{(r)}(t)=\arg\max_k \mathbf{u}_{n,t}^\top\boldsymbol{\mu}_{n,k}$
maximizes each summand in Eq.~\eqref{eq:app-kmeans-obj}, guaranteeing a non-decreasing objective value, hence cannot decrease $\mathcal{J}$.

\item[(b)] \textbf{Update improves $\mathcal{J}$ for fixed assignments.}
For fixed assignments, the contribution of cluster $k$ equals
$\sum_{t:s^{(r)}(t)=k}\mathbf{u}_{n,t}^\top \boldsymbol{\mu}_{n,k}$ subject to $\|\boldsymbol{\mu}_{n,k}\|_2=1$.
By Cauchy--Schwarz, this is maximized by setting $\boldsymbol{\mu}_{n,k}$ proportional to $\sum_{t:s^{(r)}(t)=k}\mathbf{u}_{n,t}$
and normalizing, yielding the optimal directional center, \textit{i.e.}, Eq.~\eqref{eq:app-kmeans-update}.
\end{enumerate}

Since each iteration is non-decreasing and $\mathcal{J}\le |\mathcal{U}_n|$ (each term $\le 1$), monotone ascent implies convergence.
\end{proof}

% ============================================================
\subsection{Geometric diversity guarantees induced by redundancy-aware seeding}
\label{sec:app-geometry}

\begin{lemma}[To translate cosine redundancy into geometry, we relate it to Euclidean distance on the unit sphere]
\label{lem:cos-euclid}
For any unit vectors $\mathbf{u},\mathbf{v}$, we directly have $\|\mathbf{u}-\mathbf{v}\|_2^2 = 2(1-\mathbf{u}^\top\mathbf{v})$.
\end{lemma}

\begin{proof}
Expand
$\|\mathbf{u}-\mathbf{v}\|_2^2=\|\mathbf{u}\|_2^2+\|\mathbf{v}\|_2^2-2\mathbf{u}^\top\mathbf{v}=2-2\mathbf{u}^\top\mathbf{v}$.
\end{proof}

\begin{corollary}[To ensure seeds start outside pivot neighborhoods, separation is bounded by redundancy thresholds]
\label{cor:separation}
Let $\mathcal{C}^{(0)}_n$ be the seed set in Eq.~\eqref{eq:app-seed-bottom}.
Define the (deterministic) redundancy threshold
\begin{equation}
\delta_n \triangleq \max_{t\in\mathcal{C}^{(0)}_n}\rho_{n,t},
\label{eq:app-delta}
\end{equation}
which is equivalently the $K_2$-th smallest value among $\{\rho_{n,t}\}_{t\in\mathcal{U}_n}$.
Then for every seed $\mathbf{c}\in\{\mathbf{u}_{n,t}:t\in\mathcal{C}^{(0)}_n\}$, which are chosen to minimize the redundancy with the pre-selected pivots, and any pivot $j\in\mathcal{P}_n$,
\begin{equation}
\mathbf{c}^\top \mathbf{u}_{n,j}\le \delta_n
\qquad\text{and}\qquad
\|\mathbf{c}-\mathbf{u}_{n,j}\|_2 \ge \sqrt{2(1-\delta_n)}.
\label{eq:app-sep}
\end{equation}
\end{corollary}

\begin{proof}
By construction of $\mathcal{C}^{(0)}_n$, each selected seed index $t\in\mathcal{C}^{(0)}_n$ satisfies $\rho_{n,t}\le \delta_n$.
Since $\rho_{n,t}=\max_{j\in\mathcal{P}_n}\mathbf{u}_{n,t}^\top\mathbf{u}_{n,j}$, we have
$\mathbf{u}_{n,t}^\top\mathbf{u}_{n,j}\le \rho_{n,t}\le \delta_n$ for all pivots $j$.
Applying Lemma~\ref{lem:cos-euclid} yields the Euclidean bound.
\end{proof}

\paragraph{Interpretation.}
Eq.~\eqref{eq:app-sep} formalizes what redundancy-aware seeding enforces geometrically:
initial centers lie outside a cosine-similarity ``cap'' around each pivot, guaranteeing a minimum Euclidean distance to the pivot set.
Spherical $K$-means then refines these centers to better represent modes of $\mathcal{U}_n$ without relying on random initialization.

% ============================================================
\subsection{Computational complexity}
\label{sec:appendix-complexity}

Let $M_n=|\mathcal{V}_n|$, embedding dimension be $d$, and $K_1=\lfloor a_n R\rfloor$, $K_2=R-K_1$.
We summarize the additional overhead incurred by pruning beyond a standard forward pass in the same step-wise style.

\begin{enumerate}
\item \textbf{Layer fusion (Eq.~\eqref{eq:app-fusion}).}
Computing $\{\bar{\mathbf{z}}_{n,t}\}_{t\in\mathcal{V}_n}$ costs $O(L\,M_n\,d_v)$: a single weighted sum over cached layer outputs.

\item \textbf{Saliency top-$K_1$ (Eq.~\eqref{eq:app-pivots}).}
Selecting $\mathrm{Top}_{K_1}$ costs $O(M_n\log K_1)$ via partial sort / heap.

\item \textbf{Redundancy computation for seeding (Eq.~\eqref{eq:app-redundancy}).}
Naively, computing $\rho_{n,t}=\max_{j\in\mathcal{P}_n}\mathbf{u}_{n,t}^{\top}\mathbf{u}_{n,j}$ for all $t\in\mathcal{U}_n$ costs
$O((M_n-K_1)K_1 d)$.
This can be implemented as a matrix multiplication between
$\mathbf{U}_{\mathcal{U}}\in\mathbb{R}^{(M_n-K_1)\times d}$ and
$\mathbf{U}_{\mathcal{P}}\in\mathbb{R}^{K_1\times d}$ followed by a row-wise max, yielding the required redundancy values.

\item \textbf{Bottom-$K_2$ seeding (Eq.~\eqref{eq:app-seed-bottom}).}
Selecting $\mathrm{Bottom}_{K_2}$ costs $O((M_n-K_1)\log K_2)$.

\item \textbf{Spherical $K$-means refinement (Eqs.~\eqref{eq:app-kmeans-assign}--\eqref{eq:app-kmeans-update}).}
Each iteration costs $O((M_n-K_1)K_2 d)$ for similarity evaluation/assignment plus $O((M_n-K_1)d)$ for accumulating cluster sums and normalization.
Over $T$ iterations, the refinement cost is $O(T(M_n-K_1)K_2 d)$, which remains efficient for small iteration counts.

\item \textbf{Medoid selection (Eq.~\eqref{eq:app-medoid}).}
Computing similarities of each token to its final cluster center costs $O((M_n-K_1)d)$.
\end{enumerate}

Memory overhead is dominated by storing normalized token matrices, \textit{e.g.}, $O((M_n-K_1)d)$ if $\mathbf{U}_{\mathcal{U}}$ is materialized.
Since $K_2\le R$ and $T$ is a small constant (\textit{e.g.}, $T{=}5$), the refinement term remains lightweight compared to the quadratic attention cost
of processing long visual sequences in the decoder, typically dominating the total runtime.

\section{More Experiment}
\label{sec:appendix-method-2}

\begin{figure}[h]
    \centering
    \includegraphics[width=0.95\linewidth]{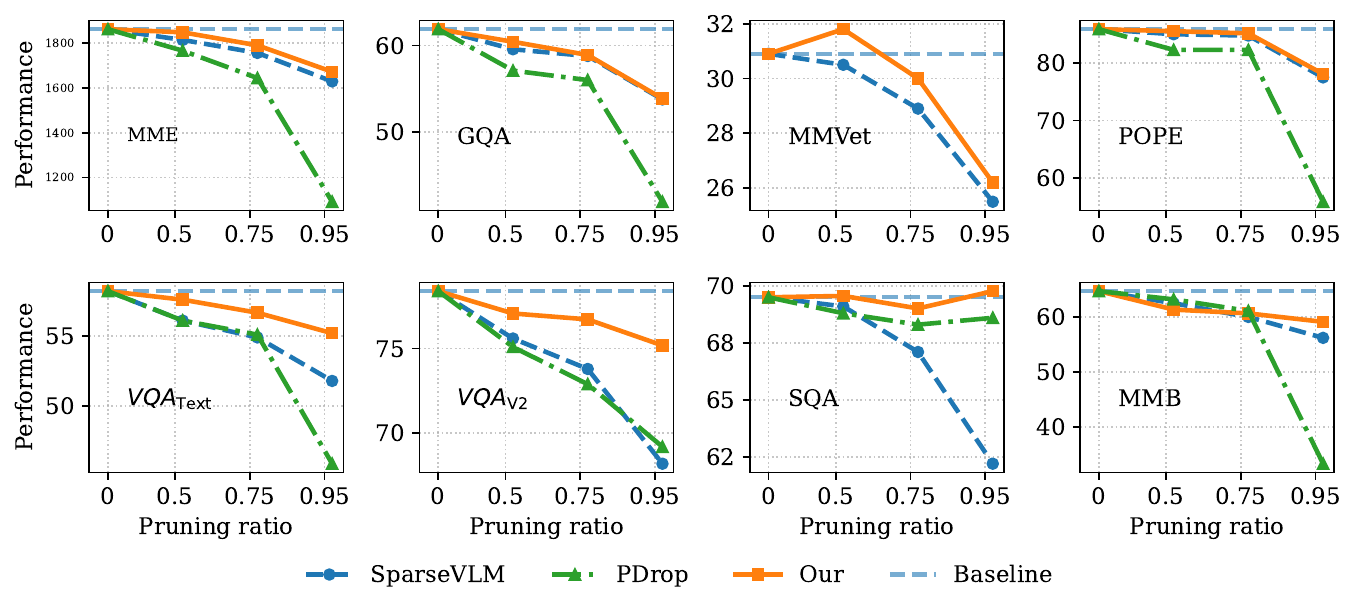}
    \caption{\textbf{Per-benchmark performance under increasing token pruning.}
    We plot performance as a function of pruning ratio on eight evaluation suites, comparing our method with representative pruning baselines (SparseVLM and PDrop).
    The dashed line denotes the unpruned model performance.
    Our method exhibits the slowest degradation and stays consistently closest to the unpruned upper bound, particularly at aggressive pruning ratios.}
    \label{fig-Visual-pDrop-sparevlm}
\end{figure}

\paragraph{Ablation of Core Components.}
To isolate and quantify the individual contributions of our proposed modules, we conduct a component-wise ablation study on LLaVA-v1.5-7B under a fixed retention budget of $R=192$.
Table~\ref{tab:core_ablation} presents the performance of the baseline pruning model, the model augmented solely with Class-Adaptive Layer Fusion (+Fusion), the model utilizing only the Category-Dependent Split Ratio (+Ratio), and the complete CLASP framework combining both (+Both).

The results demonstrate that each component independently provides robust improvements over the baseline. 
Introducing the class-adaptive layer fusion yields substantial gains (\textit{e.g.}, $+93$ points on MME and $+2.0$ on GQA), highlighting the critical importance of dynamically matching the visual abstraction depth to the user's intent. 
Similarly, adjusting the relevance-coverage pruning budget based on the prompt category (+Ratio) consistently boosts performance across all tasks. 
Ultimately, the integration of both components achieves the highest accuracy across all metrics (\textit{e.g.}, MME reaching 1848 and GQA reaching 60.4), proving that optimal visual representation extraction and adaptive budget allocation are highly complementary.

\begin{table}[ht]
    \centering
    \footnotesize
    \caption{\textbf{Ablation study of core components.} We evaluate the individual and synergistic effects of Class-Adaptive Layer Fusion (+Fusion) and Category-Dependent Split Ratio (+Ratio) on LLaVA-v1.5-7B ($R=192$). Both components contribute significantly, and their combination yields the best overall performance.}
    \label{tab:core_ablation}
    \begin{tabular}{l c c c c}
        \toprule
        \textbf{Config} & \textbf{MME} & \textbf{TextVQA} & \textbf{GQA} & \textbf{POPE} \\
        \midrule
        Baseline       & 1721 & 56.1 & 57.6 & 83.6 \\
        + Fusion       & 1814 & 56.9 & 59.6 & 85.0 \\
        + Ratio        & 1787 & 56.6 & 58.4 & 84.7 \\
        + Both (CLASP) & \textbf{1848} & \textbf{57.6} & \textbf{60.4} & \textbf{85.6} \\
        \bottomrule
    \end{tabular}
\end{table}

% \begin{figure}[t]
%     \centering
%     \includegraphics[width=0.98\linewidth]{figures/datasets_plots.pdf}
%     \caption{Example visualization of the original image and the corresponding segmentation2.}
%     \label{fig-dataset_plot}
% \endback

\paragraph{Sensitivity analysis of pruning ratios across diverse benchmarks.}
Figure~\ref{fig-Visual-pDrop-sparevlm} provides a comprehensive visualization of performance trajectories under varying degrees of pruning ratio (ranging from $0\%$ to $95\%$).
We compare our proposed method against two representative state-of-the-art baselines: SparseVLM and PDrop.
Three key observations can be drawn from these results.
First, our method (orange line) consistently forms the upper envelope of the performance curves across all eight distinct evaluation suites, demonstrating superior token retention capabilities compared to the baselines.
Second, we observe a phenomenon of graceful degradation: while competing methods, particularly PDrop (green), suffer from catastrophic performance collapse in the high-pruning regime (pruning ratio $>0.75$), our approach maintains remarkably robust efficacy.
For instance, on the MME and GQA benchmarks, even when discarding $95\%$ of visual tokens, our method retains a significant portion of the unpruned baseline accuracy, whereas PDrop approaches severely degraded performance.
Third, in challenging reasoning tasks such as MMVet (note that PDrop results are unavailable for this benchmark) and SQA, our method exhibits high resilience, significantly widening the performance gap against SparseVLM as the compression rate increases.
These findings corroborate that our similarity-based selection strategy effectively isolates and preserves semantically critical {visual regions, ensuring model reliability even under extreme computational constraints.

\paragraph{Performance analysis on a larger language model.}
To further evaluate the robustness of our method under extreme token scarcity, we compare CLASP against SparseVLM across four distinct pruning levels (retaining 192, 128, 64, and 32 tokens). 
Table~\ref{tab1:main_table-13B} details these results. 
CLASP consistently outperforms SparseVLM across all compression ratios, demonstrating superior information retention capabilities. 
Specifically, in the high-retention regime (192 tokens), CLASP achieves an average performance retention of 98.2\%, closely matching the upper bound.
At the aggressive pruning level of 64 tokens (88.9\% reduction), CLASP maintains 92.5\% of the original performance, significantly surpassing SparseVLM's 87.3\%.
Even under the extreme constraint of 32 tokens, where the baseline performance drops sharply to 80.3\%, CLASP sustains a robust 87.5\%. 
This trend is particularly evident in fine-grained tasks such as POPE and TextVQA, suggesting that our method effectively preserves critical visual semantics.
\begin{table*}[t]
    \centering
    \footnotesize
    \caption{Performance comparison of various methods on LLaVA-v1.5-13B across different benchmarks. Results are shown for different pruning ratios, with accuracy and average performance highlighted. Best results in \textcolor{MidnightBlue}{\textbf{blue}}.}
    \label{tab1:main_table-13B}
    % 根据需要调整缩放比例，标准全栏表格通常不需要 resizebox 或者只需微调
    \resizebox{0.7\linewidth}{!}{
    \begin{tabular}{l | *{6}{c} | c}
        \toprule
        \textbf{\;Methods} & \textbf{GQA} & \textbf{MMB} & \textbf{MME} & \textbf{POPE} & \textbf{SQA} & \textbf{VQA}$_{\text{Text}}$ & \textbf{Average}\\
        \midrule
        \textcolor{gray}{Upper Bound} & \textcolor{gray}{63.3} & \textcolor{gray}{68.9} & \textcolor{gray}{1818} & \textcolor{gray}{85.9} & \textcolor{gray}{72.8} & \textcolor{gray}{61.8} & \textcolor{gray}{100.0\%} \\
        \midrule
        
        \rowcolor{mygray}
        LLaVA-1.5 \textcolor{gray}{13B} & \multicolumn{7}{c}{\textit{Retain 192 Tokens} \ $\fg{(\downarrow 66.7\%)}$}\\
        SparseVLM & 58.7 & 67.4 & 1768 & 82.2 & \textcolor{MidnightBlue}{\textbf{73.1}} & \textcolor{MidnightBlue}{\textbf{59.6}} & 96.9\% \\
        \rowcolor{mygreen2}
        CLASP (ours) & \textcolor{MidnightBlue}{\textbf{60.1}} & \textcolor{MidnightBlue}{\textbf{68.0}} & \textcolor{MidnightBlue}{\textbf{1785}} & \textcolor{MidnightBlue}{\textbf{86.2}} & 72.8 & \textcolor{MidnightBlue}{\textbf{59.6}} & \textcolor{MidnightBlue}{\textbf{98.2\%}} \\
        
        \midrule
        \rowcolor{mygray}
        LLaVA-1.5 \textcolor{gray}{13B} & \multicolumn{7}{c}{\textit{Retain 128 Tokens} \ $\fg{(\downarrow 77.8\%)}$}\\
        SparseVLM & 57.9 & 65.8 & 1774 & 81.1 & 69.9 & 58.4 & 95.0\% \\
        \rowcolor{mygreen2}
        CLASP (ours) & \textcolor{MidnightBlue}{\textbf{59.6}} & \textcolor{MidnightBlue}{\textbf{68.0}} & \textcolor{MidnightBlue}{\textbf{1789}} & \textcolor{MidnightBlue}{\textbf{85.7}} & \textcolor{MidnightBlue}{\textbf{73.3}} & \textcolor{MidnightBlue}{\textbf{59.0}} & \textcolor{MidnightBlue}{\textbf{98.0\%}} \\
        
        \midrule
        \rowcolor{mygray}
        LLaVA-1.5 \textcolor{gray}{13B} & \multicolumn{7}{c}{\textit{Retain 64 Tokens} \ $\fg{(\downarrow 88.9\%)}$}\\
        SparseVLM & 54.0 & 61.3 & 1641 & 65.0 & 69.0 & 54.6 & 87.3\% \\
        \rowcolor{mygreen2}
        CLASP (ours) & \textcolor{MidnightBlue}{\textbf{56.2}} & \textcolor{MidnightBlue}{\textbf{64.4}} & \textcolor{MidnightBlue}{\textbf{1675}} & \textcolor{MidnightBlue}{\textbf{77.2}} & \textcolor{MidnightBlue}{\textbf{72.8}} & \textcolor{MidnightBlue}{\textbf{55.6}} & \textcolor{MidnightBlue}{\textbf{92.5\%}} \\
        
        \midrule
        \rowcolor{mygray}
        LLaVA-1.5 \textcolor{gray}{13B} & \multicolumn{7}{c}{\textit{Retain 32 Tokens} \ $\fg{(\downarrow 94.4\%)}$}\\
        SparseVLM & 50.2 & 56.2 & 1451 & 55.9 & 67.9 & 50.7 & 80.3\% \\
        \rowcolor{mygreen2}
        CLASP (ours) & \textcolor{MidnightBlue}{\textbf{53.5}} & \textcolor{MidnightBlue}{\textbf{60.7}} & \textcolor{MidnightBlue}{\textbf{1592}} & \textcolor{MidnightBlue}{\textbf{68.7}} & \textcolor{MidnightBlue}{\textbf{71.9}} & \textcolor{MidnightBlue}{\textbf{52.6}} & \textcolor{MidnightBlue}{\textbf{87.5\%}} \\
        \bottomrule
    \end{tabular}
    }
\end{table*}

\paragraph{Performance analysis across varying token budgets.}
Tables~\ref{tab1:main_table1-7B} and~\ref{tab1:main_table1-13B} compare CLASP against SparseVLM on LLaVA-NeXT-7B and LLaVA-NeXT-13B under three budgets (640/320/160).
Across all benchmarks, CLASP consistently surpasses the baseline and exhibits graceful degradation as the budget shrinks.
With high retention (640 tokens), CLASP approaches the unpruned upper bound, reaching 97.0\% (7B) and 97.2\% (13B) average performance, improving over SparseVLM by 2.6 and 2.3 points, respectively.
As sparsity increases, the advantage becomes more pronounced: at 320 tokens, CLASP attains 95.2\% vs.\ 90.0\% on 7B (+5.2) and 95.6\% vs.\ 92.7\% on 13B (+2.9);
under the most aggressive setting (160 tokens), CLASP maintains 92.2\% (7B) and 92.4\% (13B), while SparseVLM drops to 84.5\% and 88.5\%, widening the gap to 7.7 and 3.9 points.
Notably, CLASP yields consistent gains on grounding- and reasoning-sensitive benchmarks (\textit{e.g.}, GQA/MME) under heavy pruning, indicating that it better preserves critical visual evidence for multi-step inference.

\begin{table*}[t]
    \centering
    \footnotesize
    \caption{Performance comparison of various methods on LLaVA-NeXT-7B across different benchmarks. Results are shown for different pruning ratios, with accuracy and average performance highlighted. Best results in \textcolor{MidnightBlue}{\textbf{blue}}.}
    \label{tab1:main_table1-7B}
    % Resize to fit width
    \resizebox{0.7\linewidth}{!}{
    \begin{tabular}{l | *{6}{c} | c}
        \toprule
        \textbf{\;Methods} & \textbf{GQA} & \textbf{MMB} & \textbf{MME} & \textbf{POPE} & \textbf{SQA} & \textbf{VQA}$_{\text{Text}}$ & \textbf{Average}\\
        \midrule
        \textcolor{gray}{Upper Bound} & \textcolor{gray}{64.2} & \textcolor{gray}{67.4} & \textcolor{gray}{1851} & \textcolor{gray}{86.5} & \textcolor{gray}{70.1} & \textcolor{gray}{64.9} & \textcolor{gray}{100.0\%} \\
        \midrule
        
        \rowcolor{mygray}
        LLaVA-NeXT \textcolor{gray}{7B} & \multicolumn{7}{c}{\textit{Retain 640 Tokens}}\\
        SparseVLM & 61.2 & 62.2 & 1697 & 85.3 & 67.6 & 59.7 & 94.4\% \\
        \rowcolor{mygreen2}
        CLASP (ours) & \textcolor{MidnightBlue}{\textbf{63.1}} & \textcolor{MidnightBlue}{\textbf{62.8}} & \textcolor{MidnightBlue}{\textbf{1746}} & \textcolor{MidnightBlue}{\textbf{87.9}} & \textcolor{MidnightBlue}{\textbf{69.8}} & \textcolor{MidnightBlue}{\textbf{61.8}} & \textcolor{MidnightBlue}{\textbf{97.0\%}} \\
        
        \midrule
        \rowcolor{mygray}
        LLaVA-NeXT \textcolor{gray}{7B} & \multicolumn{7}{c}{\textit{Retain 320 Tokens}}\\
        SparseVLM & 56.1 & 60.6 & 1533 & 82.4 & 66.1 & 58.4 & 90.0\% \\
        \rowcolor{mygreen2}
        CLASP (ours) & \textcolor{MidnightBlue}{\textbf{62.7}} & \textcolor{MidnightBlue}{\textbf{61.2}} & \textcolor{MidnightBlue}{\textbf{1723}} & \textcolor{MidnightBlue}{\textbf{85.8}} & \textcolor{MidnightBlue}{\textbf{67.0}} & \textcolor{MidnightBlue}{\textbf{61.7}} & \textcolor{MidnightBlue}{\textbf{95.2\%}} \\
        
        \midrule
        \rowcolor{mygray}
        LLaVA-NeXT \textcolor{gray}{7B} & \multicolumn{7}{c}{\textit{Retain 160 Tokens}}\\
        SparseVLM & 55.8 & 56.9 & 1420 & 78.2 & 57.8 & 55.9 & 84.5\% \\
        \rowcolor{mygreen2}
        CLASP (ours) & \textcolor{MidnightBlue}{\textbf{62.1}} & \textcolor{MidnightBlue}{\textbf{59.8}} & \textcolor{MidnightBlue}{\textbf{1699}} & \textcolor{MidnightBlue}{\textbf{84.5}} & \textcolor{MidnightBlue}{\textbf{60.9}} & \textcolor{MidnightBlue}{\textbf{59.3}} & \textcolor{MidnightBlue}{\textbf{92.2\%}} \\
        \bottomrule
    \end{tabular}
    }
\end{table*}

\begin{table*}[t]
    \centering
    \footnotesize
    \caption{Performance comparison of various methods on LLaVA-NeXT-13B across different benchmarks. Results are shown for different pruning ratios, with accuracy and average performance highlighted. Best results in \textcolor{MidnightBlue}{\textbf{blue}}.}
    \label{tab1:main_table1-13B}
    % Resize to fit width
    \resizebox{0.7\linewidth}{!}{
    \begin{tabular}{l | *{6}{c} | c}
        \toprule
        \textbf{\;Methods} & \textbf{GQA} & \textbf{MMB} & \textbf{MME} & \textbf{POPE} & \textbf{SQA} & \textbf{VQA}$_{\text{Text}}$ & \textbf{Average}\\
        \midrule
        \textcolor{gray}{Upper Bound} & \textcolor{gray}{65.4} & \textcolor{gray}{70.0} & \textcolor{gray}{1901} & \textcolor{gray}{86.2} & \textcolor{gray}{73.5} & \textcolor{gray}{64.3} & \textcolor{gray}{100.0\%} \\
        \midrule
        
        \rowcolor{mygray}
        LLaVA-NeXT \textcolor{gray}{13B} & \multicolumn{7}{c}{\textit{Retain 640 Tokens}}\\
        SparseVLM & 62.7 & 62.0 & 1821 & 86.0 & 71.5 & 59.3 & 94.9\% \\
        \rowcolor{mygreen2}
        CLASP (ours) & \textcolor{MidnightBlue}{\textbf{64.2}} & \textcolor{MidnightBlue}{\textbf{63.2}} & \textcolor{MidnightBlue}{\textbf{1867}} & \textcolor{MidnightBlue}{\textbf{86.8}} & \textcolor{MidnightBlue}{\textbf{73.5}} & \textcolor{MidnightBlue}{\textbf{61.5}} & \textcolor{MidnightBlue}{\textbf{97.2\%}} \\
        
        \midrule
        \rowcolor{mygray}
        LLaVA-NeXT \textcolor{gray}{13B} & \multicolumn{7}{c}{\textit{Retain 320 Tokens}}\\
        SparseVLM & 60.9 & 60.8 & 1798 & 83.5 & 70.6 & 57.2 & 92.7\% \\
        \rowcolor{mygreen2}
        CLASP (ours) & \textcolor{MidnightBlue}{\textbf{63.0}} & \textcolor{MidnightBlue}{\textbf{62.8}} & \textcolor{MidnightBlue}{\textbf{1848}} & \textcolor{MidnightBlue}{\textbf{85.2}} & \textcolor{MidnightBlue}{\textbf{71.1}} & \textcolor{MidnightBlue}{\textbf{60.8}} & \textcolor{MidnightBlue}{\textbf{95.6\%}} \\
        
        \midrule
        \rowcolor{mygray}
        LLaVA-NeXT \textcolor{gray}{13B} & \multicolumn{7}{c}{\textit{Retain 160 Tokens}}\\
        SparseVLM & 59.4 & 56.5 & 1755 & 81.1 & 66.2 & 53.3 & 88.5\% \\
        \rowcolor{mygreen2}
        CLASP (ours) & \textcolor{MidnightBlue}{\textbf{61.8}} & \textcolor{MidnightBlue}{\textbf{59.8}} & \textcolor{MidnightBlue}{\textbf{1801}} & \textcolor{MidnightBlue}{\textbf{83.7}} & \textcolor{MidnightBlue}{\textbf{68.1}} & \textcolor{MidnightBlue}{\textbf{57.8}} & \textcolor{MidnightBlue}{\textbf{92.4\%}} \\
        \bottomrule
    \end{tabular}
    }
\end{table*}

\textbf{Performance analysis on InternVL2-26B.} 
To further validate the scalability and robustness of our class-adaptive pruning framework on more recent and massive multimodal foundation models, we extended our evaluation to InternVL2-26B~\cite{chen2024internvl}. As shown in Table~\ref{tab:internvl_results}, we compared CLASP against existing token reduction baselines (FastV and ToMe) under a strict retention budget of $R=35\%$. Our method demonstrates exceptional information preservation capabilities at this massive scale. For example, on TextVQA and MMVet, CLASP retains high accuracy ($81.7$ and $63.2$), substantially outperforming both FastV ($75.6$ and $45.0$) and ToMe ($75.7$ and $52.5$). Notably, on GQA, CLASP ($65.0$) even marginally exceeds the unpruned upper bound ($64.9$), suggesting that our redundancy-aware pruning effectively acts as a noise filter, benefiting complex compositional reasoning. These supplementary results confirm that CLASP remains highly effective and generalizable, regardless of the underlying model's parameter size.

\begin{table*}[t]
    \centering
    \footnotesize
    \caption{Performance comparison on InternVL2-26B across widely-used benchmarks. Results are shown for an aggressive pruning ratio (retaining $35\%$ of visual tokens), with accuracy highlighted. Best results in \textcolor{MidnightBlue}{\textbf{blue}}.}
    \label{tab:internvl_results}
    % Resize to fit width (adjusted slightly for fewer columns)
    \resizebox{0.6\linewidth}{!}{
    \begin{tabular}{l | *{4}{c}}
        \toprule
        \textbf{\;Methods} & \textbf{TextVQA} & \textbf{MME} & \textbf{GQA} & \textbf{MMVet} \\
        \midrule
        \textcolor{gray}{Upper Bound} & \textcolor{gray}{82.5} & \textcolor{gray}{2270} & \textcolor{gray}{64.9} & \textcolor{gray}{64.0} \\
        \midrule
        
        \rowcolor{mygray}
        InternVL2 \textcolor{gray}{26B} & \multicolumn{4}{c}{\textit{Retain 35\% Tokens}}\\
        FastV & 75.6 & 2140 & 61.2 & 45.0 \\
        ToMe  & 75.7 & 2178 & 63.6 & 52.5 \\
        \rowcolor{mygreen2}
        CLASP (ours) & \textcolor{MidnightBlue}{\textbf{81.7}} & \textcolor{MidnightBlue}{\textbf{2262}} & \textcolor{MidnightBlue}{\textbf{65.0}} & \textcolor{MidnightBlue}{\textbf{63.2}} \\
        \bottomrule
    \end{tabular}
    }
\end{table*}

\newcommand{\best}[1]{\textcolor{MidnightBlue}{\textbf{#1}}}

\begin{table*}[t]
    \centering
    \footnotesize
    \caption{
        \textbf{Ablation on layer-mixture strategies across different task categories with $R=192$.}
        The columns correspond to specific task types: 
        (0)~Object Identification, (1)~Attribute/Breed ID, (2)~Text/Symbol Recognition, 
        (3)~Scene Understanding, (4)~Spatial Relations, (5)~Counting, 
        (6)~Action/Interaction, (7)~Intention/Function, and (8)~Default.
        \textbf{Layer Mixture Legend:} 
        A:~$0.2 L_2 + 0.3 L_{6} + 0.5 L_{11}$; 
        B:~$0.2 L_5 + 0.3 L_{15} + 0.5 L_{22}$; 
        C:~$0.2 L_{12} + 0.3 L_{15} + 0.5 L_{19}$; 
        D:~$0.2 L_5 + 0.8 L_{22}$; 
        E:~$0.2 L_{20} + 0.8 L_{22}$.
        Best results are highlighted in \textcolor{MidnightBlue}{\textbf{blue}}, and the row with the highest average score in each dataset is shaded.
    }
    \label{tab:ablate_layer_mix_qtype_vbars}
    \resizebox{\linewidth}{!}{
    \begin{tabular}{l|ccccccccc|c}
        \toprule
        \textbf{Layers} &
        \textbf{Class 0} & \textbf{Class 1} & \textbf{Class 2} & \textbf{Class 3} & \textbf{Class 4} &
        \textbf{Class 5} & \textbf{Class 6} & \textbf{Class 7} & \textbf{Class 8} &
        \textbf{Avg} \\
        \midrule

        % ---------------- MMVet ----------------
        \rowcolor{mygray}
        \multicolumn{1}{c|}{} & \multicolumn{9}{c|}{\textit{MMVet ($R=192$)}} & \multicolumn{1}{c}{} \\
        A & 28.0 & 28.2 & 28.6 & 28.5 & 27.8 & 28.4 & 28.3 & \best{28.7} & 28.3 & 28.3 \\
        B & \best{28.9} & 28.3 & 27.1 & \best{28.8} & 16.2 & 27.8 & 28.2 & 28.1 & 28.2 & 26.8 \\
        C & 28.4 & \best{28.6} & 25.8 & 28.5 & 14.0 & 28.1 & \best{28.9} & 28.4 & \best{28.4} & 26.6 \\
        \rowcolor{mygreen2}
        D & 28.7 & 28.0 & \best{29.1} & 28.3 & 28.3 & \best{29.5} & 28.8 & 28.5 & 28.2 & \best{28.6} \\
        E & 28.2 & 28.3 & 27.9 & 28.4 & \best{28.6} & 27.9 & 28.1 & 28.0 & 28.2 & 28.2 \\
        \midrule

        % ---------------- TextVQA ----------------
        \rowcolor{mygray}
        \multicolumn{1}{c|}{} & \multicolumn{9}{c|}{\textit{TextVQA ($R=192$)}} & \multicolumn{1}{c}{} \\
        A & 57.56 & 57.56 & 55.94 & 57.51 & \best{57.60} & 57.19 & 57.55 & 57.58 & 57.51 & 57.33 \\
        B & \best{57.62} & 57.58 & 45.43 & 57.54 & 57.57 & 55.57 & 57.57 & 57.55 & \best{57.57} & 56.00 \\
        C & 57.61 & 57.59 & 12.03 & 57.56 & 57.55 & 52.64 & 57.53 & 57.49 & 57.53 & 51.95 \\
        D & 57.58 & 57.58 & 57.44 & \best{57.60} & 57.57 & 57.53 & 57.56 & \best{57.59} & 57.55 & 57.56 \\
        \rowcolor{mygreen2}
        E & 57.59 & \best{57.62} & \best{57.74} & 57.54 & 57.54 & \best{57.54} & \best{57.60} & \best{57.59} & 57.34 & \best{57.57} \\
        \midrule

        % ---------------- SQA ----------------
        \rowcolor{mygray}
        \multicolumn{1}{c|}{} & \multicolumn{9}{c|}{\textit{SQA ($R=192$)}} & \multicolumn{1}{c}{} \\
        A & 68.62 & 68.47 & 63.41 & \best{68.52} & \best{68.72} & \best{68.77} & \best{68.57} & 68.57 & \best{68.67} & 68.04 \\
        B & \best{68.72} & \best{68.67} & 39.43 & 68.27 & 68.57 & 68.17 & \best{68.57} & 68.57 & \best{68.67} & 65.29 \\
        C & 68.52 & 68.62 & 14.83 & 68.17 & 68.67 & 64.35 & 68.52 & 68.57 & 68.57 & 62.09 \\
        D & 68.47 & 68.52 & 68.27 & \best{68.52} & 68.57 & 68.62 & \best{68.57} & 68.62 & 68.52 & 68.52 \\
        \rowcolor{mygreen2}
        E & 68.67 & 68.62 & \best{68.42} & 68.47 & 68.62 & 68.57 & 68.52 & \best{68.67} & \best{68.67} & \best{68.58} \\
        \midrule

        % ---------------- MME ----------------
        \rowcolor{mygray}
        \multicolumn{1}{c|}{} & \multicolumn{9}{c|}{\textit{MME ($R=192$)}} & \multicolumn{1}{c}{} \\
        A & \best{1793} & 1790 & 1702 & 1774 & 1778 & 1785 & \best{1788} & 1780 & 1778 & 1774.2 \\
        B & 1780 & 1783 & 731  & 1787 & \best{1788} & 1777 & 1783 & 1781 & 1786 & 1666.2 \\
        C & 1786 & 1791 & 1433 & 1776 & 1785 & 1784 & \best{1788} & \best{1788} & 1777 & 1745.3 \\
        \rowcolor{mygreen2}
        D & 1788 & \best{1792} & \best{1806} & 1785 & 1784 & 1779 & 1785 & 1784 & 1785 & \best{1787.6} \\
        E & 1782 & 1786 & 1790 & \best{1793} & 1783 & \best{1788} & 1778 & \best{1788} & \best{1794} & 1786.9 \\
        \bottomrule
    \end{tabular}}
\end{table*}

\paragraph{Supplementary Evaluations on Diverse Open Datasets.}
To further validate the generalization and robustness of our approach beyond standard academic benchmarks, we conduct supplementary evaluations on a diverse set of open datasets. These include open-ended conversational evaluation (LLaVA-Bench, MMVet), real-world noisy visual perception (VizWiz), cross-lingual multimodal understanding (MMBench-Chinese), and fine-grained spatial reasoning (SEED-Bench). 

As summarized in Table \ref{tab:supp_open_datasets}, we compare CLASP against a strong token reduction baseline, SparseVLM, across different token retention budgets ($R=192, 128, 64$). While both methods perform comparably well at a higher budget ($R=192$), CLASP demonstrates significantly superior robustness under aggressive pruning conditions. For instance, at an extreme sparsity level of $R=64$, CLASP outperforms SparseVLM by a massive margin of $+6.5\%$ on LLaVA-Bench and $+9.5\%$ on MMB-CN. This substantial gap underscores that our class-adaptive layer fusion and dual-stage pruning mechanisms effectively preserve the critical visual tokens necessary for complex, open-ended generation and cross-lingual semantic alignment, whereas static pruning methods suffer severe degradation.

\begin{table*}[h]
    \centering
    \footnotesize
    \caption{Supplementary evaluations on diverse open datasets using LLaVA-v1.5-7B under varying token budgets. We compare our CLASP framework against SparseVLM. The best results between the two pruning methods at each budget are highlighted in bold.}
    \label{tab:supp_open_datasets}
    \resizebox{0.85\linewidth}{!}{%
    \begin{tabular}{c l c c c c c c}
        \toprule
        \textbf{Budget} & \textbf{Method} & \textbf{LLaVA-B} & \textbf{MMVet} & \textbf{VizWiz} & \textbf{MMB-CN} & \textbf{MMB} & \textbf{SEED} \\
        \midrule
        Unpruned & LLaVA-v1.5-7B & 66.8 & 30.9 & 50.0 & 58.1 & 64.7 & 66.2 \\
        \midrule
        \multirow{2}{*}{R=192} & SparseVLM & 66.1 & 33.1 & 50.5 & 53.7 & \textbf{62.5} & 64.2 \\
         & CLASP (Ours) & \textbf{66.7} & \textbf{33.3} & \textbf{52.1} & \textbf{57.9} & 61.3 & \textbf{65.4} \\
        \midrule
        \multirow{2}{*}{R=128} & SparseVLM & 62.7 & 29.0 & 51.4 & 51.1 & 60.0 & \textbf{63.6} \\
         & CLASP (Ours) & \textbf{65.6} & \textbf{30.0} & \textbf{51.9} & \textbf{57.1} & \textbf{60.7} & 63.0 \\
        \midrule
        \multirow{2}{*}{R=64} & SparseVLM & 57.5 & 24.9 & 50.1 & 46.1 & 56.2 & 56.8 \\
         & CLASP (Ours) & \textbf{64.0} & \textbf{26.2} & \textbf{51.6} & \textbf{55.6} & \textbf{59.1} & \textbf{58.5} \\
        \bottomrule
    \end{tabular}%
    }
\end{table*}

\paragraph{Ablation study of layer-mixture strategies across different task categories.} 
Table~\ref{tab:ablate_layer_mix_qtype_vbars} and Figure~\ref{fig-mme_heatmap-motiavtion} present the performance variations across nine distinct task categories under five representative layer-fusion strategies with a fixed token budget ($R=192$). These strategies (A--E) vary the source of visual features, ranging from shallow-biased mixtures to deep-layer integration. As visualized in the heatmaps and detailed in the table, while deeper layers generally provide superior semantic abstraction, specific tasks exhibit distinct sensitivities to feature depth. For instance, in the MMVet benchmark, Class 4 (Spatial Relations) experiences a severe performance drop with the shallow-focused Strategy B (16.2) compared to the deep-focused Strategy E (28.6), highlighting the necessity of high-level features for spatial reasoning. Similarly, in TextVQA, Class 2 (Text/Symbol Recognition) performs poorly under Strategy C (12.03), whereas Strategy E achieves optimal accuracy (57.74). Notably, the hybrid Strategy D, which fuses early visual cues ($L_5$) with deep semantics ($L_{22}$), consistently achieves the highest average scores across benchmarks like MMVet (28.6) and MME (1787.6), demonstrating that a balanced integration of low-level detail and high-level semantics offers the most robust generalization.

\paragraph{Ablation study of similarity and attention for different classes.} 
Table~\ref{tab:cat_results_r192_colored} and Figure~\ref{fig:ratio_heatmap-motivation} present the performance variations across nine distinct task categories under varying attention selection ratios $p$, where $p$ denotes the proportion of retained tokens based on attention scores (ranging from $0.1$ for high sparsity to $1.0$ for full attention). 
Our results indicate that the optimal attention ratio is highly task-dependent.
First, we observe that the model maintains competitive performance even at lower ratios (\textit{e.g.}, $p=0.3$ or $p=0.5$) across benchmarks like MMVet and SQA, demonstrating the redundancy in standard visual tokens. 
More notably, reducing $p$ can lead to performance gains over the baseline ($p=1.0$). 
For instance, in the MME dataset, Class 3 (Scene Understanding) achieves a peak score of 1828 at $p=0.7$, surpassing the full-attention score of 1814. 
Similarly, TextVQA shows improved accuracy in Class 5 (Counting) at $p=0.3$.
These results demonstrate that the optimal attention ratio is highly category-dependent, with specific tasks achieving peak performance at intermediate $p$ values rather than the full-attention baseline.

\begin{table*}[t]
    \centering
    \footnotesize
    \caption{
        \textbf{Ablation study on the sensitivity of different task categories to the attention selection ratio $p$ (with fixed $R=192$).}
        The columns correspond to specific task types mapped as follows: 
        (0)~Object Identification, (1)~Attribute/Breed ID, (2)~Text/Symbol Recognition, 
        (3)~Scene Understanding, (4)~Spatial Relations, (5)~Counting, 
        (6)~Action/Interaction, (7)~Intention/Function, and (8)~Default.
        Best results are highlighted in \textcolor{MidnightBlue}{\textbf{blue}}, and the row with the robust average score in each dataset is shaded.
    }
    \label{tab:cat_results_r192_colored}
    \resizebox{0.9\linewidth}{!}{%
    \begin{tabular}{l | *{9}{c}}
        \toprule
        \textbf{$p$} 
        & \textbf{Class 0} & \textbf{Class 1} & \textbf{Class 2} & \textbf{Class 3} & \textbf{Class 4} 
        & \textbf{Class 5} & \textbf{Class 6} & \textbf{Class 7} & \textbf{Class 8} \\
        \midrule

        % ---------------- MME ----------------
        \rowcolor{mygray}
        \multicolumn{1}{c|}{} & \multicolumn{9}{c}{\textit{MME ($R=192$)}}\\
        0.1 & 1815 & 1810 & \best{1815} & 1812 & 1815 & 1782 & 1812 & 1812 & 1813 \\
        0.3 & 1809 & \best{1814} & 1770 & 1814 & \best{1818} & 1802 & 1805 & 1809 & 1804 \\
        0.5 & 1799 & 1813 & 1791 & 1812 & 1811 & 1801 & 1813 & \best{1815} & 1811 \\
        \rowcolor{mygreen2}
        0.7 & \best{1817} & 1812 & 1813 & \best{1828} & \best{1818} & 1812 & 1812 & 1814 & 1813 \\
        0.9 & 1816 & 1813 & 1809 & 1816 & 1817 & \best{1816} & 1813 & 1814 & \best{1814} \\
        1.0 & 1814 & \best{1814} & 1814 & 1814 & 1814 & 1814 & \best{1814} & 1814 & \best{1814} \\
        \midrule

        % ---------------- MMVet ----------------
        \rowcolor{mygray}
        \multicolumn{1}{c|}{} & \multicolumn{9}{c}{\textit{MMVet ($R=192$)}}\\
        0.1 & 29.6 & 30.2 & 30.4 & 29.6 & 30.0 & 29.9 & 29.0 & 29.8 & 29.2 \\
        0.3 & 29.4 & 31.0 & \best{30.6} & 29.7 & 30.1 & 30.3 & 29.0 & 30.1 & 30.0 \\
        0.5 & 29.9 & 29.7 & 30.1 & \best{31.0} & 29.4 & 30.0 & 28.8 & 30.1 & 30.4 \\
        0.7 & \best{30.2} & 30.4 & 29.5 & 29.7 & 29.2 & \best{30.5} & 28.6 & 30.2 & 30.5 \\
        \rowcolor{mygreen2}
        0.9 & 29.3 & \best{31.2} & 29.5 & 29.9 & \best{30.4} & 29.6 & \best{30.2} & \best{30.5} & \best{30.7} \\
        1.0 & 29.8 & 29.8 & 29.8 & 29.8 & 29.8 & 29.8 & 29.8 & 29.8 & 29.8 \\
        \midrule

        % ---------------- TextVQA ----------------
        \rowcolor{mygray}
        \multicolumn{1}{c|}{} & \multicolumn{9}{c}{\textit{TextVQA ($R=192$)}}\\
        0.1 & 56.88 & 56.48 & 55.89 & 56.86 & 56.88 & 56.97 & 56.97 & 56.98 & 56.92 \\
        0.3 & 56.96 & 56.66 & 56.24 & 56.90 & 56.90 & \best{57.48} & \best{56.99} & \best{57.02} & \best{57.10} \\
        0.5 & 56.96 & 56.90 & 56.50 & 56.93 & 56.93 & 57.01 & 56.94 & 56.95 & 56.98 \\
        0.7 & 56.96 & 57.07 & 56.85 & 56.93 & 56.98 & 57.07 & 56.92 & 56.99 & 56.90 \\
        0.9 & 56.96 & \best{57.09} & \best{57.00} & 56.94 & 56.96 & 57.00 & 56.97 & \best{57.02} & 56.91 \\
        \rowcolor{mygreen2}
        1.0 & \best{56.99} & 56.99 & 56.99 & \best{56.99} & \best{56.99} & 56.99 & \best{56.99} & 56.99 & 56.99 \\
        \midrule

        % ---------------- SQA ----------------
        \rowcolor{mygray}
        \multicolumn{1}{c|}{} & \multicolumn{9}{c}{\textit{SQA ($R=192$)}}\\
        0.1 & 68.37 & 68.27 & \best{68.37} & \best{69.56} & 68.02 & \best{68.32} & 68.27 & 68.32 & 65.26 \\
        0.3 & 68.28 & 68.12 & 68.17 & 68.82 & 68.22 & \best{68.32} & 68.27 & \best{68.42} & 64.95 \\
        0.5 & 68.32 & 68.27 & \best{68.37} & 68.47 & 68.27 & \best{68.32} & 68.22 & 65.26 & 64.35 \\
        \rowcolor{mygreen2}
        0.7 & \best{68.42} & \best{68.32} & \best{68.37} & 68.22 & 68.37 & \best{68.32} & 68.27 & 65.32 & 65.65 \\
        0.9 & 68.17 & 68.22 & 68.27 & 68.32 & \best{68.67} & 68.22 & \best{68.32} & 65.21 & 66.41 \\
        1.0 & 68.27 & 68.27 & 68.27 & 68.27 & 68.27 & 68.27 & 68.27 & 68.27 & \best{68.27} \\
        \bottomrule
    \end{tabular}%
    }
\end{table*}

\begin{figure}[t]
    \centering
    \includegraphics[width=0.9\linewidth]{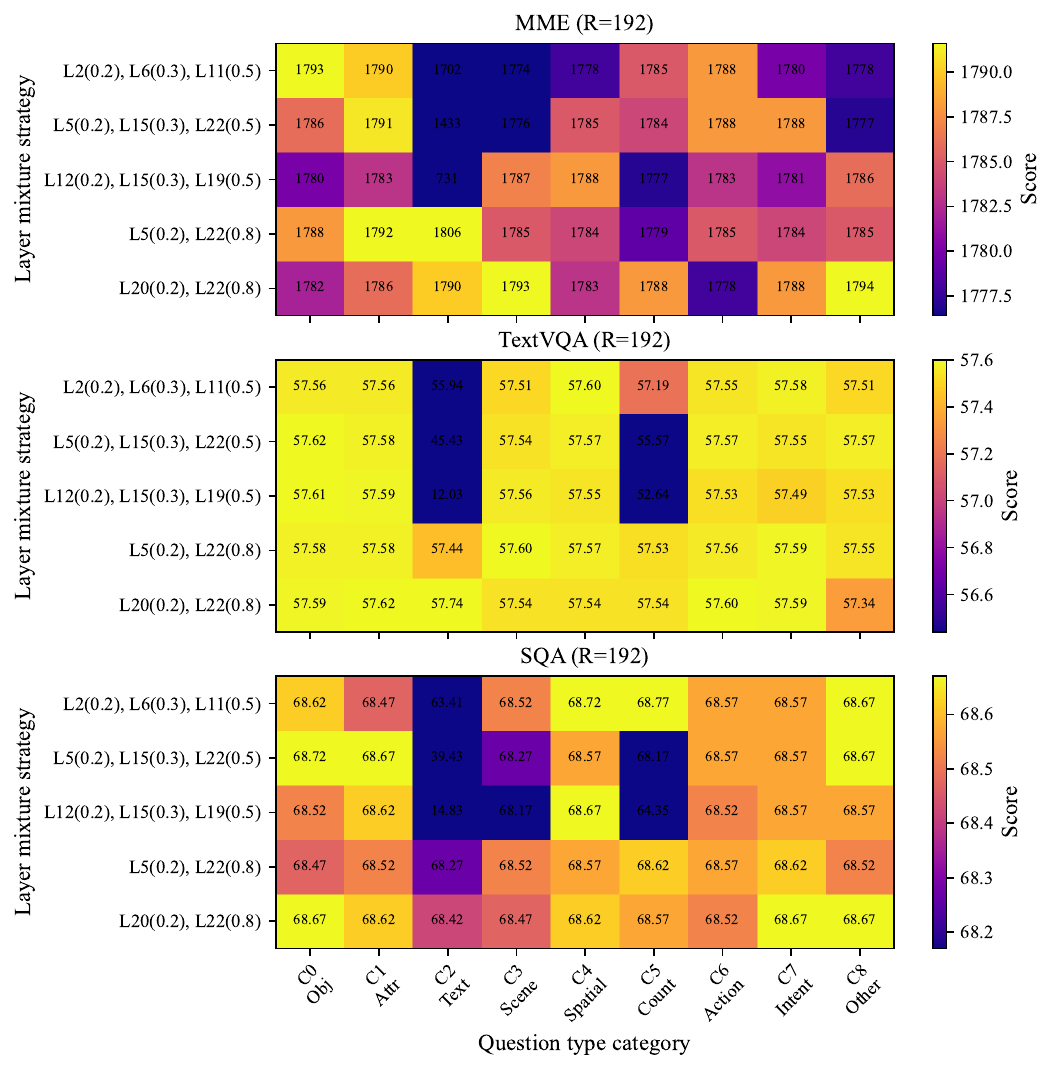}
    \caption{Layer mixture ablation under a fixed token budget ($R=192$) on MME, TextVQA and SQA.
    Rows are layer mixture strategies with weights in parentheses.
    Columns are question types: C0 object identification, C1 attribute or breed identification, C2 text or symbol recognition, C3 scene understanding, C4 spatial relations, C5 counting, C6 action or interaction, C7 intention or function, C8 default.}
    \label{fig-mme_heatmap-motiavtion}
\end{figure}

\begin{figure}[t]
    \centering
    \includegraphics[width=0.9\linewidth]{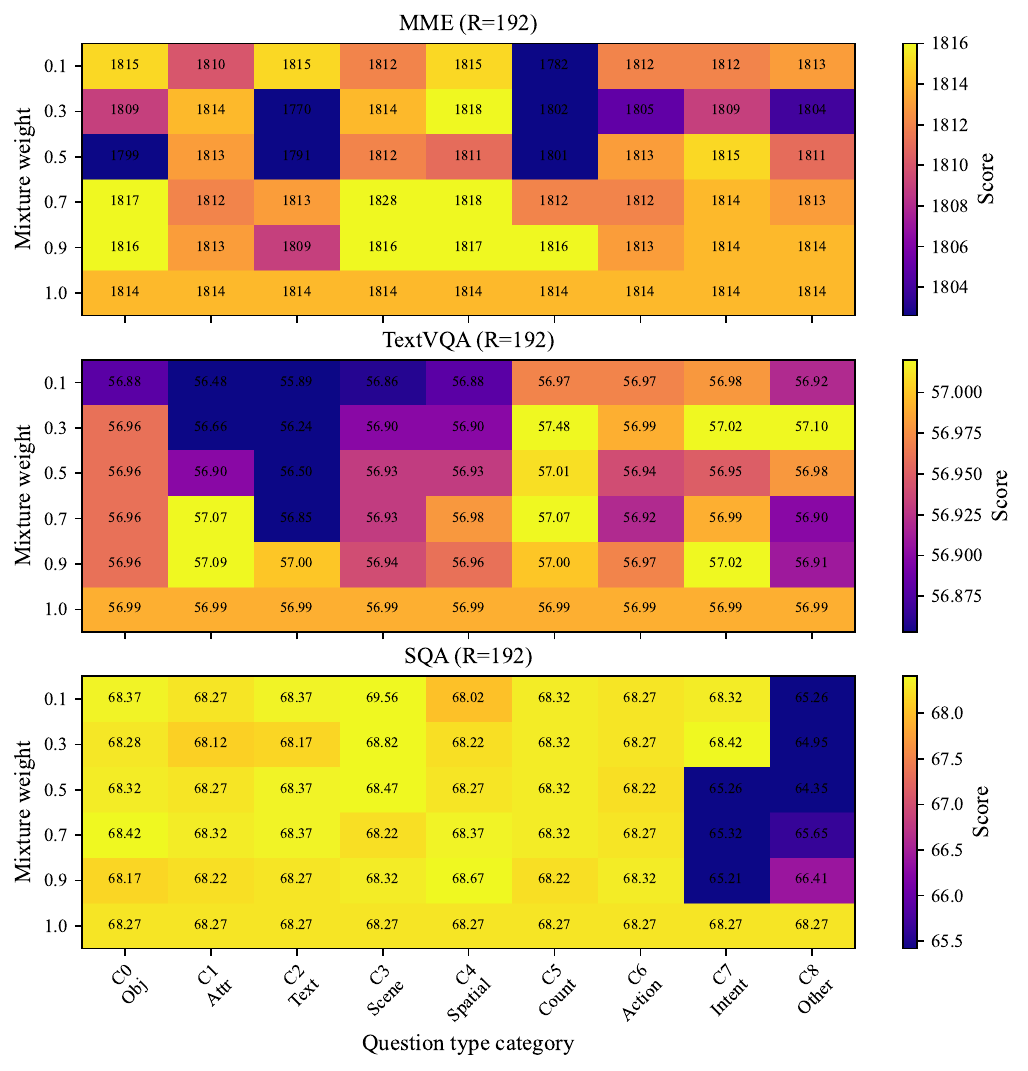}
    \vspace{-0.25em}
    \caption{Ablation on attention and similarity mixture weight under a fixed token budget ($R=192$) on MME, TextVQA and SQA.
    Rows are the mixture ratios. Columns are question types: C0 object identification, C1 attribute or breed identification, C2 text or symbol recognition, C3 scene understanding, C4 spatial relations, C5 counting, C6 action or interaction, C7 intention or function, C8 default.
    }
    \label{fig:ratio_heatmap-motivation}
\end{figure}

\begin{table*}[h]
    \centering
    \footnotesize
    \caption{Impact of the number of clustering iterations on performance ($R=192$).}
    \label{tab:ablate_cluster_iter}
    % 修改此处：将 resizebox 的参数改为 0.7\linewidth
    \resizebox{0.65\linewidth}{!}{%
    \begin{tabular}{c c c c c c c}
        \toprule
        \textbf{Iterations} & \textbf{MME} & \textbf{GQA} & \textbf{MMVet} & \textbf{POPE} & \textbf{SQA} & \textbf{TextVQA} \\
        \midrule
        3  & 1828 & 60.41 & 31.1 & 85.57 & 69.58 & 56.83 \\
        \underline{5}  & \underline{1848} & \underline{60.44} & \underline{31.8} & \underline{85.55} & \underline{69.56} & \underline{57.59} \\
        7  & 1823 & 59.08 & 31.3 & 85.48 & 68.32 & 56.81 \\
        9  & 1762 & 59.02 & 30.5 & 84.33 & 68.37 & 56.92 \\
        11 & 1757 & 58.61 & 30.1 & 84.41 & 68.22 & 56.94 \\
        \bottomrule
    \end{tabular}%
    }
\end{table*}

\paragraph{Clustering iterations.}
Under a fixed retention budget ($R=192$), we ablate the number of clustering refinement steps used in the similarity-driven stage of our pruning (Stage~II), where token affinity is computed by cosine similarity on $\ell_2$-normalized aligned features, \textit{i.e.}, $\mathrm{sim}(t,t')=\mathbf{u}_{n,t}^\top \mathbf{u}_{n,t'}$ (Eq.~\eqref{eq:redundancy-inst}). 
As shown in Table~\ref{tab:ablate_cluster_iter}, $5$ iterations achieves the best overall trade-off, yielding the highest MME score (1848) and the best/near-best performance on GQA (60.44), MMVet (31.8), and TextVQA (57.59), while leaving POPE/SQA essentially unchanged. 
With fewer iterations (\textit{e.g.}, $3$), clustering is under-refined and redundant tokens are insufficiently consolidated.
With more iterations ($\geq 7$), performance drops consistently (\textit{e.g.}, MME decreases to 1823/1762/1757), indicating over-merging or over-smoothing that can erase fine-grained evidence. 
We therefore adopt $5$ clustering iterations as the default setting.

\begin{table*}[t]
    \centering
    \footnotesize
    \caption{\textbf{Effect of initialization under a fixed budget ($R{=}192$).}
    We compare our \textbf{search-derived weight configuration} with three initialization schemes while keeping the training protocol and budget fixed.
    Evaluation covers three general multimodal benchmarks (MMBench, MMVet, MME), four VQA/QA benchmarks (TextVQA, VQA-v2, GQA, SQA), and the hallucination benchmark POPE.
    Higher is better for all metrics. Best and second-best results are \textbf{bold} and \underline{underlined}, respectively.}
    \label{tab:init_r192}
    \resizebox{0.98\linewidth}{!}{%
    \begin{tabular}{l ccc cccc c}
        \toprule
        \textbf{Methods}
        & \textbf{MMBench} & \textbf{MMVet} & \textbf{MME}
        & \textbf{TextVQA} & \textbf{VQA-v2} & \textbf{GQA} & \textbf{SQA}
        & \textbf{POPE} \\
        \midrule
        Search-derived Weights (fixed $W$)
            & \textbf{61.34} & \textbf{33.3} & \textbf{1848}
            & \textbf{57.59} & \underline{77.08} & \textbf{60.44} & \textbf{69.56}
            & 85.55 \\
        Search-derived Weights \textit{init} + Calibrated $W$
            & \underline{60.13} & 32.2 & \underline{1759}
            & \underline{56.15} & \underline{77.08} & \underline{60.22} & \underline{69.39}
            & \underline{85.58} \\
        Penultimate-biased \textit{init} + Calibrated $W$
            & 57.7 & \underline{33.2} & 1722
            & 56.05 & \textbf{77.14} & 59.78 & 68.69
            & \textbf{85.67} \\
        Uniform \textit{init} + Calibrated $W$
            & 17.70 & 12.8 & 1127
            & 30.28 & 62.43 & 43.13 & 47.63
            & 78.03 \\
        \bottomrule
    \end{tabular}%
    }
\end{table*}

\paragraph{Qualitative Analysis.}
Figure~\ref{fig:appendix_case_1_4} and Figure~\ref{fig:appendix_case_5_8} present a qualitative comparison of visual token pruning.
We visualize the results across different network layers and varied pruning ratios ($R$).
In these visualizations, the red bounding boxes denote the ground-truth target regions.
The retained tokens are visualized to illustrate the rationale behind their selection.
Specifically, blue points indicate regions preserved primarily due to their attention significance.
Meanwhile, red points represent regions selected based on similarity metrics.
As clearly observed, our method demonstrates superior capability in semantic preservation compared to state-of-the-art baselines such as PDrop and SparseVLM.
Methods like PDrop and SparseVLM often struggle to maintain the semantic structure.
They tend to lose critical foreground information or retain excessive background noise.
This limitation is particularly evident at a high pruning ratio of $R=88.9\%$.
In contrast, our approach consistently aligns the retained tokens with the target object within the red boxes.
This robust alignment is maintained across all depths, including Layers 2, 6, and 15.
These results indicate that our method effectively filters redundancy.
It successfully maintains a focused representation of the region of interest throughout the inference process.

\paragraph{Search-derived Weights vs.\ Further Calibration.}
To assess the robustness of the learned parameters, we compare our \emph{Search-derived Weights} (fixed $W$) against a variant that starts from the same \emph{Search-derived Weights init} but further calibrates $W$ on the held-out calibration set (\emph{Search-derived Weights init + Calibrated $W$}), under a fixed token budget ($R=192$).
Table~\ref{tab:init_r192} shows that the fixed \emph{Search-derived Weights} yield the strongest overall performance: it improves MMBench (61.34 vs.\ 60.13), MMVet (33.3 vs.\ 32.2), MME (1848 vs.\ 1759), and TextVQA (57.59 vs.\ 56.15), and provides consistent gains on GQA/SQA, while matching VQA-v2 and remaining comparable on POPE.
Moreover, naive alternatives are clearly suboptimal: a penultimate-biased initialization underperforms, and uniform averaging across layers severely degrades all benchmarks, highlighting the effectiveness of our discrete search strategy.

We hypothesize that this gap reflects a mismatch between \emph{continuous optimization behavior} and \emph{what token pruning needs to preserve}.
Vision encoders exhibit a strong depth hierarchy: shallow layers retain high-frequency local evidence (\textit{e.g.}, edges, strokes, and fine layouts) that is crucial for OCR- and counting-heavy queries, whereas deep layers emphasize invariant semantics.
Our \emph{Search-derived Weights} effectively capture this inductive bias by finding an optimal allocation of mixture mass to shallow/mid layers, ensuring that fine-grained evidence is not washed out.
In contrast, further continuous calibration of $W$ on limited data can exhibit shortcut behavior: optimization is naturally attracted to deep-layer features that yield strong, easy-to-fit semantic signals, resulting in a deep-biased mixture that is locally optimal for coarse semantics but dilutes the precise visual cues required by detail-sensitive tasks (\textit{e.g.}, TextVQA).
This explanation is consistent with common ablation patterns reported in recent VLM pruning studies, where shallow/early cues often matter disproportionately for fine-grained perception, and overly aggressive averaging or smoothing across layers harms OCR-centric performance.

\paragraph{Comparison with Static Multi-Layer Fusion Baseline.} 
To further validate the necessity and superiority of our class-adaptive layer fusion strategy, we compare CLASP against a strong static multi-layer fusion baseline. Specifically, we adopt the optimal static layer combination $[2, 17, 23]$ (0-indexed, corresponding to layers 3, 18, and 24) identified as the empirical best practice in recent literature \cite{multilayer_fusion}. 

As shown in Table \ref{tab:static_vs_adaptive}, we evaluate both methods on LLaVA-v1.5-7B across various token retention budgets ($R=192, 128, 64$). While the optimally searched static fusion provides a reasonable baseline, CLASP consistently outperforms it across all budgets and benchmarks. For instance, at a moderate budget of $R=192$, CLASP achieves a significant $+92$ point improvement on MME and $+3.0\%$ on GQA compared to the static baseline. The performance gap widens further under extreme sparsity ($R=64$), where CLASP retains $1709$ on MME versus the static baseline's $1614$ ($+95$ points). These results strongly suggest that different instruction intents inherently require distinct levels of visual abstraction; a fixed combination of layers, even when globally optimized, is fundamentally limited compared to our dynamic, class-conditioned routing approach.

\begin{table*}[h]
    \centering
    \footnotesize
    \caption{Performance comparison between the optimal Static Fusion baseline and our Class-Adaptive Fusion (CLASP) on LLaVA-v1.5-7B under varying token budgets. Best results are highlighted in bold.}
    \label{tab:static_vs_adaptive}
    \resizebox{0.65\linewidth}{!}{%
    \begin{tabular}{c c c c c c}
        \toprule
        \textbf{Budget} & \textbf{Method} & \textbf{MME} & \textbf{GQA} & \textbf{POPE} & \textbf{TextVQA} \\
        \midrule
        Unpruned & LLaVA-v1.5-7B & 1862 & 61.9 & 85.9 & 58.2 \\
        \midrule
        \multirow{2}{*}{R=192} & Static [3, 18, 24] & 1756 & 57.4 & 83.9 & 57.6 \\
         & CLASP (Ours) & \textbf{1848} & \textbf{60.4} & \textbf{85.6} & \textbf{57.6} \\
        \midrule
        \multirow{2}{*}{R=128} & Static [3, 18, 24] & 1701 & 56.2 & 82.7 & 55.3 \\
         & CLASP (Ours) & \textbf{1790} & \textbf{58.9} & \textbf{85.2} & \textbf{56.7} \\
        \midrule
        \multirow{2}{*}{R=64} & Static [3, 18, 24] & 1614 & 54.1 & 78.2 & 51.9 \\
         & CLASP (Ours) & \textbf{1709} & \textbf{57.0} & \textbf{82.8} & \textbf{55.2} \\
        \bottomrule
    \end{tabular}%
    }
\end{table*}

\paragraph{Robustness Analysis across Random Seeds.}
To rigorously verify that the performance improvements achieved by CLASP are statistically significant and not artifacts of random noise or specific initialization configurations, we conducted a variance analysis on the video understanding benchmarks. 
Specifically, we evaluated our method under a $50\%$ retention budget using five different random seeds ($[42, 43, 44, 45, 46]$) for the redundancy-aware clustering initialization. 
As shown in Table~\ref{tab:random_seeds}, the standard deviations across all three video benchmarks (TGIF, MSVD, MSRVTT) are extremely small (ranging from $\pm0.13$ to $\pm0.21$). 
Furthermore, the mean performance across all seeds consistently outperforms the baseline methods evaluated under similar computational constraints. 
This confirms that the dual-stage pruning mechanism---particularly the similarity-driven completion stage---is highly stable and introduces negligible variance, ensuring robust representation compression.

\begin{table}[ht]
    \centering
    \footnotesize
    \caption{\textbf{Variance analysis of CLASP on video benchmarks (TGIF, MSVD, MSRVTT).} 
    Results are reported as mean $\pm$ standard deviation across five random seeds ($42, 43, 44, 45, 46$) under a $50\%$ retention budget.}
    \label{tab:random_seeds}
    \begin{tabular}{l c c c}
        \toprule
        \textbf{Retention Budget ($R$)} & \textbf{TGIF} & \textbf{MSVD} & \textbf{MSRVTT} \\
        \midrule
        $50\%$ & 45.53 $\pm$ 0.18 & 61.72 $\pm$ 0.21 & 51.45 $\pm$ 0.13 \\
        \bottomrule
    \end{tabular}
\end{table}

\paragraph{Robustness of the Prompt-to-Class Router.}
To further evaluate the robustness of our intent classifier and understand the impact of potential routing errors, we conduct two simulated stress tests. 
First, we simulate classifier uncertainty by forcing a specific proportion of samples to be assigned to the ``Default'' category. 
Second, we simulate critical routing failures by randomly misclassifying a certain ratio of samples into incorrect non-default categories.

As shown in Table~\ref{tab:router_robustness}, falling back to the ``Default'' category yields a highly graceful degradation. 
Even when $50\%$ of the samples are forced to the default class, the model maintains competitive performance (\textit{e.g.}, MME merely drops to 1756, and POPE to 83.6). 
In contrast, random misclassification leads to severe performance penalties, with MME dropping to 1722 at a $50\%$ error rate, and experiencing a catastrophic collapse to 1627 at a $100\%$ error rate. 
These findings not only demonstrate the robustness of our routing mechanism but also strongly justify our design choice: when the router encounters ambiguous or out-of-distribution prompts, falling back to a generalized ``Default'' prototype effectively preserves a robust baseline representation. This strategy is significantly safer and more effective than making arbitrary categorical guesses.

\begin{table}[ht]
    \centering
    \footnotesize
    \caption{\textbf{Ablation on router robustness and classification errors.} We compare the impact of forcing a specific ratio of samples to the ``Default'' class (top) versus randomly misclassifying them (bottom). Results demonstrate that falling back to the default class under uncertainty is significantly safer than random assignment.}
    \label{tab:router_robustness}
    \begin{tabular}{l c c c c c}
        \toprule
        \textbf{Forced ``Default'' Ratio} & \textbf{MME} & \textbf{TextVQA} & \textbf{SQA} & \textbf{MMVet} & \textbf{POPE} \\
        \midrule
        $5\%$ & 1840 & 57.4 & 69.5 & 33.3 & 85.5 \\
        $20\%$ & 1799 & 55.0 & 68.7 & 32.8 & 84.7 \\
        $50\%$ & 1756 & 55.1 & 68.0 & 31.3 & 83.6 \\
        $100\%$ & 1715 & 53.7 & 66.7 & 29.6 & 82.2 \\
        \midrule
        \midrule
        \textbf{Random Error Ratio} & \textbf{MME} & \textbf{TextVQA} & \textbf{SQA} & \textbf{MMVet} & \textbf{POPE} \\
        \midrule
        $5\%$ & 1834 & 57.4 & 69.4 & 33.1 & 85.3 \\
        $10\%$ & 1819 & 56.0 & 68.9 & 32.2 & 84.6 \\
        $20\%$ & 1789 & 54.1 & 67.2 & 29.3 & 83.1 \\
        $50\%$ & 1722 & 51.5 & 64.9 & 26.2 & 80.7 \\
        $100\%$ & 1627 & 49.7 & 60.1 & 23.8 & 77.4 \\
        \bottomrule
    \end{tabular}
\end{table}

\paragraph{Detailed Latency and Overhead Breakdown.}
To provide a more granular view of the computational efficiency of our framework and to empirically validate the theoretical complexity discussed in Appendix~C.5, we conduct a micro-benchmark analysis of the inference latency. 
Table~\ref{tab:latency_breakdown} details the time allocation across different components of the pipeline, including Total Prefilling, the Prompt-to-Class Classifier, Multi-layer Weight Fusion, and Spherical K-means clustering, under varying token retention budgets ($R$).

The results clearly demonstrate that the additional operations introduced by CLASP contribute negligibly to the overall inference time. 
Specifically, the text-only intent classifier executes in under $2$ms, and the token-wise weight fusion takes less than $0.5$ms across all settings. 
The K-means clustering refinement (Stage II pruning) requires only $27.2$ms at $R=192$, and its cost naturally scales down to $17.2$ms at $R=64$ as the number of candidate tokens decreases. 
Importantly, these minor overheads are heavily offset by the immense savings in the vision-language decoder. 
By aggressively reducing the sequence length, the Total Prefilling time drops from $125.2$ms to $74.7$ms, and the FLOPs are nearly halved (from $2.67$T to $1.39$T), leading to a highly favorable end-to-end speedup. This confirms that CLASP is an operationally lightweight plug-and-play module.

\begin{table*}[ht]
    \centering
    \footnotesize
    \caption{\textbf{Detailed latency and computational overhead breakdown.} We report the end-to-end Total Time, Memory footprint, FLOPs, and the specific time consumed by distinct pipeline components (Total Prefilling, Classifier, Weight Fusion, and K-means) under varying retention budgets ($R$). The overhead of our proposed components is minimal compared to the overall prefilling savings.}
    \label{tab:latency_breakdown}
    \begin{tabular}{l c c c c c c c}
        \toprule
        \textbf{Retained Tokens ($R$)} & \textbf{Total Time} & \textbf{Memory} & \textbf{FLOPs} & \textbf{Total Prefill} & \textbf{Classifier Time} & \textbf{Weight Fusion} & \textbf{K-means} \\
        \midrule
        192 & 223.1 ms & 17.62 GB & 2.67 T & 125.2 ms & $<2$ ms & $<0.5$ ms & 27.2 ms \\
        128 & 201.7 ms & 17.59 GB & 1.97 T & 109.6 ms & $<2$ ms & $<0.5$ ms & 23.1 ms \\
        64  & 155.8 ms & 17.45 GB & 1.39 T & 74.7 ms  & $<2$ ms & $<0.5$ ms & 17.2 ms \\
        \bottomrule
    \end{tabular}
\end{table*}

\paragraph{Robustness to Router Model Choice.}
To demonstrate that the CLASP framework is robust and not heavily reliant on the specific instruction-following model used for routing, we conduct an ablation study comparing our default text-only router against alternative architectures. 
Table~\ref{tab:router_ablation} details the downstream performance across diverse multimodal benchmarks (MME, TextVQA, SQA, MMVet, and POPE) when substituting the default Qwen3-8B router with Llama3-8B~\cite{grattafiori2024llama} and a larger Qwen3-32B model.

The results clearly demonstrate that while varying the router model introduces slight differences in the intent assignment distribution (indicated by the Count Variance), the overall multimodal performance remains remarkably stable. 
Specifically, the MME score fluctuates marginally between $1840$ and $1848$, and the POPE accuracy remains tightly bounded between $85.3\%$ and $85.7\%$ across all settings. 
Importantly, this stability indicates that our discrete subspace search strategy and category-conditioned fusion mechanisms are highly resilient to minor routing shifts. This confirms that CLASP is effectively model-agnostic with respect to the prompt-to-class router, ensuring reliable and consistent deployment in diverse environments.

\begin{table*}[ht]
    \centering
    \footnotesize
    \caption{\textbf{Performance comparison of CLASP across different prompt-to-class router models.} We report the Count Variance (measuring the intent assignment deviation relative to the default Qwen3-8B router) alongside task accuracy across five representative benchmarks. The downstream performance remains highly stable regardless of the chosen router.}
    \label{tab:router_ablation}
    \begin{tabular}{l c c c c c c}
        \toprule
        \textbf{Router Model} & \textbf{Count Variance} & \textbf{MME} & \textbf{TextVQA} & \textbf{SQA} & \textbf{MMVet} & \textbf{POPE} \\
        \midrule
        Qwen3-8B (Default) & 0    & 1848 & 57.6 & 69.6 & 33.3 & 85.6 \\
        Llama3-8B          & 37.2 & 1840 & 57.5 & 69.4 & 32.9 & 85.7 \\
        Qwen3-32B          & 7.9  & 1848 & 57.7 & 69.4 & 33.1 & 85.3 \\
        \bottomrule
    \end{tabular}
\end{table*}

\paragraph{Comparison with Concurrent State-of-the-Art Methods.}
To further contextualize our contributions against the most recent advancements in the field, we compare CLASP with several concurrent and newly introduced visual token reduction methods: Nuwa~\cite{nuwa2026}, Holov~\cite{holov2025}, and BTP~\cite{btp2025}. We conduct this evaluation on the LLaVA-v1.5-7B architecture under a retention budget of $R=192$ tokens.

As detailed in Table~\ref{tab:concurrent_sota}, CLASP maintains a strong competitive edge against these latest baselines. Most notably, CLASP achieves the highest performance on complex multi-step reasoning benchmarks, securing $1848$ on MME and $33.3$ on MMVet, which substantially outperforms the recent Nuwa model ($1834$ and $30.5$, respectively). On OCR-centric and general scientific QA tasks like TextVQA and SQA, CLASP consistently ranks at or near the top ($57.6$ and $69.6$), demonstrating that our dual-stage class-adaptive pruning preserves essential local visual cues better than fixed-strategy concurrent works. Even on hallucination metrics (POPE), our method remains extremely robust ($85.6$). These results affirm that CLASP represents a leading solution in the rapidly evolving landscape of efficient MLLM inference.

\begin{table}[ht]
    \centering
    \caption{\textbf{Comparison with concurrent state-of-the-art methods.} Evaluation is performed on LLaVA-v1.5-7B with a token retention budget of $R=192$. ``-'' indicates that the result is not reported by the respective authors.}
    \label{tab:concurrent_sota}
    \begin{tabular}{l c c c c c}
        \toprule
        \textbf{Method} & \textbf{MME} & \textbf{TextVQA} & \textbf{SQA} & \textbf{MMVet} & \textbf{POPE} \\
        \midrule
        Nuwa (ICLR 2026)~\cite{nuwa2026}      & 1834 & 57.4 & 68.2 & 30.5 & \textcolor{black}{\textbf{86.4}} \\
        Holov (NeurIPS 2025)~\cite{holov2025} & 1820 & 57.4 & \textcolor{black}{\textbf{69.8}} & -    & 85.6 \\
        BTP (NeurIPS 2025)~\cite{btp2025}     & 1816 & -    & 69.1 & 29.1 & 85.6 \\
        \midrule
        CLASP (ours)                          & \textcolor{black}{\textbf{1848}} & \textcolor{black}{\textbf{57.6}} & 69.6 & \textcolor{black}{\textbf{33.3}} & 85.6 \\
        \bottomrule
    \end{tabular}
\end{table}

\begin{figure*}[t]
    \centering

    \includegraphics[width=0.98\linewidth]{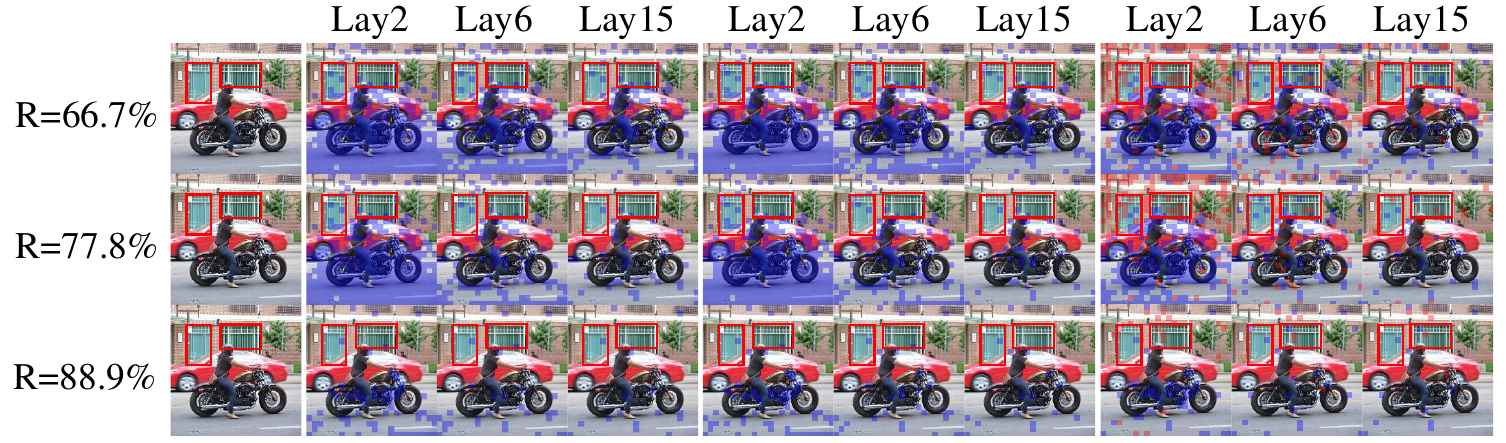}\par\vspace{2pt}
    \includegraphics[width=0.98\linewidth]{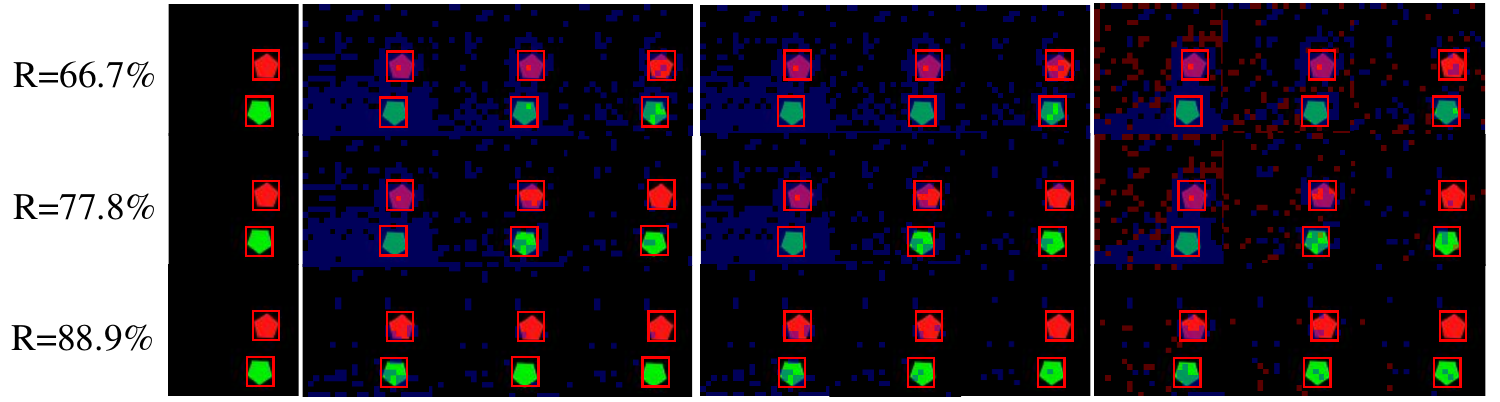}\par\vspace{2pt}
    \includegraphics[width=0.98\linewidth]{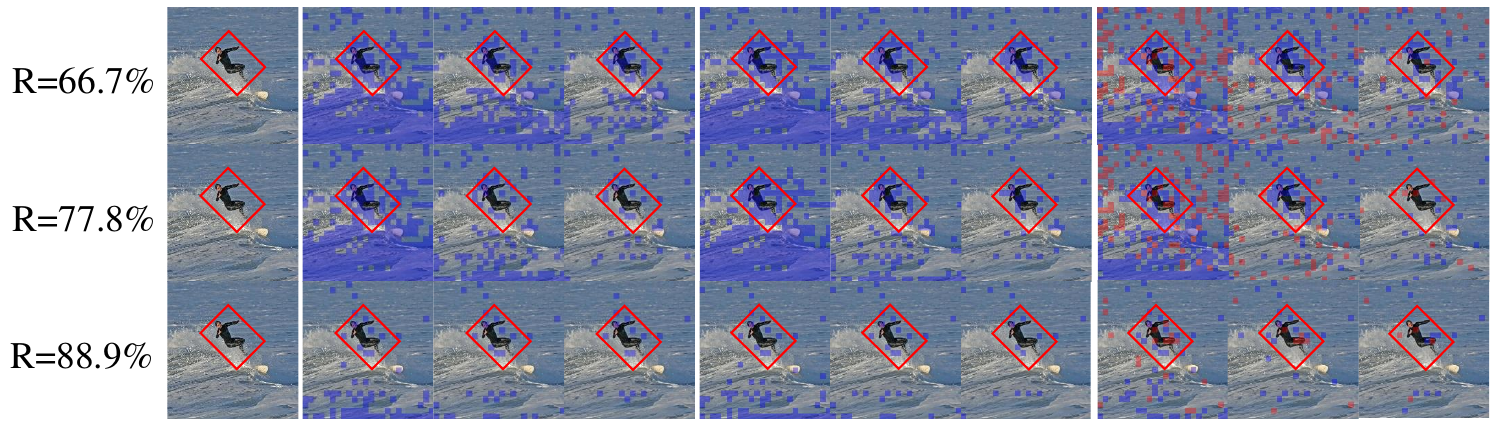}\par\vspace{2pt}
    \includegraphics[width=0.98\linewidth]{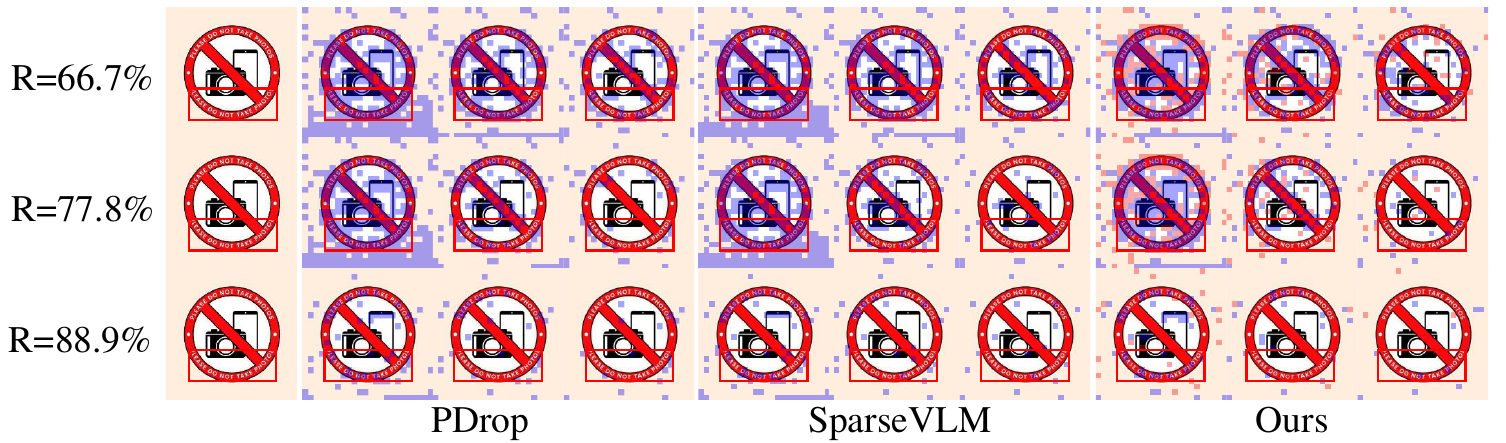}

    \caption{Qualitative comparisons on Cases 1--4. Each case visualizes the original image and the corresponding token-retention map results under three pruning ratios ($R=66.7\%, 77.8\%, 88.9\%$) at layers 2, 6, and 15.}
    \label{fig:appendix_case_1_4}
\end{figure*}

% ===================== Group 2: case_5 - case_8 (4x1) =====================
\begin{figure*}[t]
    \centering

    \includegraphics[width=0.98\linewidth]{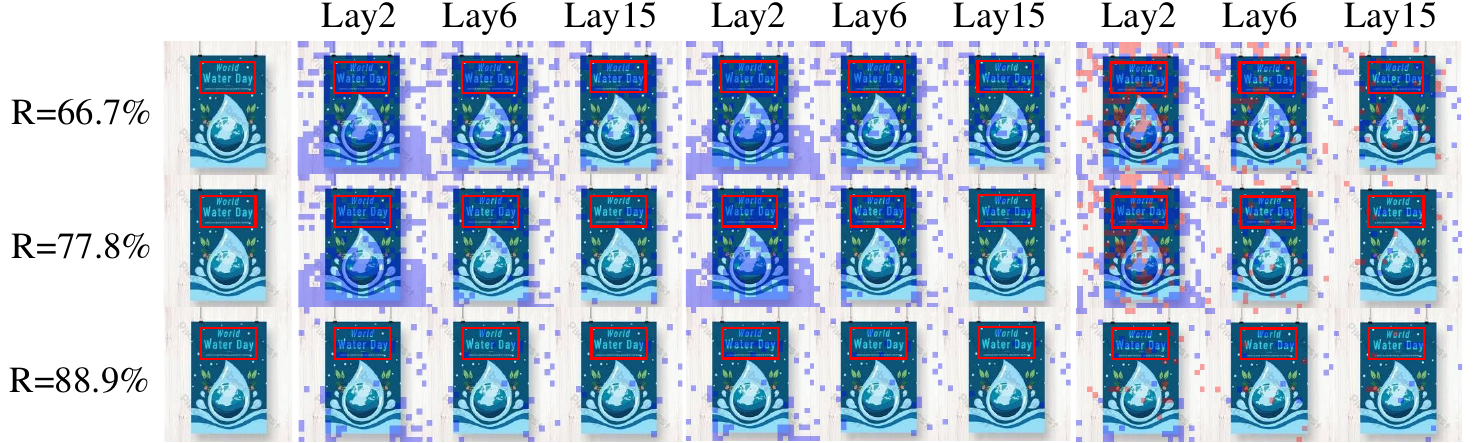}\par\vspace{2pt}
    \includegraphics[width=0.98\linewidth]{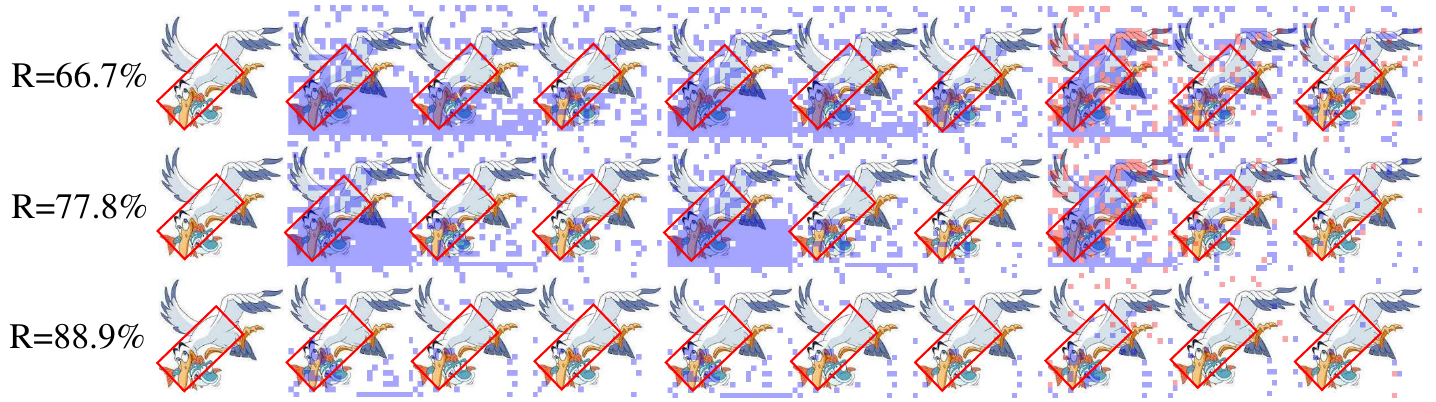}\par\vspace{2pt}
    \includegraphics[width=0.98\linewidth]{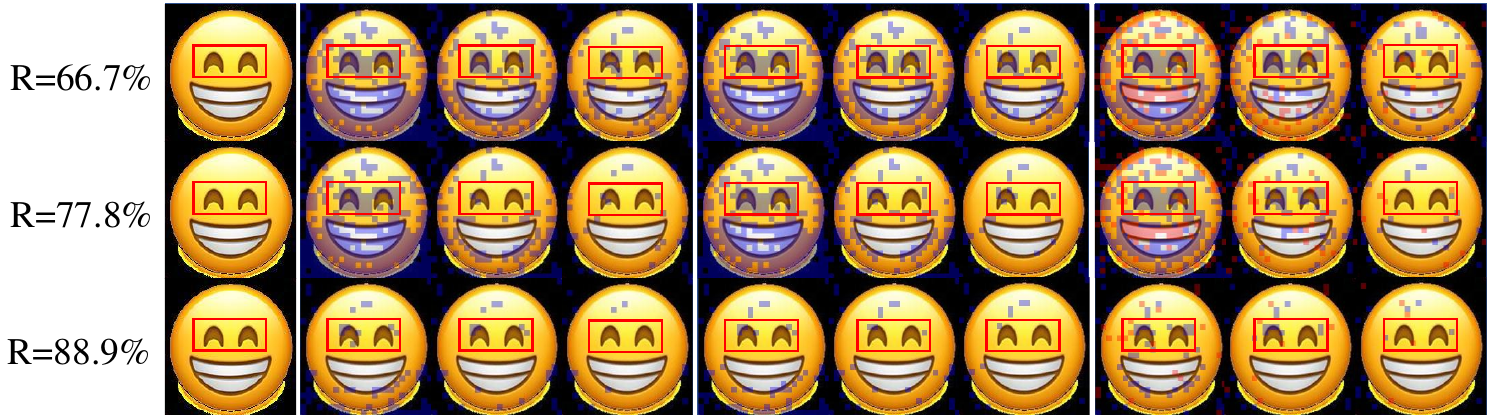}\par\vspace{2pt}
    \includegraphics[width=0.98\linewidth]{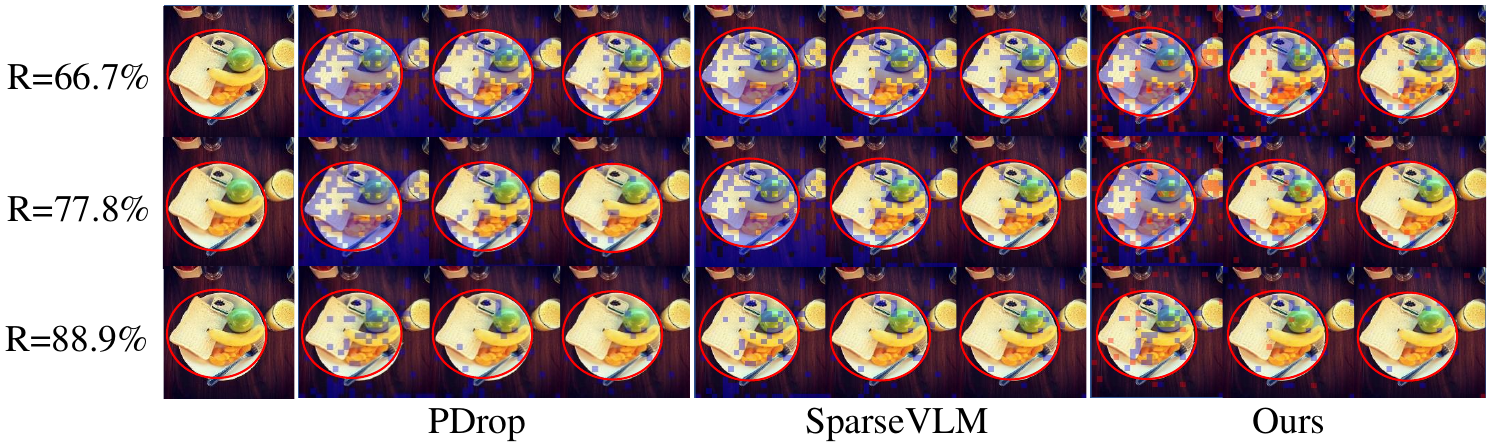}

    \caption{Qualitative comparisons on Cases 5--8. Each case visualizes the original image and the corresponding token-retention map results under three pruning ratios ($R=66.7\%, 77.8\%, 88.9\%$) at layers 2, 6, and 15.}
    \label{fig:appendix_case_5_8}
\end{figure*}

\end{document}